\documentclass{article} 
\usepackage{iclr2023_conference,times}

\usepackage{amsmath,amsfonts,bm}

\def\eqref#1{equation~\ref{#1}}

\def\1{\bm{1}}

\DeclareMathAlphabet{\mathsfit}{\encodingdefault}{\sfdefault}{m}{sl}
\SetMathAlphabet{\mathsfit}{bold}{\encodingdefault}{\sfdefault}{bx}{n}

\usepackage[utf8]{inputenc} 
\usepackage[T1]{fontenc}    
\usepackage[pagebackref=true,breaklinks=true,colorlinks,bookmarks=false, allcolors=blue]{hyperref}
\usepackage[pdftex]{graphicx}
\usepackage{url}            
\usepackage{booktabs}       
\usepackage{amsfonts}       
\usepackage{amsthm}
\usepackage{nicefrac}       
\usepackage{microtype}      
\usepackage{xcolor,pifont}
\usepackage{epsfig}
\usepackage{subfigure}
\usepackage{amsmath}
\usepackage{multirow}
\usepackage{array}
\usepackage{pifont}
\usepackage{booktabs}
\usepackage{wrapfig}
\usepackage{hhline}
\usepackage{mathrsfs} 
\usepackage{algorithmic}
\usepackage[vlined,ruled]{algorithm2e}
\usepackage{color, colortbl}
\usepackage{bm}
\usepackage{tikz}
\def\sym#1{\ifmmode^{#1}\else\(^{#1}\)\fi}

\newcommand{\bfx}{\mathbf{x}}

\newcommand{\eg}{\textit{e.g.,~}}

\newtheorem{definition}{Definition}
\newtheorem{assumption}{Assumption}

\newtheorem{theorem}{Theorem}
\newtheorem{lemma}{Lemma}
\newtheorem{remark}{Remark}
\newtheorem{example}{Example}

\definecolor{forestgreen}{rgb}{0.133, 0.545, 0.133}
\definecolor{yellowyellow}{rgb}{0.133, 0.545, 0.133}
\definecolor{COLOR_MEAN}{HTML}{f0f0f0}
\definecolor{ROW_COLOR}{HTML}{C9F7F4}
\definecolor{greendot}{HTML}{06d6a0}
\definecolor{magneta}{HTML}{FE6D73}
\definecolor{dark_green}{HTML}{17C3B2}

\newcommand*\colourcheck[1]{%
  \expandafter\newcommand\csname #1check\endcsname{\textcolor{#1}{\ding{52}}}%
}
\newcommand*\colourcross[1]{%
  \expandafter\newcommand\csname #1check\endcsname{\textcolor{#1}{\ding{55}}}%
}
\colourcheck{blue}
\colourcheck{green}
\colourcross{red}

\title{Sparse Mixture-of-Experts are Domain Generalizable Learners}

\author{
    Bo Li\textsuperscript{1}\thanks{Equal contribution} \quad
    Yifei Shen\textsuperscript{2}$^{*}$ \;
    Jingkang Yang\textsuperscript{1} \;
    Yezhen Wang\textsuperscript{3} \;
    \textbf{Jiawei Ren\textsuperscript{1}} \; \\
    \;\textbf{Tong Che\textsuperscript{3, 4}} \; \quad
    \textbf{Jun Zhang\textsuperscript{2}} \quad
    \textbf{Ziwei Liu\textsuperscript{1}} \\
    \textsuperscript{1}S-Lab, Nanyang Technological University \\
    \textsuperscript{2}The Hong Kong University of Science and Technology \\
    \textsuperscript{3}Mila-Quebec AI Institute \quad \textsuperscript{4}Nvidia Research \\
    \texttt{\{libo0013,ziwei.liu\}@ntu.edu.sg}\\
}

\iclrfinalcopy 
\begin{document}

\maketitle

\maketitle
\begin{abstract}
    Human visual perception can easily generalize to out-of-distributed visual data, which is far beyond the capability of modern machine learning models. Domain generalization (DG) aims to close this gap, with existing DG methods mainly focusing on the loss function design. In this paper, we propose to explore an orthogonal direction, i.e., the design of the backbone architecture. It is motivated by an empirical finding that transformer-based models trained with empirical risk minimization (ERM) outperform CNN-based models employing state-of-the-art (SOTA) DG algorithms on multiple DG datasets. We develop a formal framework to characterize a network's robustness to distribution shifts by studying its architecture's alignment with the correlations in the dataset. This analysis guides us to propose a novel DG model built upon vision transformers, namely \emph{Generalizable Mixture-of-Experts (GMoE)}. Extensive experiments on DomainBed demonstrate that GMoE trained with ERM outperforms SOTA DG baselines by a large margin. Moreover, GMoE is complementary to existing DG methods and its performance is substantially improved when trained with DG algorithms.

\end{abstract}

\section{Introduction}
\subsection{Motivations}
Generalizing to out-of-distribution (OOD) data is an innate ability for human vision, but highly challenging for machine learning models~\citep{recht2019imagenet,geirhos2021partial,ma2022principles}. Domain generalization (DG) is one approach to address this problem, which encourages models to be resilient under various distribution shifts such as background, lighting, texture, shape, and geographic/demographic attributes. 

From the perspective of representation learning, there
are several paradigms towards this goal, including domain alignment~\citep{DBLP:journals/jmlr/GaninUAGLLML16,DBLP:conf/nips/HoffmanMZ18}, invariant causality prediction~\citep{DBLP:journals/corr/abs-1907-02893,DBLP:conf/icml/KruegerCJ0BZPC21}, meta-learning~\citep{DBLP:journals/corr/abs-2110-09410,zhang2021adaptive}, ensemble learning~\citep{DBLP:conf/icip/ManciniBC018,cha2021swad}, and feature disentanglement~\citep{DBLP:journals/corr/abs-2109-05826,DBLP:journals/corr/abs-2111-13839}. The most popular approach to implementing these ideas is to design a specific loss function. For example, DANN~\citep{DBLP:journals/jmlr/GaninUAGLLML16} aligns domain distributions by adversarial losses. Invariant causal prediction can be enforced by a penalty of gradient norm~\citep{DBLP:journals/corr/abs-1907-02893} or variance of training risks~\citep{DBLP:conf/icml/KruegerCJ0BZPC21}. Meta-learning and domain-specific loss functions~\citep{DBLP:journals/corr/abs-2110-09410,zhang2021adaptive} have also been employed to enhance the performance. Recent studies have shown that these approaches improve ERM and achieve promising results on large-scale DG datasets~\citep{wiles2021fine}.

Meanwhile, in various computer vision tasks, the innovations in backbone architectures play a pivotal role in performance boost and have attracted much attention~\citep{he2016deep,hu2018squeeze,DBLP:conf/iccv/LiuL00W0LG21}. Additionally, it has been empirically demonstrated in~\cite{sivaprasad2021reappraising} that different CNN architectures have different performances on DG datasets. Inspired by these pioneering works, we conjecture that \emph{backbone architecture design would be promising for DG}. To verify this intuition, we evaluate a transformer-based model and compare it with CNN-based architectures of equivalent computational overhead, as shown in Fig. \ref{fig:vit_comp}. To our surprise, a vanilla ViT-S/16~\citep{DBLP:conf/iclr/DosovitskiyB0WZ21} trained with empirical risk minimization (ERM) outperforms ResNet-50 trained with SOTA DG algorithms~\citep{cha2021swad,rame2021fishr,DBLP:journals/corr/abs-2104-09937} on DomainNet, OfficeHome and VLCS datasets, despite the fact that both architectures have a similar number of parameters and enjoy close performance on in-distribution domains. We theoretically validate this effect based on the algorithmic alignment framework~\citep{xu2019what,li2021does}. We first prove that a network trained with the ERM loss function is more robust to distribution shifts if its architecture is more \emph{similar} to the invariant correlation, where the similarity is formally measured by the \emph{alignment value} defined in~\citet{xu2019what}. On the contrary, a network is less robust if its architecture aligns with the spurious correlation. We then investigate the alignment between backbone architectures (i.e., convolutions and attentions) and the correlations in these datasets, which explains the superior performance of ViT-based methods.

To further improve the performance, our analysis indicates that we should exploit properties of invariant correlations in vision tasks and design network architectures to align with these properties. This requires an investigation that sits at the intersection of domain generalization and classic computer vision. In domain generalization, it is widely believed that the data are composed of some sets of attributes and distribution shifts of data are distribution shifts of these attributes~\citep{wiles2021fine}. The latent factorization model of these attributes is almost identical to the generative model of visual attributes in classic computer vision~\citep{ferrari2007learning}. To capture these diverse attributes, we propose a Generalizable Mixture-of-Experts (GMoE), which is built upon sparse mixture-of-experts (sparse MoEs)~\citep{DBLP:conf/iclr/ShazeerMMDLHD17} and vision transformer~\citep{DBLP:conf/iclr/DosovitskiyB0WZ21}. The sparse MoEs were originally proposed as key enablers for extremely large, but efficient models~\citep{fedus2022review}. By theoretical and empirical evidence, we demonstrate that MoEs are experts for processing visual attributes, leading to a better alignment with invariant correlations. Based on our analysis, we modify the architecture of sparse MoEs to enhance their performance in DG. Extensive experiments demonstrate that GMoE achieves superior domain generalization performance both with and without DG algorithms.

\vspace{-2mm}
\subsection{Contributions}
\vspace{-2mm}
In this paper, we formally investigate the impact of the backbone architecture on DG and propose to develop effective DG methods by backbone architecture design. Specifically, our main contributions are summarized as follows:

\textbf{A Novel View of DG:} In contrast to previous works, this paper initiates a formal exploration of the backbone architecture in DG. Based on algorithmic alignment~\citep{xu2019what}, we prove that a network is more robust to distribution shifts if its architecture aligns with the invariant correlation, whereas less robust if its architecture aligns with spurious correlation. The theorems are verified on synthetic and real datasets.

\textbf{A Novel Model for DG:} Based on our theoretical analysis, we propose Generalizable Mixture-of-Experts (GMoE) and prove that it enjoys a better alignment than vision transformers. GMoE is built upon sparse mixture-of-experts~\citep{DBLP:conf/iclr/ShazeerMMDLHD17} and vision transformer~\citep{DBLP:conf/iclr/DosovitskiyB0WZ21}, with a theory-guided performance enhancement for DG.

\textbf{Excellent Performance:} We validate GMoE's performance on all $8$ large-scale datasets of DomainBed. Remarkably, GMoE trained with ERM achieves SOTA performance on $7$ datasets in the train-validation setting and on $8$ datasets in the leave-one-domain-out setting. Furthermore, the GMoE trained with DG algorithms achieves better performance than GMoE trained with ERM.

\section{Preliminaries}
\subsection{Notations}
Throughout this paper, $a$, $\bm{a}$, $\bm{A}$ stand for a scalar, a column vector, a matrix, respectively. $O(\cdot)$ and $\omega(\cdot)$ are asymptotic notations. We denote the training dataset, training distribution, test dataset, and test distribution as $\mathcal{E}_{tr}$, $D_{tr}$, $\mathcal{E}_{te}$, and $D_{te}$, respectively.

\subsection{Attribute Factorization}
The attribute factorization~\citep{wiles2021fine} is a realistic generative model under distribution shifts. Consider a joint distribution of the input $\bm{x}$ and corresponding attributes $a^1, \cdots, a^K$ (denoted as $a^{1:K}$) with $a^i \in \mathcal{A}^i$, where $\mathcal{A}^i$ is a finite set. The label can depend on one or multiple attributes. Denote the latent factor as $\bm{z}$, the data generation process is given by
\begin{align}\label{eq:data}
    \bm{z} \sim p(\bm{z}), \quad a^i \sim p(a^i|\bm{z}), \quad \bm{x} \sim p(\bm{x}|\bm{z}), \quad p(a^{1:K},\bm{x}) = p(a^{1:K})\int p(\bm{x}|\bm{z})p(\bm{z}|a^{1:K})d\bm{z}.
\end{align}
The distribution shift arises if different marginal distributions of the attributes are given but they share the same conditional generative process. Specifically, we have $p_{\text{train}}(a^{1:K}) \neq p_{\text{test}}(a^{1:K})$, but the generative model in \eqref{eq:data} is shared across the distributions, i.e., we have $p_{\text{test}}(a^{1:K},\bm{x}) = p_{\text{test}}(a^{1:K})\int p(\bm{x}|\bm{z})p(\bm{z}|a^{1:K})d\bm{z}$ and similarly for $p_{\text{train}}$. The above description is abstract and we will illustrate with an example.


\begin{example}\label{exmp:dsprites} \textup{(DSPRITES \citep{dsprites17}) Consider $\mathcal{A}^1 = \{\text{red, blue}\}$ and $\mathcal{A}^2 = \{\text{ellipse, square}\}$. The target task is a shape classification task, where the label depends on attribute $a^2$. In the training dataset, $90\%$ ellipses are red and $50\%$ squares are blue, while in the test dataset all the attributes are distributed uniformly. As the majority of ellipses are red, the classifier will use color as a shortcut in the training dataset, which is so-called \emph{geometric skews}~\citep{nagarajan2020understanding}. However, this shortcut does not exist in the test dataset and the network fails to generalize.}
\end{example}

In classic computer vision, the attributes are named \emph{visual attributes} and they follow a similar data generation process~\citep{ferrari2007learning}. We shall discuss them in detail in Section \ref{sec:method}.

\subsection{Algorithmic Alignment}
We first introduce algorithmic alignment, which characterizes the easiness of IID reasoning tasks by measuring the similarity between the backbone architecture and target function. The alignment is formally defined as the following.

\begin{definition}\label{def:aa} (Alignment; \citep{xu2019what}) Let $\mathcal{N}$ denote a neural network with $n$ modules $\{\mathcal{N}_i\}_{i=1}^n$ and assume that a target function for learning $y = g(\bm{x})$ can be decomposed into $n$ functions $f_1,\cdots, f_n$. The network $\mathcal{N}$ aligns with the target function if replacing $\mathcal{N}_i$ with $f_i$, it outputs the same value as algorithm $g$. The alignment value between $\mathcal{N}$ and $f$ is defined as 
\begin{align}\label{eq:alignment}
    \text{Alignment}(\mathcal{N},f,\epsilon,\delta) := n \cdot \max_{i} \mathcal{M}(f_i, \mathcal{N}_i, \epsilon, \delta),
\end{align}
where $\mathcal{M}(f_i, \mathcal{N}_i, \epsilon, \delta)$ denotes the sample complexity measure for $\mathcal{N}_i$ to learn $f_i$ with $\epsilon$ precision at failure probability $\delta$ under a learning algorithm when the training distribution is the same as the test distribution.
\end{definition}

In Definition \ref{def:aa}, the original task is to learn $f$, which is a challenging problem. Intuitively, if we could find a backbone architecture that is suitable for this task, it helps to break the original task into simpler sub-tasks, i.e., to learn $f_1, \cdots, f_n$ instead. Under the assumptions of algorithmic alignment \citep{xu2019what}, $f$ can be learned optimally if the sub-task $f_1, \cdots, f_n$ can be learned optimally. Thus, a good alignment makes the target task easier to learn, and thus improves IID generalization, which is given in Theorem \ref{thm:aa} in Appendix \ref{sec:proof_arch}. In Section \ref{sec:neural_arch_for_dg}, we extend this framework to the DG setting.

\section{On the Importance of Neural Architecture for Domain Generalization}
\label{sec:neural_arch_for_dg}
In this section, we investigate the impact of the backbone architecture on DG, from a motivating example to a formal framework.

\begin{figure*}
    \centering
    \subfigure[Performance comparison.]
    {
        \includegraphics[width=0.4\columnwidth]{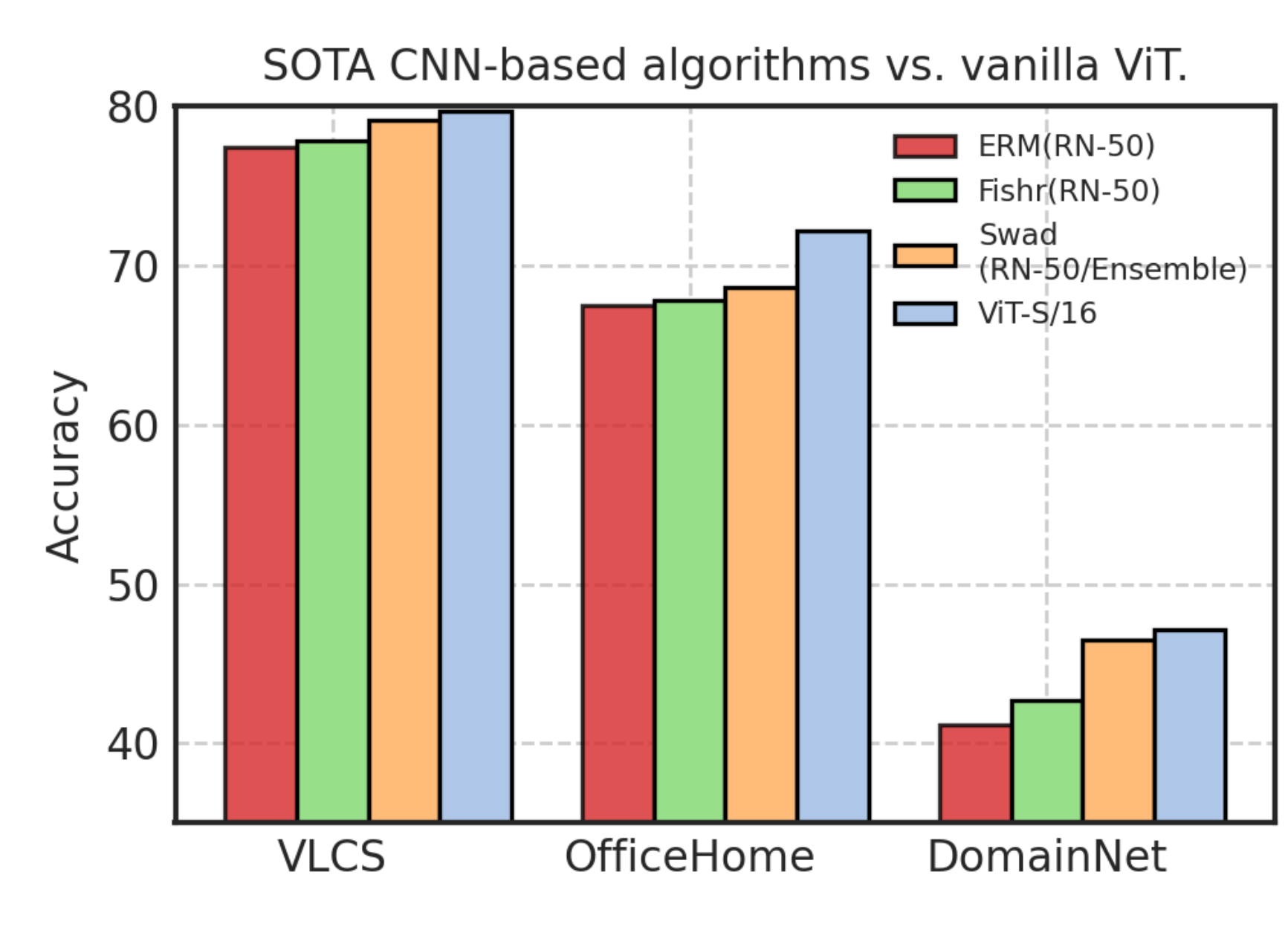}
        \label{fig:vit_comp}
    }\hfil
    \subfigure[Dataset visualization.]
    {
        \includegraphics[width=0.4\columnwidth]{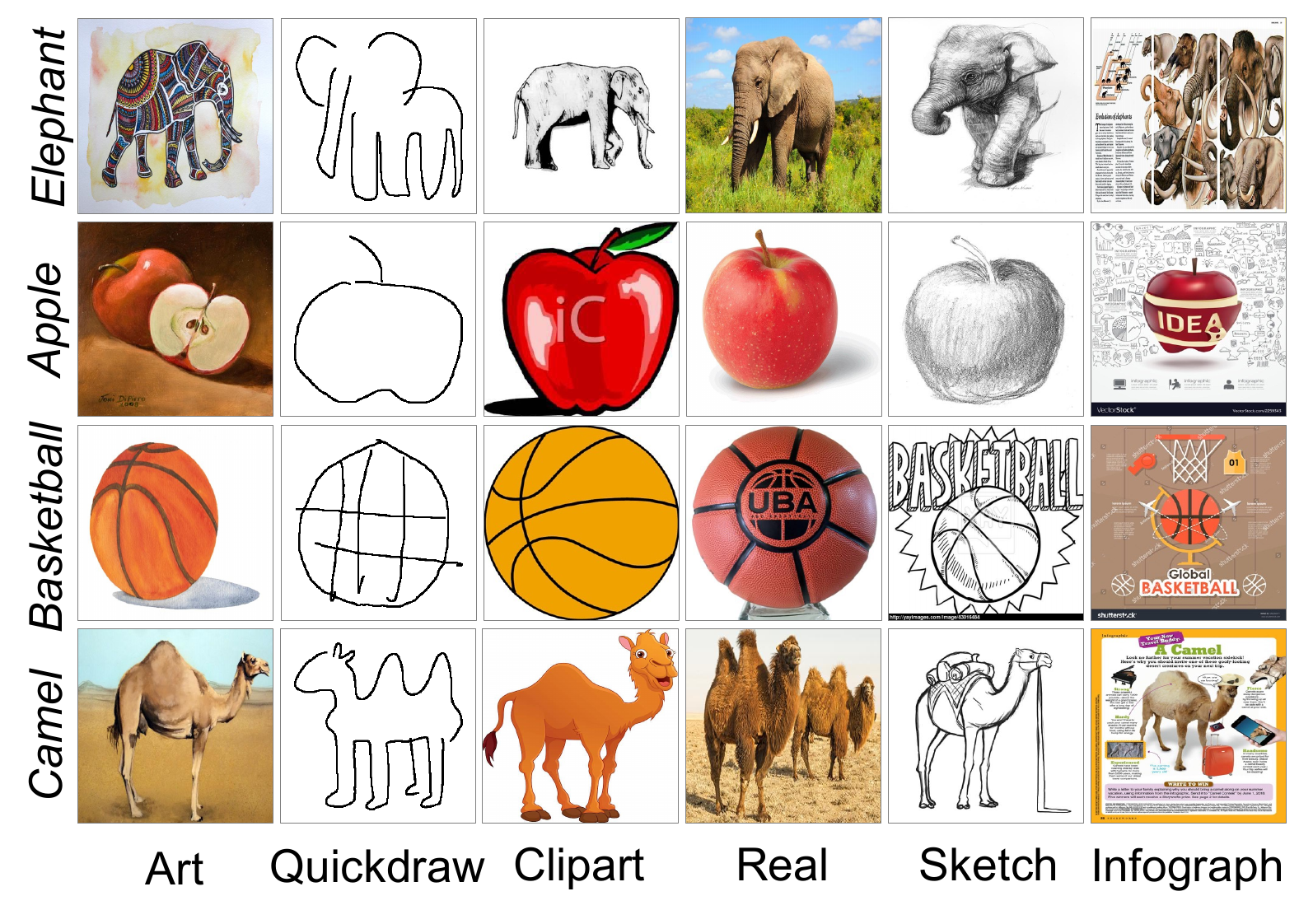}
        \label{fig:DomainNet}
    }\hfil
    
    \caption{(a) Performance comparison of ViT-S/16 (w/21.8M trainable parameters) with ERM and the ResNet-50 (w/25.6M) backbone with SOTA DG algorithms. ERM(RN-50), Fishr(RN-50), Swad(RN-50/Ensemble) denotes the ResNet-50 trained with ERM, Fish~\citep{DBLP:journals/corr/abs-2104-09937} , Fishr~\citep{rame2021fishr}, and Swad~\citep{cha2021swad}, respectively. (b) Image examples from DomainNet. Each row shows one class and each column shows one domain. }
\end{figure*}

\subsection{A Motivating Example: CNNs versus Vision Transformers}
We adopt DomainBed~\citep{DBLP:conf/iclr/GulrajaniL21} as the benchmark, which implements SOTA DG algorithms with ResNet50 as the backbone. We test the performance of ViT trained with ERM on this benchmark, without applying any DG method. The results are shown in Fig. \ref{fig:vit_comp}. To our surprise, ViT trained with ERM already outperforms CNNs with SOTA DG algorithms on several datasets, which indicates that the selection of the backbone architecture is potentially more important than the loss function in DG. In the remaining of this article, we will obtain a theoretical understanding of this phenomenon and improve ViT for DG by modifying its architecture.

\vspace{-1mm}
\subsection{Understanding from a Theoretical Perspective}
\vspace{-1mm}
The above experiment leads to an intriguing question: how does the backbone architecture impact the network's performance in DG? In this subsection, we endeavor to answer this question by extending the algorithmic alignment framework~\citep{xu2019what} to the DG setting.

To have a tractable analysis for nonlinear function approximation, we first make an assumption on the distribution shift.

\begin{assumption}\label{assumption:ds} Denote $\mathcal{N}_1$ as the first module of the network (including one or multiple layers) of the network. Let $p_{\text{train},\mathcal{N}_1}(\bm{s})$ and $p_{\text{test},\mathcal{N}_1}(\bm{s})$ denote the probability density functions of features after $\mathcal{N}_1$. Assume that the support of the training feature distribution covers that of the test feature distribution, i.e., $\max_{\bm{s}} \frac{p_{\text{test},\mathcal{N}_1}(\bm{s})}{p_{\text{train},\mathcal{N}_1}(\bm{s})}\leq C$, where $C$ is a constant independent of the number of training samples.
\end{assumption}

\begin{remark} \textup{(Interpretations of Assumption \ref{assumption:ds}) This condition is practical in DG, especially when we have a pretrained model for disentanglement (e.g., on DomainBed~\citep{DBLP:conf/iclr/GulrajaniL21}). In Example \ref{exmp:dsprites}, the training distribution and test distribution have the same support. In DomainNet, although the elephants in quickdraw are visually different from the elephants in other domains, the quickdraw picture's attributes/features (e.g., big ears and long noise) are covered in the training domains. From a technical perspective, it is impossible for networks trained with gradient descent to approximate a wide range of nonlinear functions in the out-of-support regime~\citep{xu2020neural}. Thus, this condition is necessary if we do not impose strong constraints on the target functions.}
\end{remark}

We define several key concepts in DG. The target function is an invariant correlation across the training and test datasets. For simplicity, we assume that the labels are noise-free.

\begin{assumption}\label{assumption:invariant} (Invariant correlation) Assume there exists a function $g_c$ such that for training data, we have $g_c(\mathcal{N}_1(\bm{x}) ) = y, \forall x \in \mathcal{E}_{tr}$, and for test data, we have $\mathbb{P}_{D_{\text{te}}}[\|g_c(\mathcal{N}_1(\bm{x}) ) - y\| \leq \epsilon] > 1 - \delta$.
\end{assumption}

We then introduce the spurious correlation~\citep{wiles2021fine}, i.e., some attributes are correlated with labels in the training dataset but not in the test dataset. The spurious correlation exists only if the training distribution differs from the test distribution and this distinguishes DG from classic PAC-learning settings~\citep{xu2019what}.

\begin{assumption}\label{assumption:spurious} (Spurious correlation) Assume there exists a function $g_s$ such that for training data $g_s(\mathcal{N}_1(\bm{x}) ) = y, \forall x \in \mathcal{E}_{tr}$, and for test data, we have $\mathbb{P}_{D_{\text{te}}}[\|g_s(\mathcal{N}_1(\bm{x}) ) - y\| > \omega(\epsilon)] > 1 - \delta$.
\end{assumption}

The next theorem extends algorithmic alignment from IID generalization (Theorem \ref{thm:aa}) to DG.

\begin{theorem}\label{thm:arch} (Impact of Backbone Architecture in Domain Generalization) Denote $\mathcal{N}' = \{\mathcal{N}_2,\cdots, \mathcal{N}_n \}$. Assuming we train the neural network with ERM, and Assumption \ref{assumption:ds}, \ref{assumption:invariant}, \ref{assumption:spurious} hold, we have the following statements:
\begin{enumerate}
    \item If $\text{Alignment}(\mathcal{N}',g_c,\epsilon,\delta) \leq |\mathcal{E}_{tr}|$, we have $\mathbb{P}_{D_{te}}[\|\mathcal{N}(\bm{x}) - y\| \leq O(\epsilon)] > 1 - O(\delta)$;
    \item If $\text{Alignment}(\mathcal{N}',g_s,\epsilon,\delta) \leq |\mathcal{E}_{tr}|$, we have $\mathbb{P}_{D_{te}}[\|\mathcal{N}(\bm{x}) - y\| > \omega(\epsilon)] > 1 - O(\delta)$. 
\end{enumerate}
\end{theorem}

\begin{remark} \textup{(Interpretations of Theorem \ref{thm:arch}) By choosing a sufficiently small $\epsilon$, only one of $\text{Alignment}(\mathcal{N}',g_c,\epsilon,\delta) \leq |\mathcal{E}_{tr}|$ and $\text{Alignment}(\mathcal{N}',g_s,\epsilon,\delta) \leq |\mathcal{E}_{tr}|$ holds. Thus, Theorem \ref{thm:arch} shows that the networks aligned with invariant correlations are more robust to distribution shifts. In Appendix \ref{subsec:synthetic_experiments}, we build a synthetic dataset that satisfies all the assumptions. The experimental results exactly match Theorem \ref{thm:arch}. In practical datasets, the labels may have colored noise, which depends on the spurious correlation. Under such circumstances, the network should rely on multiple correlations to fit the label well and the correlation that best aligns with the network will have the major impact on its performance. Please refer to Appendix \ref{sec:proof_arch} for the proof.}
\end{remark}


\paragraph{ViT with ERM versus CNNs with DG algorithms} We now use Theorem \ref{thm:arch} to explain the experiments in the last subsection. The first condition of Theorem \ref{thm:arch} shows that if the neural architecture aligns with the invariant correlation, ERM is sufficient to achieve a good performance. In some domains of OfficeHome or DomainNet, the shape attribute has an invariant correlation with the label, illustrated in Fig. \ref{fig:DomainNet}. On the contrary, a spurious correlation exists between the attribute texture and the label. According to the analysis in~\citet{park2022vision}, multi-head attentions (MHA) are low-pass filters with a \emph{shape} bias while convolutions are high-pass filters with a \emph{texture} bias. As a result, a ViT simply trained with ERM can outperform CNNs trained with SOTA DG algorithms.

To improve ViT's performance, Theorem \ref{thm:arch} suggests that we should exploit the properties of invariant correlations. In image recognition, objects are described by functional parts (e.g., visual attributes), with words associated with them~\citep{zhou2014object}. The configuration of the objects has a large degree of freedom, resulting in different shapes among one category. Therefore, functional parts are more fundamental than shape in image recognition and we will develop backbone architectures to capture them in the next section.

\section{Generalizable Mixture-of-Experts for Domain Generalization}
\label{sec:impl}
In this section, we propose Generalizable Mixture-of-Experts (GMoE) for domain generalization, supported by effective neural architecture design and theoretical analysis.
\subsection{Mixture-of-Experts Layer}
\label{sec:impl:moe}
In this subsection, we introduce the mixture-of-experts (MoE) layer, which is an essential component of GMoE. One ViT layer is composed of an MHA and an FFN. In the MoE layer, the FFN is replaced by mixture-of-experts and each expert is implemented by an FFN~\citep{DBLP:conf/iclr/ShazeerMMDLHD17}. Denoting the output of the MHA as $\bm{x}$, the output of the MoE layer with $N$ experts is given by
\begin{align}\label{eq:moe}
    f_{\text{MoE}}(\bm{x}) = \sum_{i=1}^N  G(\bm{x})_i \cdot E_{i}(\bm{x}) = \sum_{i=1}^N \text{TOP}_k(\text{Softmax}(\bm{W}\bm{x})) \cdot \bm{W}^2_{\text{FFN}_i} \phi(\bm{W}^1_{\text{FFN}_i}\bm{x}),
\end{align}
where $\bm{W}$ is the learnable parameter for the gate, $\bm{W}^1_{\text{FFN}_i}$ and $\bm{W}^2_{\text{FFN}_i}$ are learnable parameters for the $i$-th expert, $\phi(\cdot)$ is a nonlinear activation function, and $\text{TOP}_k(\cdot)$ operation is a one-hot embedding that sets all other elements in the output vector as zero except for the elements with the largest $k$ values where $k$ is a hyperparameter. Given $\bm{x}_{\text{in}}$ as the input of the MoE layer, the update is given by
\begin{align*}
    \bm{x} = f_{\text{MHA}}(\text{LN}(\bm{x}_{\text{in}})) + \bm{x}_{\text{in}} ,\quad \bm{x}_{\text{out}} = f_{\text{MoE}}(\text{LN}(\bm{x})) + \bm{x},
\end{align*}
where $f_{\text{MHA}}$ is the MHA layer, $\text{LN}$ represents layer normalization, and $\bm{x}_{\text{out}}$ is the output of the MoE layer.

\subsection{Visual Attributes, Conditional Statements, and Sparse MoEs}\label{sec:method}
\begin{wrapfigure}[12]{R}{0.30\textwidth}
\vspace{-1em}
    \begin{algorithm}[H]\label{alg:condition}
      \DontPrintSemicolon
      Define intervals $I_i \subset \mathbb{R}, i = 1, \cdots, M$\;
      Define functions $h_i, , i = 1, \cdots, M+1$\;
      \Switch(){$h_1(\bm{x})$}
      { \If{$ h_1(\bm{x}) \in I_i$}{apply $h_{i+1}$ to $\bm{x}$}}
    \caption{\texttt{Conditional Statements} }
    \end{algorithm}
\end{wrapfigure}
In real world image data, the label depends on multiple attributes. Capturing \emph{diverse} visual attributes is especially important for DG. For example, the definition of an \emph{elephant} in the Oxford dictionary is ``a very large animal with thick grey skin, large ears, two curved outer teeth called tusks, and a long nose called a trunk''. The definition involves three shape attributes (i.e., large ears, curved outer teeth, and a long nose) and one texture attribute (i.e., thick grey skin). In the IID ImageNet task, using the most discriminative attribute, i.e., the thick grey skin, is sufficient to achieve high accuracy~\citep{geirhos2018imagenet}. However, in DomainNet, elephants no longer have grey skins while the long nose and big ears are preserved and the network relying on grey skins will fail to generalize.

The conditional statement (i.e., IF/ELSE in programming), as shown in Algorithm \ref{alg:condition}, is a powerful tool to efficiently capture the visual attributes and combine them for DG. Suppose we train the network to recognize the elephants on DomainNet, as illustrated in the first row of Fig. \ref{fig:DomainNet}. For the elephants in different domains, shape and texture vary significantly while the visual attributes (large ears, curved teeth, long nose) are invariant across all the domains. Equipped with conditional statements, the recognition of the elephants can be expressed as ``if an animal has large ears, two curved outer teeth, and a long nose, then it is an elephant''. Then the subtasks are to recognize these visual attributes, which also requires conditional statements. For example, the operation for ``curved outer teeth'' is that ``if the patch belongs to the teeth, then we apply a shape filter to it''. In literature, the MoE layer is considered an effective approach to implement conditional computations~\citep{DBLP:conf/iclr/ShazeerMMDLHD17,DBLP:conf/nips/RiquelmePMNJPKH21}. We formalize this intuition in the next theorem.

\begin{theorem}\label{thm:moe} An MoE module in ~\eqref{eq:moe} with $N$ experts and $k=1$ aligns with the conditional statements in Algorithm \ref{alg:condition} with 
\begin{align}\label{eq:align_moe}
    \text{Alignment} = \left\{
\begin{aligned}
&(N+1) \cdot \max\left( \mathcal{M}^*_{\mathscr{P}}, \mathcal{M}(G, h_{1},\epsilon,\delta) \right), && \text{if } N < M,\\
&(N+1) \cdot \max \left( \max_{i\in \{1,\cdots, M\} } \mathcal{M}(f_{\text{FFN}_i}, h_{i+1},\epsilon,\delta), \mathcal{M}(G, h_{1},\epsilon,\delta) \right), &&\text{if } N \geq M,  \\
\end{aligned}
\right.
\end{align}
where $\mathcal{M}(\cdot,\cdot,\cdot,\cdot)$ is defined in Definition \ref{def:aa}, and $\mathcal{M}^*_{\mathscr{P}}$ is the optimal objective value of the following optimization problem:
\begin{equation}\label{eq:sc_opt}
\begin{aligned}
\mathscr{P}:&\underset{\mathcal{I}_1, \cdots \mathcal{I}_N}{\text{minimize}}
& &  \max_{i \in \{1,\cdots,N \} } \quad \mathcal{M}(f_{\text{FFN}_i}, ([1_{I_j}]_{j\in \mathcal{I}_{i} } \circ h_1)^T \cdot [h_j]_{j\in \mathcal{I}_{i} },\epsilon,\delta) \\
& \text{ subject to }
& & \cup_{i=1}^N \mathcal{I}_i = \{2,3, \cdots, M+1\},
\end{aligned},
\end{equation}
where $1_{I_j}$ is the indicator function on interval $I_j$.
\end{theorem}

\begin{remark} \textup{(Interpretations of Theorem \ref{thm:moe}) In algorithmic alignment, the network better aligns with the algorithm if the alignment value in \eqref{eq:alignment} is lower. The alignment value between MoE and conditional statements depends on the product of $N+1$ and a sample complexity term. When we increase the number of experts $N$, the alignment value first decreases as multiple experts decompose the original conditional statements into several simpler tasks. As we further increase $N$, the alignment value increases because of the factor $N+1$ in the product. Therefore, the MoE aligns better with conditional statements than with the original FFN (i.e., $N=1$). In addition, to minimize \eqref{eq:sc_opt}, similar conditional statements should be grouped together. By experiments in Section \ref{subsec:model_analysis} and Appendix \ref{subsec:cub_dg}, we find that sparse MoE layers are indeed experts for visual attributes, and similar visual attributes are handled by one expert. Please refer to Appendix \ref{sec:proof_moe} for the proof.}
\end{remark}



%

\subsection{Adapting MoE to Domain Generalization}
\label{sec:impl:diff_teq}
In literature, there are several variants of MoE architectures, e.g., ~\citet{DBLP:conf/nips/RiquelmePMNJPKH21,fedus2022review}, and we should identify one for DG. By algorithmic alignment, in order to achieve a better generalization, the architecture of sparse MoEs should be designed to effectively handle visual attributes. In the following, we discuss our architecture design for this purpose.

\paragraph{Routing scheme} Linear routers (i.e., \eqref{eq:moe}) are often adopted in MoEs for vision tasks~\citep{DBLP:conf/nips/RiquelmePMNJPKH21} while recent studies in NLP show that the cosine router achieves better performance in cross-lingual language tasks~\citep{DBLP:journals/corr/abs-2204-09179}. For the cosine router, given input $\bm{x} \in \mathbb{R}^d$, the embedding $\bm{W}\bm{x} \in \mathbb{R}^{d_e}$ is first projected onto a hypersphere, followed by multiplying a learned embedding $\bm{E} \in \mathbb{R}^{d_e \times N}$. Specifically, the expression for the gate is given by
\begin{align*}
    G(\bm{x}) = \text{TOP}_k\left(\text{Softmax}\left(\frac{\bm{E}^T\bm{W}\bm{x}}{\tau \|\bm{W}\bm{x}\| \|\bm{E} \|}  \right) \right),
\end{align*}
where $\tau$ is a hyper-parameter. In the view of image processing, $\bm{E}$ can be interpreted as the cookbook for visual attributes~\citep{ferrari2007learning,zhou2014object} and the dot product between $\bm{E}$ and $\bm{W}\bm{x}$ with $\ell_2$ normalization is a matched filter. We opine that the linear router would face difficulty in DG. For example, the elephant image (and its all patches) in the Clipart domain is likely more similar to other images in the Clipart domain than in other domains. The issue can be alleviated with a codebook for visual attributes and matched filters for detecting them. Please refer to Appendix~\ref{subsec:ablation_model_design} for the ablation study.

\paragraph{Number of MoE layers} \emph{Every-two} and \emph{last-two} are two commonly adopted placement methods in existing MoE studies~\citep{DBLP:conf/nips/RiquelmePMNJPKH21,DBLP:conf/iclr/LepikhinLXCFHKS21}. Specifically, every-two refers to replacing the even layer's FFN with MoE, and last-two refers to placing MoE at the last two even layers. For IID generalization, every-two often outperforms last-two~\citep{DBLP:conf/nips/RiquelmePMNJPKH21}. We argue that last-two is more suitable for DG as the conditional sentences for processing visual attributes are high-level. From experiments, we empirically find that last-two achieves better performance than every-two with fewer computations. Please refer to Appendix~\ref{subsec:detailed_arch} for more discussions and Appendix~\ref{subsec:ablation_model_design} for the ablation study.


The overall backbone architecture of GMoE is shown in Fig. \ref{fig:gmoe}. To train diverse experts, we adopt the perturbation trick and load balance loss as in~\citet{DBLP:conf/nips/RiquelmePMNJPKH21}. Due to space limitation, we leave them in Appendix~\ref{subsec:aux_loss}. 

\begin{figure}[tp]
\centering
\vspace{-4mm}
\includegraphics[width=0.8\textwidth]{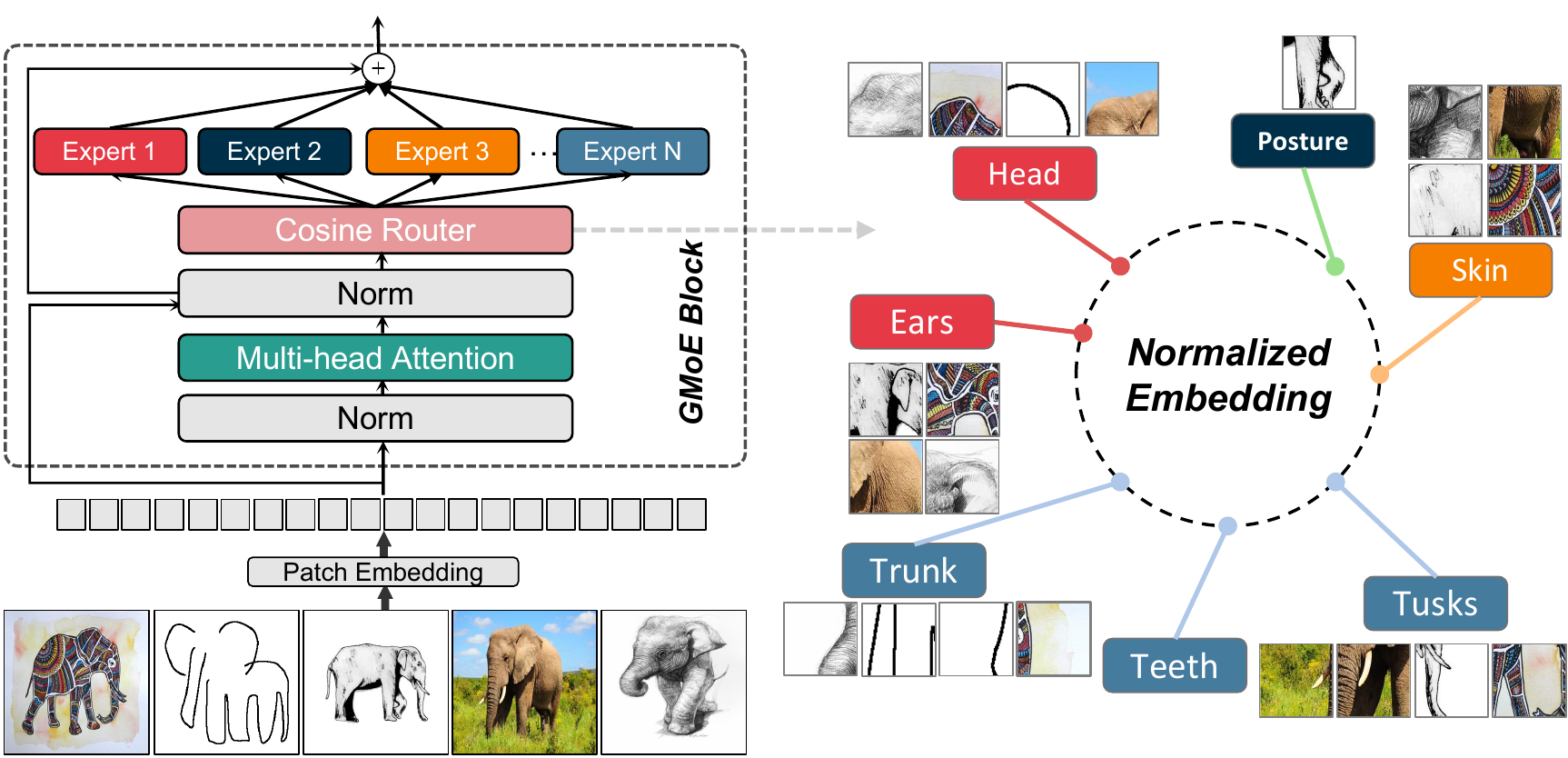}
\caption{Overview architecture of GMoE. The cosine router distributes normalized image patches of different visual attributes to corresponding experts. Our analysis and experiments (in Section \ref{sec:method} and Section \ref{subsec:model_analysis}) demonstrate that an expert is potentially responsible for a group of similar visual attributes.}
\vspace{-4mm}
\label{fig:gmoe}
\end{figure}

\section{Experimental Results}
\label{sec:exps}
In this section, we evaluate the performance of GMoE on large-scale DG datasets and present model analysis to understand GMoE.

\subsection{DomainBed Results}
\label{subsec:domainbed_results}
In this subsection, we evaluate GMoE on DomainBed~\citep{DBLP:conf/iclr/GulrajaniL21} with $8$ benchmark datasets: PACS, VLCS, OfficeHome, TerraIncognita, DomainNet, SVIRO, Wilds-Camelyon and Wilds-FMOW. Detailed information on datasets and evaluation protocols are provided in Appendix \ref{subsec:domainbed}. The experiments are averaged over 3 runs as suggested in \citep{DBLP:conf/iclr/GulrajaniL21}.

We present results in Table~\ref{tab:main_results} with \textbf{train-validation selection}, which include baseline methods and recent DG algorithms and GMoE trained with ERM. The results demonstrate that GMoE without DG algorithms already outperforms counterparts on almost all the datasets. Meanwhile, GMoE has excellent performance in \textbf{leave-one-domain-out criterion}, and we leave the results in Appendix \ref{subsec:lodo} due to space limit. In the lower part of Table~\ref{tab:main_results}, we test our methods on three large-scale datasets: SVIRO, Wilds-Camelyon, and Wilds-FMOW. The three datasets capture real-world distribution shifts across a diverse range of domains. We adopt the data preprocessing and domain split in DomainBed. As there is no previous study conducting experiments on these datasets with DomainBed criterion, we only report the results of our methods, which reveal that GMoE outperforms the other two baselines.
\vspace{-3mm}

\begin{table}[t]
\vspace{-2em}
\setlength{\tabcolsep}{5pt}
\renewcommand{\arraystretch}{1.0}
\centering
\caption{Overall out-of-domain accuracies with train-validation selection criterion. The best result is highlighted in \textbf{bold}. GMoE achieves the \textbf{best} performance on PACS, VLCS, OfficeHome, and DomainNet while ranking in the \underline{third-best} on TerraIncognita.}
\label{tab:main_results}
\resizebox{\textwidth}{!}{%
\begin{tabular}{c|ccccc}
\toprule
\rowcolor{COLOR_MEAN}
\textbf{Algorithm} & \textbf{PACS} & \textbf{VLCS} & \textbf{OfficeHome} & \textbf{TerraInc} & \textbf{DomainNet} \\ \midrule
ERM (ResNet50)~\citep{vapnik1991principles}   & 85.7 $\pm$ 0.5  & 77.4 $\pm$ 0.3  & 67.5 $\pm$ 0.5        & 47.2 $\pm$ 0.4      & 41.2 $\pm$ 0.2       \\
IRM [ArXiv 20]~\citep{DBLP:journals/corr/abs-1907-02893}   & 83.5 $\pm$ 0.8  & 78.5 $\pm$ 0.5  & 64.3 $\pm$ 2.2        & 47.6 $\pm$ 0.8      & 33.9 $\pm$ 2.8       \\
DANN [JMLR 16]~\citep{DBLP:journals/jmlr/GaninUAGLLML16}  & 84.6 $\pm$ 1.1  & 78.7 $\pm$ 0.3  & 68.6 $\pm$ 0.4        & 46.4 $\pm$ 0.8      & 41.8 $\pm$ 0.2       \\
CORAL [ECCV 16]~\citep{DBLP:conf/eccv/SunS16} & 86.0 $\pm$ 0.2  & 77.7 $\pm$ 0.5  & 68.6 $\pm$ 0.4        & 46.4 $\pm$ 0.8      & 41.8 $\pm$ 0.2       \\
MMD [CVPR 18]~\citep{DBLP:conf/cvpr/LiPWK18}   & 85.0 $\pm$ 0.2  & 76.7 $\pm$ 0.9  & 67.7 $\pm$ 0.1        & 42.2 $\pm$ 1.4      & 39.4 $\pm$ 0.8       \\
FISH [ICLR 22]~\citep{DBLP:journals/corr/abs-2104-09937}  & 85.5 $\pm$ 0.3  & 77.8 $\pm$ 0.3  & 68.6 $\pm$ 0.4        & 45.1 $\pm$ 1.3      & 42.7 $\pm$ 0.2       \\
SWAD [NeurIPS 21]~\citep{DBLP:conf/nips/ChaCLCPLP21}  & 88.1 $\pm$ 0.1  & 79.1 $\pm$ 0.1  & 70.6 $\pm$ 0.2        & 50.0 $\pm$ 0.3      & 46.5 $\pm$ 0.1       \\
Fishr [ICML 22]~\citep{rame2021fishr}  &  85.5 $\pm$ 0.2  & 77.8 $\pm$ 0.2  & 68.6 $\pm$ 0.2        & 47.4 $\pm$ 1.6 & 41.7 $\pm$ 0.0       \\
MIRO [ECCV 22]~\citep{DBLP:journals/corr/abs-2203-10789}  & 85.4 $\pm$ 0.4  & 79.0 $\pm$ 0.0  & 70.5 $\pm$ 0.4        & \textbf{50.4 $\pm$ 1.1}     & 44.3 $\pm$ 0.2       \\
ERM (ViT-S/16) [ICLR 21]~\citep{DBLP:conf/iclr/DosovitskiyB0WZ21}   & 86.2 $\pm$ 0.1  & 79.7 $\pm$ 0.0  & 72.2 $\pm$ 0.4        & 42.0 $\pm$ 0.8      & 47.3 $\pm$ 0.2       \\
\rowcolor{ROW_COLOR}
\textbf{GMoE-S/16 (Ours)} & \textbf{88.1 $\pm$ 0.1} & \textbf{80.2 $\pm$ 0.2}  & \textbf{74.2   $\pm$ 0.4}  & \underline{48.5} $\pm$ 0.4 & \textbf{48.7 $\pm$ 0.2} \\
\midrule
\rowcolor{COLOR_MEAN}
\midrule
\textbf{Algorithms} & \multicolumn{2}{c}{\textbf{SVIRO}}   & \multicolumn{2}{c}{\textbf{Wilds-Camelyon}}   & \textbf{Wilds-FMOW} \\ \midrule
ERM (ResNet50)~\citep{vapnik1991principles}     & \multicolumn{2}{c}{85.7 $\pm$ 0.1} & \multicolumn{2}{c}{93.1 $\pm$ 0.2} & 40.6 $\pm$ 0.4 \\
ERM (ViT-S/16) [ICLR 21]~\citep{DBLP:conf/iclr/DosovitskiyB0WZ21}     & \multicolumn{2}{c}{89.6 $\pm$ 0.0} & \multicolumn{2}{c}{91.1 $\pm$ 0.1} & 44.8 $\pm$ 0.2 \\
\rowcolor{ROW_COLOR}
\textbf{GMoE-S/16 (Ours)}              & \multicolumn{2}{c}{\textbf{90.3 $\pm$ 0.1}} & \multicolumn{2}{c}{\textbf{93.7 $\pm$ 0.2}} & \textbf{46.6 $\pm$ 0.4} \\

\bottomrule
\end{tabular}%
}
\end{table}
\subsection{GMoE with DG Algorithms}
\label{subsec:loss_design}
GMoE's generalization ability comes from its internal backbone architecture, which is orthogonal to existing DG algorithms. This implies that the DG algorithms can be applied to improve the GMoE's performance.  To validate this idea, we apply two DG algorithms to GMoE, including one modifying loss functions approaches (Fish) and one adopting model ensemble (Swad). The results in Table~\ref{tab:with_loss} demonstrate that adopting GMoE instead of ResNet-50 brings significant accuracy promotion to these DG algorithms. Experiments on GMoE with more DG algorithms are in Appendix~\ref{subsec:ablation_dg_algos}.
\vspace{-3mm}
\begin{wraptable}[11]{R}{0.42\textwidth}
\vspace{-4mm}
\small
    \caption{GMoE trained with DG algorithms.}
    \label{tab:with_loss}
    \renewcommand{\arraystretch}{1.1}
    \setlength{\tabcolsep}{10pt}
    \begin{tabular}{c|c}
    \toprule
    \rowcolor{COLOR_MEAN}
    \textbf{Algorithm} & \textbf{DomainNet} \\ \midrule
    GMoE & 48.7 \\ \hline
    Fish~\citep{rame2021fishr} & 42.7 \\
    \rowcolor{ROW_COLOR}
    \textbf{GMoE w/Fish}  &      \textbf{48.8}              \\ \hline
    Swad~\citep{DBLP:conf/nips/ChaCLCPLP21} & 46.5 \\
    \rowcolor{ROW_COLOR}
    \textbf{GMoE w/Swad}  & \textbf{49.6}              \\\bottomrule
    \end{tabular}%
\vspace{-4mm}
\end{wraptable}


\subsection{Single-Source Domain Generalization Results}
In this subsection, we create a challenging task, \emph{single-source domain generalization}, to focus on generalization ability of backbone architecture. Specifically, we train the model only on data from one domain, and then test the model on multiple domains to validate its performance across all domains. This is a challenging task as we cannot rely on multiple domains to identify invariant correlations, and popular DG algorithms cannot be applied. We compare several models mentioned in above analysis (\eg ResNet, ViT, GMoE) with different scale of parameters, including their float-point-operations per second (flops), IID and OOD accuracy. From the results in Table~\ref{tab:purely_generalization}, we see that GMoE's OOD generalization gain over ResNet or ViT is much larger than that in the IID setting, which shows that GMoE is suitable for challenging domain generalization. Due to space limitation, we leave experiments with other training domains in Appendix~\ref{subsec:purely_ood_results}.
\vspace{-3mm}

\subsection{Model Analysis}
\label{subsec:model_analysis}
In this subsection, we present diagnostic datasets to study the connections between MoE layer and the visual attributes.



\begin{table}[t]
\vspace{-2em}
\centering
\caption{Single-source DG accuracy (\%). Models are trained on Paint domain and tested (1) on Paint validation set (2) the rest 5 domains' validation sets on DomainNet. The flops are reference values to compare the model's computational efficiency. \textbf{IID Imp.} denotes IID improvement on Paint's validation set comparing with ResNet50. \textbf{OOD Imp.} denotes average OOD improvement across 5 test domains.}
\label{tab:purely_generalization}
\setlength{\tabcolsep}{6pt}
\renewcommand{\arraystretch}{1.2}
\footnotesize
\begin{tabular}{c|ccccccc|cc}
\toprule
\rowcolor{ROW_COLOR}
\textbf{MACs} & \cellcolor{COLOR_MEAN} \textbf{Paint} & \textbf{Clipart} & \textbf{Info} & \cellcolor{COLOR_MEAN} \textbf{Paint} & \textbf{Quick} & \textbf{Real} & \textbf{Sketch} & \cellcolor{COLOR_MEAN} \textbf{IID Imp.} & \textbf{OOD Imp.} \\ \midrule
4.1G & ResNet50 & 37.1 & 12.9 & 62.7 & 2.2 & 49.3 & 33.3 & - & - \\
7.9G & ResNet101 & 40.5 & 13.1 & 63.4 & 3.1 & 51.2 & 35.4 & 1.1\% & 12.4\% \\
4.6G & ViT-S/16 & 42.7 & 15.9 & 69.0 & 5.0 & \textbf{56.4} & 37.0 & 10.0\% & 38.6\% \\
4.8G & GMoE-S/16 & \textbf{43.5} & \textbf{16.1} & \textbf{69.3} & \textbf{5.3} & \textbf{56.4} & \textbf{38.0} & \textbf{10.5\%} & \textbf{42.3\%} \\
\bottomrule
\end{tabular}%
\vspace{-4mm}
\end{table}
\paragraph{Diagnostic datasets: CUB-DG}
We create CUB-DG from the original Caltech-UCSD Birds (CUB) dataset~\citep{wah2011caltech}. We stylize the original images into another three domains, Candy, Mosaic and Udnie. The examples of CUB-DG are presented in Fig. \ref{fig:cub-dg}. We evaluate GMoE and other DG algorithms that address domain-invariant representation (e.g., DANN). The results are in Appendix~\ref{subsec:cub_dg}, which demonstrate superior performance of GMoE compared with other DG algorithms. CUB-DG datasets provides rich visual attributes (e.g. beak's shape, belly's color) for each image. That additional information enables us to measure the correlation between visual attributes and the router's expert selection.
\vspace{-3mm}

\paragraph{Visual Attributes \& Experts Correlation}
We choose GMoE-S/16 with $6$ experts in each MoE layer. After training the model on CUB-DG, we perform forward passing with training images and save the routers top-1 selection. Since the MoE model routes patches instead of images and CUB-DG has provided the location of visual attributes, we first match a visual attribute with its nearest $9$ patches, and then correlate the visual attribute and its current $9$ patches' experts. In Fig. \ref{fig:correlation}, we show the 2D histogram correlation between the selected experts and attributes. Without any supervision signal for visual attributes, the ability to correlate attributes automatically emerges during training. Specifically, 1) each expert focus on distinct visual attributes; and 2) similar attributes are attended by the same expert (e.g., the left wing and the right wing are both attended by e4). This verifies the predictions by Theorem \ref{thm:moe}.
\vspace{-3mm}
\paragraph{Expert Selection} To further understand the expert selections of the entire image, we record the router's selection for each patch (in GMoE-S/16, we process an image into $16 \times 16$ patches), and then visualize the top-1 selections for each patch in Fig. \ref{fig:expert}. In this image, we mark the selected expert for each patch with a black number and draw different visual attributes of the bird (\eg beak and tail types) with large circles. We see the consistent relationship between Fig. \ref{fig:correlation} and Fig. \ref{fig:expert}. In detail, we see (1) experts $0$ and $2$ are consistently routed with patches in the \textit{background} area; (2) the \textit{left wing}, \textit{right wing}, \textit{beak} and \textit{tail} areas are attended by expert $3$; (3) the \textit{left leg} and \textit{right leg} areas are attended by expert $4$. More examples are given in Appendix \ref{subsec:cub_dg}.

\vspace{-2mm}
\begin{figure*}[ht]
    \centering
    \subfigure[CUB-DG examples.]
    {
        \includegraphics[width=0.24\columnwidth]{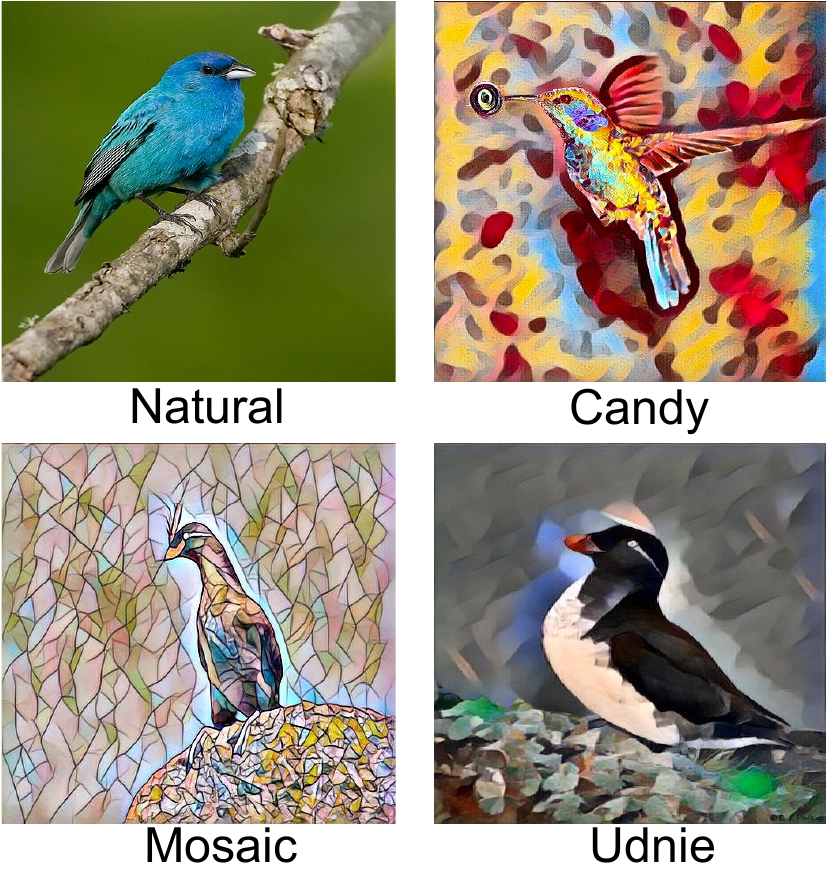}
        \label{fig:cub-dg}
    }\hfil
    \subfigure[Visual attributes.]
    {
        \includegraphics[width=0.31\columnwidth]{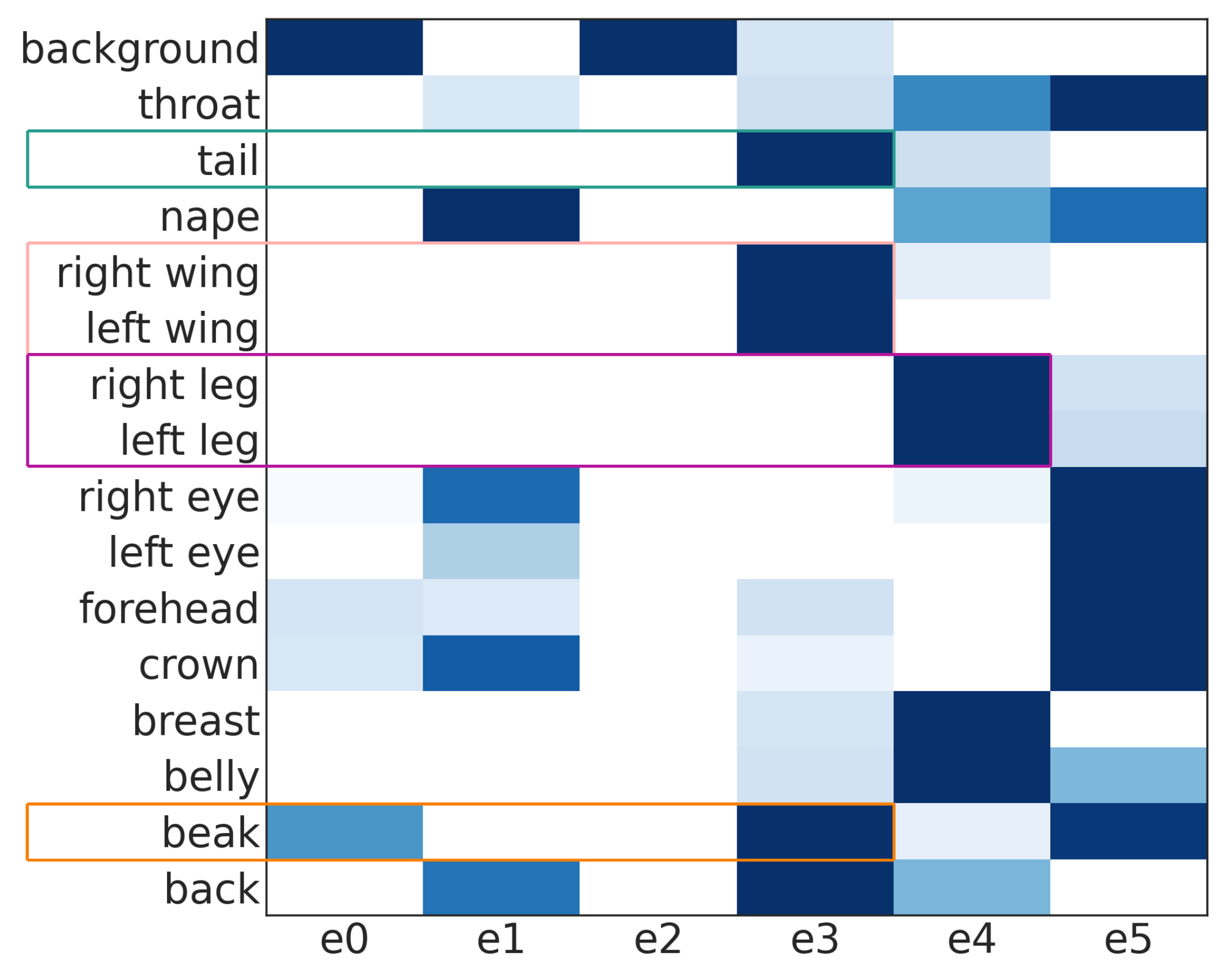}
        \label{fig:correlation}
    }\hfil
    \subfigure[Expert selection.]
    {
        \includegraphics[width=0.34\columnwidth]{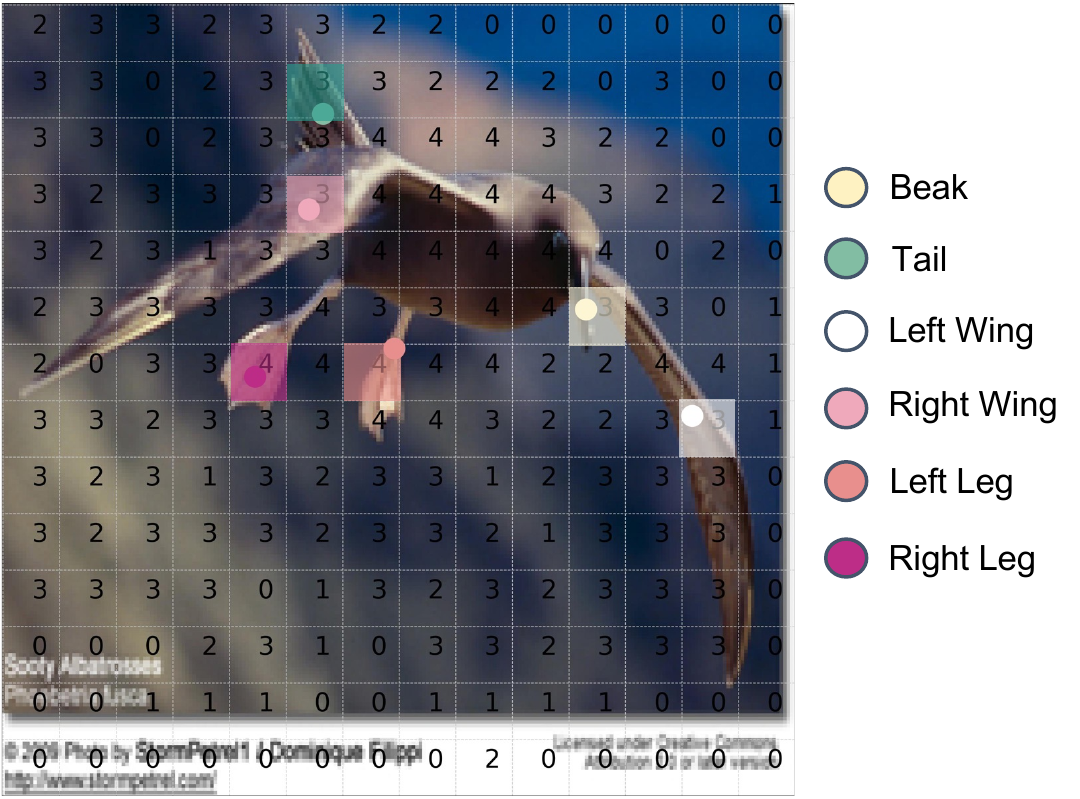}
        \label{fig:expert}
    }
    \vspace{-2mm}
    \caption{(a) Examples of CUB-DG datasets from four domains. (b)The $y$-axis corresponds to $15+1$ attributes ($15$ visual attributes + $1$ background). The $x$-axis corresponds to the selected expert id. (c) Finetuned GMoE's router decision of block-$10$. The image is from CUB-DG's natural domain.}
    \vspace{-4mm}
\end{figure*}
\vspace{-2mm}

\section{Conclusions}
\vspace{-3mm}
This paper is an initial step in exploring the impact of the backbone architecture in domain generalization. We proved that a network is more robust to distribution shifts if its architecture aligns well with the invariant correlation, which is verified on synthetic and real datasets. Based on our theoretical analysis, we proposed GMoE and demonstrated its superior performance on DomainBed. As for future directions, it is interesting to develop novel backbone architectures for DG based on algorithmic alignment and classic computer vision.

\bibliographystyle{iclr2023_conference}
\bibliography{egbib}

\begin{thebibliography}{91}
\providecommand{\natexlab}[1]{#1}
\providecommand{\url}[1]{\texttt{#1}}
\expandafter\ifx\csname urlstyle\endcsname\relax
  \providecommand{\doi}[1]{doi: #1}\else
  \providecommand{\doi}{doi: \begingroup \urlstyle{rm}\Url}\fi

\bibitem[Ahuja et~al.(2021)Ahuja, Caballero, Zhang, Gagnon{-}Audet, Bengio,
  Mitliagkas, and Rish]{DBLP:conf/nips/AhujaCZGBMR21}
Kartik Ahuja, Ethan Caballero, Dinghuai Zhang, Jean{-}Christophe
  Gagnon{-}Audet, Yoshua Bengio, Ioannis Mitliagkas, and Irina Rish.
\newblock Invariance principle meets information bottleneck for
  out-of-distribution generalization.
\newblock In Marc'Aurelio Ranzato, Alina Beygelzimer, Yann~N. Dauphin, Percy
  Liang, and Jennifer~Wortman Vaughan (eds.), \emph{Advances in Neural
  Information Processing Systems 34: Annual Conference on Neural Information
  Processing Systems 2021, NeurIPS 2021, December 6-14, 2021, virtual}, pp.\
  3438--3450, 2021.
\newblock URL
  \url{https://proceedings.neurips.cc/paper/2021/hash/1c336b8080f82bcc2cd2499b4c57261d-Abstract.html}.

\bibitem[Arjovsky et~al.(2019)Arjovsky, Bottou, Gulrajani, and
  Lopez{-}Paz]{DBLP:journals/corr/abs-1907-02893}
Mart{\'{\i}}n Arjovsky, L{\'{e}}on Bottou, Ishaan Gulrajani, and David
  Lopez{-}Paz.
\newblock Invariant risk minimization.
\newblock \emph{CoRR}, 2019.
\newblock URL \url{http://arxiv.org/abs/1907.02893}.

\bibitem[Arpit et~al.(2021)Arpit, Wang, Zhou, and
  Xiong]{DBLP:journals/corr/abs-2110-10832}
Devansh Arpit, Huan Wang, Yingbo Zhou, and Caiming Xiong.
\newblock Ensemble of averages: Improving model selection and boosting
  performance in domain generalization.
\newblock \emph{CoRR}, 2021.
\newblock URL \url{https://arxiv.org/abs/2110.10832}.

\bibitem[Bai et~al.(2021)Bai, Zhou, Hong, Ye, Chan, and Li]{bai2021ood}
Haoyue Bai, Fengwei Zhou, Lanqing Hong, Nanyang Ye, S-H~Gary Chan, and Zhenguo
  Li.
\newblock Nas-ood: Neural architecture search for out-of-distribution
  generalization.
\newblock In \emph{Proceedings of the IEEE/CVF International Conference on
  Computer Vision}, pp.\  8320--8329, 2021.

\bibitem[Beery et~al.(2018)Beery, Horn, and Perona]{DBLP:conf/eccv/BeeryHP18}
Sara Beery, Grant~Van Horn, and Pietro Perona.
\newblock Recognition in terra incognita.
\newblock In \emph{The European Conference on Computer Vision (ECCV)}, 2018.
\newblock \doi{10.1007/978-3-030-01270-0\_28}.
\newblock URL \url{https://doi.org/10.1007/978-3-030-01270-0\_28}.

\bibitem[Bui et~al.(2021)Bui, Tran, Tran, and
  Phung]{DBLP:journals/corr/abs-2110-09410}
Manh{-}Ha Bui, Toan Tran, Anh~Tuan Tran, and Dinh Phung.
\newblock Exploiting domain-specific features to enhance domain generalization.
\newblock \emph{CoRR}, 2021.
\newblock URL \url{https://arxiv.org/abs/2110.09410}.

\bibitem[Cha et~al.(2021{\natexlab{a}})Cha, Chun, Lee, Cho, Park, Lee, and
  Park]{DBLP:conf/nips/ChaCLCPLP21}
Junbum Cha, Sanghyuk Chun, Kyungjae Lee, Han{-}Cheol Cho, Seunghyun Park,
  Yunsung Lee, and Sungrae Park.
\newblock {SWAD:} domain generalization by seeking flat minima.
\newblock In \emph{Advances in Neural Information Processing Systems},
  2021{\natexlab{a}}.
\newblock URL
  \url{https://proceedings.neurips.cc/paper/2021/hash/bcb41ccdc4363c6848a1d760f26c28a0-Abstract.html}.

\bibitem[Cha et~al.(2021{\natexlab{b}})Cha, Chun, Lee, Cho, Park, Lee, and
  Park]{cha2021swad}
Junbum Cha, Sanghyuk Chun, Kyungjae Lee, Han-Cheol Cho, Seunghyun Park, Yunsung
  Lee, and Sungrae Park.
\newblock Swad: Domain generalization by seeking flat minima.
\newblock \emph{Advances in Neural Information Processing Systems},
  2021{\natexlab{b}}.

\bibitem[Cha et~al.(2022)Cha, Lee, Park, and
  Chun]{DBLP:journals/corr/abs-2203-10789}
Junbum Cha, Kyungjae Lee, Sungrae Park, and Sanghyuk Chun.
\newblock Domain generalization by mutual-information regularization with
  pre-trained models.
\newblock \emph{CoRR}, 2022.
\newblock \doi{10.48550/arXiv.2203.10789}.
\newblock URL \url{https://doi.org/10.48550/arXiv.2203.10789}.

\bibitem[Chen et~al.(2022)Chen, Deng, Wu, Gu, and
  Li]{DBLP:journals/corr/abs-2208-02813}
Zixiang Chen, Yihe Deng, Yue Wu, Quanquan Gu, and Yuanzhi Li.
\newblock Towards understanding mixture of experts in deep learning.
\newblock \emph{CoRR}, abs/2208.02813, 2022.
\newblock \doi{10.48550/arXiv.2208.02813}.
\newblock URL \url{https://doi.org/10.48550/arXiv.2208.02813}.

\bibitem[Chi et~al.(2022)Chi, Dong, Huang, Dai, Ma, Patra, Singhal, Bajaj,
  Song, and Wei]{DBLP:journals/corr/abs-2204-09179}
Zewen Chi, Li~Dong, Shaohan Huang, Damai Dai, Shuming Ma, Barun Patra, Saksham
  Singhal, Payal Bajaj, Xia Song, and Furu Wei.
\newblock On the representation collapse of sparse mixture of experts.
\newblock \emph{CoRR}, abs/2204.09179, 2022.
\newblock \doi{10.48550/arXiv.2204.09179}.
\newblock URL \url{https://doi.org/10.48550/arXiv.2204.09179}.

\bibitem[Christie et~al.(2018)Christie, Fendley, Wilson, and
  Mukherjee]{DBLP:conf/cvpr/ChristieFWM18}
Gordon~A. Christie, Neil Fendley, James Wilson, and Ryan Mukherjee.
\newblock Functional map of the world.
\newblock In \emph{2018 {IEEE} Conference on Computer Vision and Pattern
  Recognition, {CVPR} 2018, Salt Lake City, UT, USA, June 18-22, 2018}, pp.\
  6172--6180. Computer Vision Foundation / {IEEE} Computer Society, 2018.
\newblock \doi{10.1109/CVPR.2018.00646}.
\newblock URL
  \url{http://openaccess.thecvf.com/content\_cvpr\_2018/html/Christie\_Functional\_Map\_of\_CVPR\_2018\_paper.html}.

\bibitem[Cordonnier et~al.(2020)Cordonnier, Loukas, and
  Jaggi]{DBLP:journals/corr/abs-2006-16362}
Jean{-}Baptiste Cordonnier, Andreas Loukas, and Martin Jaggi.
\newblock Multi-head attention: Collaborate instead of concatenate.
\newblock \emph{CoRR}, 2020.
\newblock URL \url{https://arxiv.org/abs/2006.16362}.

\bibitem[Cruz et~al.(2020)Cruz, Wasenm{\"{u}}ller, Beise, Stifter, and
  Stricker]{DBLP:conf/wacv/CruzWBSS20}
Steve Dias~Da Cruz, Oliver Wasenm{\"{u}}ller, Hans{-}Peter Beise, Thomas
  Stifter, and Didier Stricker.
\newblock {SVIRO:} synthetic vehicle interior rear seat occupancy dataset and
  benchmark.
\newblock In \emph{{IEEE} Winter Conference on Applications of Computer Vision,
  {WACV} 2020, Snowmass Village, CO, USA, March 1-5, 2020}, pp.\  962--971.
  {IEEE}, 2020.
\newblock \doi{10.1109/WACV45572.2020.9093315}.
\newblock URL \url{https://doi.org/10.1109/WACV45572.2020.9093315}.

\bibitem[Dosovitskiy et~al.(2021)Dosovitskiy, Beyer, Kolesnikov, Weissenborn,
  Zhai, Unterthiner, Dehghani, Minderer, Heigold, Gelly, Uszkoreit, and
  Houlsby]{DBLP:conf/iclr/DosovitskiyB0WZ21}
Alexey Dosovitskiy, Lucas Beyer, Alexander Kolesnikov, Dirk Weissenborn,
  Xiaohua Zhai, Thomas Unterthiner, Mostafa Dehghani, Matthias Minderer, Georg
  Heigold, Sylvain Gelly, Jakob Uszkoreit, and Neil Houlsby.
\newblock An image is worth 16x16 words: Transformers for image recognition at
  scale.
\newblock In \emph{9th International Conference on Learning Representations,
  {ICLR} 2021, Virtual Event, Austria, May 3-7, 2021}, 2021.
\newblock URL \url{https://openreview.net/forum?id=YicbFdNTTy}.

\bibitem[Dou et~al.(2019)Dou, Coelho~de Castro, Kamnitsas, and
  Glocker]{dou2019domain}
Qi~Dou, Daniel Coelho~de Castro, Konstantinos Kamnitsas, and Ben Glocker.
\newblock Domain generalization via model-agnostic learning of semantic
  features.
\newblock \emph{Advances in Neural Information Processing Systems}, 32, 2019.

\bibitem[Fang et~al.(2013)Fang, Xu, and Rockmore]{DBLP:conf/iccv/FangXR13}
Chen Fang, Ye~Xu, and Daniel~N. Rockmore.
\newblock Unbiased metric learning: On the utilization of multiple datasets and
  web images for softening bias.
\newblock In \emph{{IEEE} International Conference on Computer Vision, {ICCV}
  2013, Sydney, Australia, December 1-8, 2013}, pp.\  1657--1664. {IEEE}
  Computer Society, 2013.
\newblock \doi{10.1109/ICCV.2013.208}.
\newblock URL \url{https://doi.org/10.1109/ICCV.2013.208}.

\bibitem[Fedus et~al.(2022)Fedus, Dean, and Zoph]{fedus2022review}
William Fedus, Jeff Dean, and Barret Zoph.
\newblock A review of sparse expert models in deep learning.
\newblock \emph{arXiv preprint arXiv:2209.01667}, 2022.

\bibitem[Ferrari \& Zisserman(2007)Ferrari and Zisserman]{ferrari2007learning}
Vittorio Ferrari and Andrew Zisserman.
\newblock Learning visual attributes.
\newblock \emph{Advances in neural information processing systems}, 20, 2007.

\bibitem[Ganin et~al.(2016)Ganin, Ustinova, Ajakan, Germain, Larochelle,
  Laviolette, Marchand, and Lempitsky]{DBLP:journals/jmlr/GaninUAGLLML16}
Yaroslav Ganin, Evgeniya Ustinova, Hana Ajakan, Pascal Germain, Hugo
  Larochelle, Fran{\c{c}}ois Laviolette, Mario Marchand, and Victor~S.
  Lempitsky.
\newblock Domain-adversarial training of neural networks.
\newblock \emph{J. Mach. Learn. Res.}, 2016.
\newblock URL \url{http://jmlr.org/papers/v17/15-239.html}.

\bibitem[Geirhos et~al.(2018)Geirhos, Rubisch, Michaelis, Bethge, Wichmann, and
  Brendel]{geirhos2018imagenet}
Robert Geirhos, Patricia Rubisch, Claudio Michaelis, Matthias Bethge, Felix~A
  Wichmann, and Wieland Brendel.
\newblock Imagenet-trained cnns are biased towards texture; increasing shape
  bias improves accuracy and robustness.
\newblock \emph{arXiv preprint arXiv:1811.12231}, 2018.

\bibitem[Geirhos et~al.(2021)Geirhos, Narayanappa, Mitzkus, Thieringer, Bethge,
  Wichmann, and Brendel]{geirhos2021partial}
Robert Geirhos, Kantharaju Narayanappa, Benjamin Mitzkus, Tizian Thieringer,
  Matthias Bethge, Felix~A Wichmann, and Wieland Brendel.
\newblock Partial success in closing the gap between human and machine vision.
\newblock In \emph{{Advances in Neural Information Processing Systems 34}},
  2021.

\bibitem[Gulrajani \& Lopez{-}Paz(2021)Gulrajani and
  Lopez{-}Paz]{DBLP:conf/iclr/GulrajaniL21}
Ishaan Gulrajani and David Lopez{-}Paz.
\newblock In search of lost domain generalization.
\newblock In \emph{9th International Conference on Learning Representations,
  {ICLR}}, 2021.
\newblock URL \url{https://openreview.net/forum?id=lQdXeXDoWtI}.

\bibitem[Guo et~al.(2018)Guo, Shah, and Barzilay]{guo2018multi}
Jiang Guo, Darsh~J Shah, and Regina Barzilay.
\newblock Multi-source domain adaptation with mixture of experts.
\newblock \emph{arXiv preprint arXiv:1809.02256}, 2018.

\bibitem[He et~al.(2016)He, Zhang, Ren, and Sun]{he2016deep}
Kaiming He, Xiangyu Zhang, Shaoqing Ren, and Jian Sun.
\newblock Deep residual learning for image recognition.
\newblock In \emph{Proceedings of the IEEE conference on computer vision and
  pattern recognition}, 2016.

\bibitem[Hoffman et~al.(2018)Hoffman, Mohri, and
  Zhang]{DBLP:conf/nips/HoffmanMZ18}
Judy Hoffman, Mehryar Mohri, and Ningshan Zhang.
\newblock Algorithms and theory for multiple-source adaptation.
\newblock In Samy Bengio, Hanna~M. Wallach, Hugo Larochelle, Kristen Grauman,
  Nicol{\`{o}} Cesa{-}Bianchi, and Roman Garnett (eds.), \emph{Advances in
  Neural Information Processing Systems 31: Annual Conference on Neural
  Information Processing Systems 2018, NeurIPS 2018, December 3-8, 2018,
  Montr{\'{e}}al, Canada}, pp.\  8256--8266, 2018.
\newblock URL
  \url{https://proceedings.neurips.cc/paper/2018/hash/2e2079d63348233d91cad1fa9b1361e9-Abstract.html}.

\bibitem[Hu et~al.(2018)Hu, Shen, and Sun]{hu2018squeeze}
Jie Hu, Li~Shen, and Gang Sun.
\newblock Squeeze-and-excitation networks.
\newblock In \emph{Proceedings of the IEEE conference on computer vision and
  pattern recognition}, pp.\  7132--7141, 2018.

\bibitem[Hu et~al.(1997)Hu, Palreddy, and Tompkins]{hu1997patient}
Yu~Hen Hu, Surekha Palreddy, and Willis~J Tompkins.
\newblock A patient-adaptable ecg beat classifier using a mixture of experts
  approach.
\newblock \emph{IEEE transactions on biomedical engineering}, 1997.

\bibitem[Jacobs et~al.(1991)Jacobs, Jordan, Nowlan, and
  Hinton]{DBLP:journals/neco/JacobsJNH91}
Robert~A. Jacobs, Michael~I. Jordan, Steven~J. Nowlan, and Geoffrey~E. Hinton.
\newblock Adaptive mixtures of local experts.
\newblock \emph{Neural Comput.}, 1991.
\newblock \doi{10.1162/neco.1991.3.1.79}.
\newblock URL \url{https://doi.org/10.1162/neco.1991.3.1.79}.

\bibitem[Jordan \& Jacobs(1994)Jordan and Jacobs]{DBLP:journals/neco/JordanJ94}
Michael~I. Jordan and Robert~A. Jacobs.
\newblock Hierarchical mixtures of experts and the {EM} algorithm.
\newblock \emph{Neural Comput.}, 1994.
\newblock \doi{10.1162/neco.1994.6.2.181}.
\newblock URL \url{https://doi.org/10.1162/neco.1994.6.2.181}.

\bibitem[Kingma \& Ba(2015)Kingma and Ba]{DBLP:journals/corr/KingmaB14}
Diederik~P. Kingma and Jimmy Ba.
\newblock Adam: {A} method for stochastic optimization.
\newblock In Yoshua Bengio and Yann LeCun (eds.), \emph{3rd International
  Conference on Learning Representations, {ICLR} 2015, San Diego, CA, USA, May
  7-9, 2015, Conference Track Proceedings}, 2015.
\newblock URL \url{http://arxiv.org/abs/1412.6980}.

\bibitem[Koh et~al.(2021)Koh, Sagawa, Marklund, Xie, Zhang, Balsubramani, Hu,
  Yasunaga, Phillips, Gao, Lee, David, Stavness, Guo, Earnshaw, Haque, Beery,
  Leskovec, Kundaje, Pierson, Levine, Finn, and
  Liang]{DBLP:conf/icml/KohSMXZBHYPGLDS21}
Pang~Wei Koh, Shiori Sagawa, Henrik Marklund, Sang~Michael Xie, Marvin Zhang,
  Akshay Balsubramani, Weihua Hu, Michihiro Yasunaga, Richard~Lanas Phillips,
  Irena Gao, Tony Lee, Etienne David, Ian Stavness, Wei Guo, Berton Earnshaw,
  Imran Haque, Sara~M. Beery, Jure Leskovec, Anshul Kundaje, Emma Pierson,
  Sergey Levine, Chelsea Finn, and Percy Liang.
\newblock {WILDS:} {A} benchmark of in-the-wild distribution shifts.
\newblock In Marina Meila and Tong Zhang (eds.), \emph{Proceedings of the 38th
  International Conference on Machine Learning, {ICML} 2021, 18-24 July 2021,
  Virtual Event}, volume 139 of \emph{Proceedings of Machine Learning
  Research}, pp.\  5637--5664. {PMLR}, 2021.
\newblock URL \url{http://proceedings.mlr.press/v139/koh21a.html}.

\bibitem[Krueger et~al.(2021)Krueger, Caballero, Jacobsen, Zhang, Binas, Zhang,
  Priol, and Courville]{DBLP:conf/icml/KruegerCJ0BZPC21}
David Krueger, Ethan Caballero, J{\"{o}}rn{-}Henrik Jacobsen, Amy Zhang,
  Jonathan Binas, Dinghuai Zhang, R{\'{e}}mi~Le Priol, and Aaron~C. Courville.
\newblock Out-of-distribution generalization via risk extrapolation (rex).
\newblock In Marina Meila and Tong Zhang (eds.), \emph{Proceedings of the 38th
  International Conference on Machine Learning, {ICML} 2021, 18-24 July 2021,
  Virtual Event}, volume 139 of \emph{Proceedings of Machine Learning
  Research}, pp.\  5815--5826. {PMLR}, 2021.
\newblock URL \url{http://proceedings.mlr.press/v139/krueger21a.html}.

\bibitem[Lepikhin et~al.(2021)Lepikhin, Lee, Xu, Chen, Firat, Huang, Krikun,
  Shazeer, and Chen]{DBLP:conf/iclr/LepikhinLXCFHKS21}
Dmitry Lepikhin, HyoukJoong Lee, Yuanzhong Xu, Dehao Chen, Orhan Firat, Yanping
  Huang, Maxim Krikun, Noam Shazeer, and Zhifeng Chen.
\newblock Gshard: Scaling giant models with conditional computation and
  automatic sharding.
\newblock In \emph{9th International Conference on Learning Representations,
  {ICLR}}, 2021.
\newblock URL \url{https://openreview.net/forum?id=qrwe7XHTmYb}.

\bibitem[Li et~al.(2017)Li, Yang, Song, and Hospedales]{DBLP:conf/iccv/LiYSH17}
Da~Li, Yongxin Yang, Yi{-}Zhe Song, and Timothy~M. Hospedales.
\newblock Deeper, broader and artier domain generalization.
\newblock In \emph{{IEEE} International Conference on Computer Vision, {ICCV}
  2017, Venice, Italy, October 22-29, 2017}, pp.\  5543--5551. {IEEE} Computer
  Society, 2017.
\newblock \doi{10.1109/ICCV.2017.591}.
\newblock URL \url{https://doi.org/10.1109/ICCV.2017.591}.

\bibitem[Li et~al.(2018{\natexlab{a}})Li, Yang, Song, and
  Hospedales]{DBLP:conf/aaai/LiYSH18}
Da~Li, Yongxin Yang, Yi{-}Zhe Song, and Timothy~M. Hospedales.
\newblock Learning to generalize: Meta-learning for domain generalization.
\newblock In \emph{Proceedings of the Thirty-Second {AAAI} Conference on
  Artificial Intelligence}, 2018{\natexlab{a}}.
\newblock URL
  \url{https://www.aaai.org/ocs/index.php/AAAI/AAAI18/paper/view/16067}.

\bibitem[Li et~al.(2018{\natexlab{b}})Li, Pan, Wang, and
  Kot]{DBLP:conf/cvpr/LiPWK18}
Haoliang Li, Sinno~Jialin Pan, Shiqi Wang, and Alex~C. Kot.
\newblock Domain generalization with adversarial feature learning.
\newblock In \emph{2018 {IEEE} Conference on Computer Vision and Pattern
  Recognition}, 2018{\natexlab{b}}.
\newblock \doi{10.1109/CVPR.2018.00566}.
\newblock URL
  \url{http://openaccess.thecvf.com/content\_cvpr\_2018/html/Li\_Domain\_Generalization\_With\_CVPR\_2018\_paper.html}.

\bibitem[Li et~al.(2021)Li, Zhang, Xu, Dickerson, and Ba]{li2021does}
Jingling Li, Mozhi Zhang, Keyulu Xu, John Dickerson, and Jimmy Ba.
\newblock How does a neural network's architecture impact its robustness to
  noisy labels?
\newblock \emph{Advances in Neural Information Processing Systems},
  34:\penalty0 9788--9803, 2021.

\bibitem[Li et~al.(2019)Li, Yang, Zhou, and Hospedales]{li2019feature}
Yiying Li, Yongxin Yang, Wei Zhou, and Timothy Hospedales.
\newblock Feature-critic networks for heterogeneous domain generalization.
\newblock In \emph{International Conference on Machine Learning}, pp.\
  3915--3924. PMLR, 2019.

\bibitem[Li et~al.(2022)Li, Ren, Jiang, Li, Zhang, and
  Li]{DBLP:journals/corr/abs-2203-04600}
Ziyue Li, Kan Ren, Xinyang Jiang, Bo~Li, Haipeng Zhang, and Dongsheng Li.
\newblock Domain generalization using pretrained models without fine-tuning.
\newblock \emph{CoRR}, 2022.
\newblock \doi{10.48550/arXiv.2203.04600}.
\newblock URL \url{https://doi.org/10.48550/arXiv.2203.04600}.

\bibitem[Liu et~al.(2018)Liu, Liu, Yeh, and Wang]{DBLP:conf/nips/LiuLYW18}
Alexander~H. Liu, Yen{-}Cheng Liu, Yu{-}Ying Yeh, and Yu{-}Chiang~Frank Wang.
\newblock A unified feature disentangler for multi-domain image translation and
  manipulation.
\newblock In \emph{Advances in Neural Information Processing Systems}, 2018.
\newblock URL
  \url{https://proceedings.neurips.cc/paper/2018/hash/84438b7aae55a0638073ef798e50b4ef-Abstract.html}.

\bibitem[Liu et~al.(2021)Liu, Lin, Cao, Hu, Wei, Zhang, Lin, and
  Guo]{DBLP:conf/iccv/LiuL00W0LG21}
Ze~Liu, Yutong Lin, Yue Cao, Han Hu, Yixuan Wei, Zheng Zhang, Stephen Lin, and
  Baining Guo.
\newblock Swin transformer: Hierarchical vision transformer using shifted
  windows.
\newblock In \emph{2021 {IEEE/CVF} International Conference on Computer Vision,
  {ICCV}}, 2021.
\newblock \doi{10.1109/ICCV48922.2021.00986}.
\newblock URL \url{https://doi.org/10.1109/ICCV48922.2021.00986}.

\bibitem[Ma et~al.(2022)Ma, Tsao, and Shum]{ma2022principles}
Yi~Ma, Doris Tsao, and Heung-Yeung Shum.
\newblock On the principles of parsimony and self-consistency for the emergence
  of intelligence.
\newblock \emph{Frontiers of Information Technology \& Electronic Engineering},
  pp.\  1--26, 2022.

\bibitem[Mancini et~al.(2018)Mancini, Bul{\`{o}}, Caputo, and
  Ricci]{DBLP:conf/icip/ManciniBC018}
Massimiliano Mancini, Samuel~Rota Bul{\`{o}}, Barbara Caputo, and Elisa Ricci.
\newblock Best sources forward: Domain generalization through source-specific
  nets.
\newblock In \emph{2018 {IEEE} International Conference on Image Processing},
  2018.
\newblock \doi{10.1109/ICIP.2018.8451318}.
\newblock URL \url{https://doi.org/10.1109/ICIP.2018.8451318}.

\bibitem[Matthey et~al.(2017)Matthey, Higgins, Hassabis, and
  Lerchner]{dsprites17}
Loic Matthey, Irina Higgins, Demis Hassabis, and Alexander Lerchner.
\newblock dsprites: Disentanglement testing sprites dataset.
\newblock https://github.com/deepmind/dsprites-dataset/, 2017.

\bibitem[M{\"u}ller \& Hutter(2021)M{\"u}ller and
  Hutter]{muller2021trivialaugment}
Samuel~G M{\"u}ller and Frank Hutter.
\newblock Trivialaugment: Tuning-free yet state-of-the-art data augmentation.
\newblock In \emph{Proceedings of the IEEE/CVF International Conference on
  Computer Vision}, pp.\  774--782, 2021.

\bibitem[Mustafa et~al.(2022)Mustafa, Riquelme, Puigcerver, Jenatton, and
  Houlsby]{mustafa2022multimodal}
Basil Mustafa, Carlos Riquelme, Joan Puigcerver, Rodolphe Jenatton, and Neil
  Houlsby.
\newblock Multimodal contrastive learning with limoe: the language-image
  mixture of experts.
\newblock \emph{arXiv preprint arXiv:2206.02770}, 2022.

\bibitem[Nagarajan et~al.(2020)Nagarajan, Andreassen, and
  Neyshabur]{nagarajan2020understanding}
Vaishnavh Nagarajan, Anders Andreassen, and Behnam Neyshabur.
\newblock Understanding the failure modes of out-of-distribution
  generalization.
\newblock \emph{arXiv preprint arXiv:2010.15775}, 2020.

\bibitem[Nazari \& Kovashka(2020)Nazari and Kovashka]{DBLP:conf/eccv/NazariK20}
Narges~Honarvar Nazari and Adriana Kovashka.
\newblock Domain generalization using shape representation.
\newblock In \emph{Proceedings of European Conference on Computer Vision -
  {ECCV} 2020 Workshops}, 2020.
\newblock \doi{10.1007/978-3-030-66415-2\_45}.
\newblock URL \url{https://doi.org/10.1007/978-3-030-66415-2\_45}.

\bibitem[Park \& Kim(2022)Park and Kim]{park2022vision}
Namuk Park and Songkuk Kim.
\newblock How do vision transformers work?
\newblock \emph{arXiv preprint arXiv:2202.06709}, 2022.

\bibitem[Peng et~al.(2019)Peng, Bai, Xia, Huang, Saenko, and
  Wang]{DBLP:conf/iccv/PengBXHSW19}
Xingchao Peng, Qinxun Bai, Xide Xia, Zijun Huang, Kate Saenko, and Bo~Wang.
\newblock Moment matching for multi-source domain adaptation.
\newblock In \emph{2019 {IEEE/CVF} International Conference on Computer Vision,
  {ICCV}}, 2019.
\newblock \doi{10.1109/ICCV.2019.00149}.
\newblock URL \url{https://doi.org/10.1109/ICCV.2019.00149}.

\bibitem[Qiao et~al.(2020)Qiao, Zhao, and Peng]{DBLP:conf/cvpr/QiaoZP20}
Fengchun Qiao, Long Zhao, and Xi~Peng.
\newblock Learning to learn single domain generalization.
\newblock In \emph{2020 {IEEE/CVF} Conference on Computer Vision and Pattern
  Recognition, {CVPR} 2020, Seattle, WA, USA, June 13-19, 2020}, 2020.
\newblock \doi{10.1109/CVPR42600.2020.01257}.
\newblock URL \url{https://doi.org/10.1109/CVPR42600.2020.01257}.

\bibitem[Qin et~al.(2021)Qin, Zhang, Chen, Lakshminarayanan, Beutel, and
  Wang]{DBLP:journals/corr/abs-2110-07858}
Yao Qin, Chiyuan Zhang, Ting Chen, Balaji Lakshminarayanan, Alex Beutel, and
  Xuezhi Wang.
\newblock Understanding and improving robustness of vision transformers through
  patch-based negative augmentation.
\newblock \emph{CoRR}, 2021.
\newblock URL \url{https://arxiv.org/abs/2110.07858}.

\bibitem[Rahaman et~al.(2021)Rahaman, Gondal, Joshi, Gehler, Bengio, Locatello,
  and Sch{\"o}lkopf]{rahaman2021dynamic}
Nasim Rahaman, Muhammad~Waleed Gondal, Shruti Joshi, Peter Gehler, Yoshua
  Bengio, Francesco Locatello, and Bernhard Sch{\"o}lkopf.
\newblock Dynamic inference with neural interpreters.
\newblock \emph{Advances in Neural Information Processing Systems},
  34:\penalty0 10985--10998, 2021.

\bibitem[Rame et~al.(2021)Rame, Dancette, and Cord]{rame2021fishr}
Alexandre Rame, Corentin Dancette, and Matthieu Cord.
\newblock Fishr: Invariant gradient variances for out-of-distribution
  generalization.
\newblock \emph{arXiv preprint arXiv:2109.02934}, 2021.

\bibitem[Recht et~al.(2019)Recht, Roelofs, Schmidt, and
  Shankar]{recht2019imagenet}
Benjamin Recht, Rebecca Roelofs, Ludwig Schmidt, and Vaishaal Shankar.
\newblock Do imagenet classifiers generalize to imagenet?
\newblock In \emph{International Conference on Machine Learning}, pp.\
  5389--5400. PMLR, 2019.

\bibitem[Riemer et~al.(2019)Riemer, Cases, Ajemian, Liu, Rish, Tu, and
  Tesauro]{DBLP:conf/iclr/RiemerCALRTT19}
Matthew Riemer, Ignacio Cases, Robert Ajemian, Miao Liu, Irina Rish, Yuhai Tu,
  and Gerald Tesauro.
\newblock Learning to learn without forgetting by maximizing transfer and
  minimizing interference.
\newblock In \emph{Proceedings of The 7th International Conference on Learning
  Representations}, 2019.
\newblock URL \url{https://openreview.net/forum?id=B1gTShAct7}.

\bibitem[Riquelme et~al.(2021)Riquelme, Puigcerver, Mustafa, Neumann, Jenatton,
  Pinto, Keysers, and Houlsby]{DBLP:conf/nips/RiquelmePMNJPKH21}
Carlos Riquelme, Joan Puigcerver, Basil Mustafa, Maxim Neumann, Rodolphe
  Jenatton, Andr{\'{e}}~Susano Pinto, Daniel Keysers, and Neil Houlsby.
\newblock Scaling vision with sparse mixture of experts.
\newblock In \emph{Advances in Neural Information Processing Systems, NeurIPS
  2021, December 6-14, 2021}, 2021.
\newblock URL
  \url{https://proceedings.neurips.cc/paper/2021/hash/48237d9f2dea8c74c2a72126cf63d933-Abstract.html}.

\bibitem[Seg{\`{u}} et~al.(2020)Seg{\`{u}}, Tonioni, and
  Tombari]{DBLP:journals/corr/abs-2011-12672}
Mattia Seg{\`{u}}, Alessio Tonioni, and Federico Tombari.
\newblock Batch normalization embeddings for deep domain generalization.
\newblock \emph{CoRR}, 2020.
\newblock URL \url{https://arxiv.org/abs/2011.12672}.

\bibitem[Shazeer et~al.(2017)Shazeer, Mirhoseini, Maziarz, Davis, Le, Hinton,
  and Dean]{DBLP:conf/iclr/ShazeerMMDLHD17}
Noam Shazeer, Azalia Mirhoseini, Krzysztof Maziarz, Andy Davis, Quoc~V. Le,
  Geoffrey~E. Hinton, and Jeff Dean.
\newblock Outrageously large neural networks: The sparsely-gated
  mixture-of-experts layer.
\newblock In \emph{5th International Conference on Learning Representations,
  {ICLR}}, 2017.
\newblock URL \url{https://openreview.net/forum?id=B1ckMDqlg}.

\bibitem[Shi et~al.(2021)Shi, Seely, Torr, Siddharth, Hannun, Usunier, and
  Synnaeve]{DBLP:journals/corr/abs-2104-09937}
Yuge Shi, Jeffrey Seely, Philip H.~S. Torr, N.~Siddharth, Awni~Y. Hannun,
  Nicolas Usunier, and Gabriel Synnaeve.
\newblock Gradient matching for domain generalization.
\newblock \emph{CoRR}, 2021.
\newblock URL \url{https://arxiv.org/abs/2104.09937}.

\bibitem[Sivaprasad et~al.(2021)Sivaprasad, Goindani, Garg, and
  Gandhi]{sivaprasad2021reappraising}
Sarath Sivaprasad, Akshay Goindani, Vaibhav Garg, and Vineet Gandhi.
\newblock Reappraising domain generalization in neural networks.
\newblock \emph{arXiv preprint arXiv:2110.07981}, 2021.

\bibitem[Sun \& Saenko(2016)Sun and Saenko]{DBLP:conf/eccv/SunS16}
Baochen Sun and Kate Saenko.
\newblock Deep {CORAL:} correlation alignment for deep domain adaptation.
\newblock In \emph{Computer Vision - {ECCV} 2016 Workshops}, 2016.
\newblock \doi{10.1007/978-3-319-49409-8\_35}.
\newblock URL \url{https://doi.org/10.1007/978-3-319-49409-8\_35}.

\bibitem[Szegedy et~al.(2016)Szegedy, Vanhoucke, Ioffe, Shlens, and
  Wojna]{szegedy2016rethinking}
Christian Szegedy, Vincent Vanhoucke, Sergey Ioffe, Jon Shlens, and Zbigniew
  Wojna.
\newblock Rethinking the inception architecture for computer vision.
\newblock In \emph{Proceedings of the IEEE conference on computer vision and
  pattern recognition}, pp.\  2818--2826, 2016.

\bibitem[Tani \& Nolfi(1999)Tani and Nolfi]{tani1999learning}
Jun Tani and Stefano Nolfi.
\newblock Learning to perceive the world as articulated: an approach for
  hierarchical learning in sensory-motor systems.
\newblock \emph{Neural Networks}, 1999.

\bibitem[Touvron et~al.(2019)Touvron, Vedaldi, Douze, and
  J{\'e}gou]{touvron2019fixing}
Hugo Touvron, Andrea Vedaldi, Matthijs Douze, and Herv{\'e} J{\'e}gou.
\newblock Fixing the train-test resolution discrepancy.
\newblock \emph{Advances in neural information processing systems}, 32, 2019.

\bibitem[Touvron et~al.(2021)Touvron, Cord, Douze, Massa, Sablayrolles, and
  J{\'{e}}gou]{DBLP:conf/icml/TouvronCDMSJ21}
Hugo Touvron, Matthieu Cord, Matthijs Douze, Francisco Massa, Alexandre
  Sablayrolles, and Herv{\'{e}} J{\'{e}}gou.
\newblock Training data-efficient image transformers {\&} distillation through
  attention.
\newblock In Marina Meila and Tong Zhang (eds.), \emph{Proceedings of the 38th
  International Conference on Machine Learning, {ICML} 2021, 18-24 July 2021,
  Virtual Event}, volume 139 of \emph{Proceedings of Machine Learning
  Research}, pp.\  10347--10357. {PMLR}, 2021.
\newblock URL \url{http://proceedings.mlr.press/v139/touvron21a.html}.

\bibitem[Valiant(1984)]{valiant1984theory}
Leslie~G Valiant.
\newblock A theory of the learnable.
\newblock \emph{Communications of the ACM}, 27\penalty0 (11):\penalty0
  1134--1142, 1984.

\bibitem[Vapnik(1991)]{vapnik1991principles}
Vladimir Vapnik.
\newblock Principles of risk minimization for learning theory.
\newblock \emph{Advances in neural information processing systems}, 4, 1991.

\bibitem[Vaswani et~al.(2017)Vaswani, Shazeer, Parmar, Uszkoreit, Jones, Gomez,
  Kaiser, and Polosukhin]{DBLP:conf/nips/VaswaniSPUJGKP17}
Ashish Vaswani, Noam Shazeer, Niki Parmar, Jakob Uszkoreit, Llion Jones,
  Aidan~N. Gomez, Lukasz Kaiser, and Illia Polosukhin.
\newblock Attention is all you need.
\newblock In \emph{Advances in Neural Information Processing Systems}, 2017.
\newblock URL
  \url{https://proceedings.neurips.cc/paper/2017/hash/3f5ee243547dee91fbd053c1c4a845aa-Abstract.html}.

\bibitem[Venkateswara et~al.(2017)Venkateswara, Eusebio, Chakraborty, and
  Panchanathan]{DBLP:conf/cvpr/VenkateswaraECP17}
Hemanth Venkateswara, Jose Eusebio, Shayok Chakraborty, and Sethuraman
  Panchanathan.
\newblock Deep hashing network for unsupervised domain adaptation.
\newblock In \emph{2017 {IEEE} Conference on Computer Vision and Pattern
  Recognition, {CVPR} 2017, Honolulu, HI, USA, July 21-26, 2017}, pp.\
  5385--5394. {IEEE} Computer Society, 2017.
\newblock \doi{10.1109/CVPR.2017.572}.
\newblock URL \url{https://doi.org/10.1109/CVPR.2017.572}.

\bibitem[Wah et~al.(2011)Wah, Branson, Welinder, Perona, and
  Belongie]{wah2011caltech}
Catherine Wah, Steve Branson, Peter Welinder, Pietro Perona, and Serge
  Belongie.
\newblock The caltech-ucsd birds-200-2011 dataset.
\newblock \emph{California Institute of Technology}, 2011.

\bibitem[Wang et~al.(2022)Wang, Bao, Dong, Bjorck, Peng, Liu, Aggarwal,
  Mohammed, Singhal, Som, et~al.]{wang2022image}
Wenhui Wang, Hangbo Bao, Li~Dong, Johan Bjorck, Zhiliang Peng, Qiang Liu, Kriti
  Aggarwal, Owais~Khan Mohammed, Saksham Singhal, Subhojit Som, et~al.
\newblock Image as a foreign language: Beit pretraining for all vision and
  vision-language tasks.
\newblock \emph{arXiv preprint arXiv:2208.10442}, 2022.

\bibitem[Wang et~al.(2021)Wang, Li, Chau, and
  Kot]{DBLP:journals/corr/abs-2109-05826}
Yufei Wang, Haoliang Li, Lap{-}Pui Chau, and Alex~C. Kot.
\newblock Variational disentanglement for domain generalization.
\newblock \emph{CoRR}, 2021.
\newblock URL \url{https://arxiv.org/abs/2109.05826}.

\bibitem[Wiles et~al.(2021)Wiles, Gowal, Stimberg, Alvise-Rebuffi, Ktena,
  Cemgil, et~al.]{wiles2021fine}
Olivia Wiles, Sven Gowal, Florian Stimberg, Sylvestre Alvise-Rebuffi, Ira
  Ktena, Taylan Cemgil, et~al.
\newblock A fine-grained analysis on distribution shift.
\newblock \emph{arXiv preprint arXiv:2110.11328}, 2021.

\bibitem[Xie \& Jiang(2021)Xie and Jiang]{xie2021batch}
Tengyang Xie and Nan Jiang.
\newblock Batch value-function approximation with only realizability.
\newblock In \emph{International Conference on Machine Learning}, pp.\
  11404--11413. PMLR, 2021.

\bibitem[Xu et~al.(2020{\natexlab{a}})Xu, Li, Zhang, Du, Kawarabayashi, and
  Jegelka]{xu2019what}
Keyulu Xu, Jingling Li, Mozhi Zhang, Simon Du, Kenichi Kawarabayashi, and
  Stefanie Jegelka.
\newblock What can neural networks reason about?
\newblock In \emph{Proceedings of International Conference on Learning
  Representation}, Apr. 2020{\natexlab{a}}.

\bibitem[Xu et~al.(2020{\natexlab{b}})Xu, Zhang, Li, Du, Kawarabayashi, and
  Jegelka]{xu2020neural}
Keyulu Xu, Mozhi Zhang, Jingling Li, Simon~S Du, Ken-ichi Kawarabayashi, and
  Stefanie Jegelka.
\newblock How neural networks extrapolate: From feedforward to graph neural
  networks.
\newblock \emph{arXiv preprint arXiv:2009.11848}, 2020{\natexlab{b}}.

\bibitem[Ye et~al.(2021)Ye, Xie, Cai, Li, Li, and Wang]{ye2021towards}
Haotian Ye, Chuanlong Xie, Tianle Cai, Ruichen Li, Zhenguo Li, and Liwei Wang.
\newblock Towards a theoretical framework of out-of-distribution
  generalization.
\newblock \emph{Advances in Neural Information Processing Systems},
  34:\penalty0 23519--23531, 2021.

\bibitem[Yue et~al.(2019)Yue, Zhang, Zhao, Sangiovanni{-}Vincentelli, Keutzer,
  and Gong]{DBLP:conf/iccv/YueZZSKG19}
Xiangyu Yue, Yang Zhang, Sicheng Zhao, Alberto~L. Sangiovanni{-}Vincentelli,
  Kurt Keutzer, and Boqing Gong.
\newblock Domain randomization and pyramid consistency: Simulation-to-real
  generalization without accessing target domain data.
\newblock In \emph{2019 {IEEE/CVF} International Conference on Computer
  Vision}, 2019.
\newblock \doi{10.1109/ICCV.2019.00219}.
\newblock URL \url{https://doi.org/10.1109/ICCV.2019.00219}.

\bibitem[Zakharov et~al.(2019)Zakharov, Kehl, and
  Ilic]{DBLP:conf/iccv/ZakharovKI19}
Sergey Zakharov, Wadim Kehl, and Slobodan Ilic.
\newblock Deceptionnet: Network-driven domain randomization.
\newblock In \emph{2019 {IEEE/CVF} International Conference on Computer
  Vision}, 2019.
\newblock \doi{10.1109/ICCV.2019.00062}.
\newblock URL \url{https://doi.org/10.1109/ICCV.2019.00062}.

\bibitem[Zhang et~al.(2022)Zhang, Zhang, Zhang, Jin, Zhou, Cai, Zhao, Liu, and
  Liu]{zhang2022delving}
Chongzhi Zhang, Mingyuan Zhang, Shanghang Zhang, Daisheng Jin, Qiang Zhou,
  Zhongang Cai, Haiyu Zhao, Xianglong Liu, and Ziwei Liu.
\newblock Delving deep into the generalization of vision transformers under
  distribution shifts.
\newblock In \emph{Proceedings of the IEEE/CVF Conference on Computer Vision
  and Pattern Recognition}, pp.\  7277--7286, 2022.

\bibitem[Zhang et~al.(2021{\natexlab{a}})Zhang, Ahuja, Xu, Wang, and
  Courville]{DBLP:conf/icml/ZhangAX0C21}
Dinghuai Zhang, Kartik Ahuja, Yilun Xu, Yisen Wang, and Aaron~C. Courville.
\newblock Can subnetwork structure be the key to out-of-distribution
  generalization?
\newblock In \emph{Proceedings of International Conference on Machine
  Learning}, 2021{\natexlab{a}}.
\newblock URL \url{http://proceedings.mlr.press/v139/zhang21a.html}.

\bibitem[Zhang et~al.(2021{\natexlab{b}})Zhang, Zhang, Liu, Weller,
  Sch{\"{o}}lkopf, and Xing]{DBLP:journals/corr/abs-2111-13839}
Hanlin Zhang, Yi{-}Fan Zhang, Weiyang Liu, Adrian Weller, Bernhard
  Sch{\"{o}}lkopf, and Eric~P. Xing.
\newblock Towards principled disentanglement for domain generalization.
\newblock \emph{CoRR}, 2021{\natexlab{b}}.
\newblock URL \url{https://arxiv.org/abs/2111.13839}.

\bibitem[Zhang et~al.(2021{\natexlab{c}})Zhang, Marklund, Dhawan, Gupta,
  Levine, and Finn]{zhang2021adaptive}
Marvin Zhang, Henrik Marklund, Nikita Dhawan, Abhishek Gupta, Sergey Levine,
  and Chelsea Finn.
\newblock Adaptive risk minimization: Learning to adapt to domain shift.
\newblock \emph{Advances in Neural Information Processing Systems},
  34:\penalty0 23664--23678, 2021{\natexlab{c}}.

\bibitem[Zhao et~al.(2019{\natexlab{a}})Zhao, des Combes, Zhang, and
  Gordon]{DBLP:conf/icml/0002CZG19}
Han Zhao, Remi~Tachet des Combes, Kun Zhang, and Geoffrey~J. Gordon.
\newblock On learning invariant representations for domain adaptation.
\newblock In \emph{Proceedings of International Conference on Machine
  Learning}, 2019{\natexlab{a}}.
\newblock URL \url{http://proceedings.mlr.press/v97/zhao19a.html}.

\bibitem[Zhao et~al.(2019{\natexlab{b}})Zhao, Li, Yue, Gu, Xu, Hu, Chai, and
  Keutzer]{DBLP:conf/nips/ZhaoLYG0HCK19}
Sicheng Zhao, Bo~Li, Xiangyu Yue, Yang Gu, Pengfei Xu, Runbo Hu, Hua Chai, and
  Kurt Keutzer.
\newblock Multi-source domain adaptation for semantic segmentation.
\newblock In \emph{Advances in Neural Information Processing Systems},
  2019{\natexlab{b}}.
\newblock URL
  \url{https://proceedings.neurips.cc/paper/2019/hash/db9ad56c71619aeed9723314d1456037-Abstract.html}.

\bibitem[Zhong et~al.(2017)Zhong, Zheng, Kang, Li, and Yang]{zhong2017random}
Zhun Zhong, Liang Zheng, Guoliang Kang, Shaozi Li, and Yi~Yang.
\newblock Random erasing data augmentation. arxiv.
\newblock \emph{arXiv preprint arXiv:1708.04896}, 2017.

\bibitem[Zhou et~al.(2014)Zhou, Khosla, Lapedriza, Oliva, and
  Torralba]{zhou2014object}
Bolei Zhou, Aditya Khosla, Agata Lapedriza, Aude Oliva, and Antonio Torralba.
\newblock Object detectors emerge in deep scene cnns.
\newblock \emph{arXiv preprint arXiv:1412.6856}, 2014.

\bibitem[Zhou et~al.(2022)Zhou, Lin, Zhang, and Zhang]{zhou2022sparse}
Xiao Zhou, Yong Lin, Weizhong Zhang, and Tong Zhang.
\newblock Sparse invariant risk minimization.
\newblock In \emph{International Conference on Machine Learning}, pp.\
  27222--27244. PMLR, 2022.

\bibitem[Ziyin et~al.(2020)Ziyin, Hartwig, and Ueda]{ziyin2020neural}
Liu Ziyin, Tilman Hartwig, and Masahito Ueda.
\newblock Neural networks fail to learn periodic functions and how to fix it.
\newblock \emph{Advances in Neural Information Processing Systems},
  33:\penalty0 1583--1594, 2020.

\end{thebibliography}

 \clearpage
 \appendix
 \tableofcontents
 \section{Related Works}
\label{sec:related}
\subsection{Domain Generalization}
Domain Generalization (DG) aims to maintain the good performance of machine learning models even in the domains that are different from the training (source) domain.
The following are the categories of mainstream domain generalization research.

\textbf{(1) Invariant Learning}: Aligning domain distributions and finding invariance across domains has been often studied with empirical results and theoretical proofs~\citep{DBLP:journals/jmlr/GaninUAGLLML16,DBLP:conf/icml/0002CZG19}. Specifically, researchers have explicitly sought aligning feature distributions based on the maximum mean discrepancy (MMD)~\citep{DBLP:conf/cvpr/LiPWK18}, second order correlation~\citep{DBLP:conf/eccv/SunS16}, moment matching~\citep{DBLP:conf/iccv/PengBXHSW19}, etc.
Besides aligning feature distributions, Arjovsky et al.~\citep{DBLP:journals/corr/abs-1907-02893} proposed IRM to learn an ideal invariant classifier on top of the representation space, which has inspired many follow-up works~\citep{DBLP:conf/icml/KruegerCJ0BZPC21,DBLP:conf/nips/AhujaCZGBMR21,zhou2022sparse}. These studies implement invariant learning via loss function designs. In this paper, we show that if we have a suitable backbone architecture, ERM is sufficient to find the invariant correlations, which indicates that the backbone architecture is as important as loss function design.

\textbf{(2) Ensemble Learning and Meta-learning}: Apart from learning invariant features and correlations across the domain, model-specific or domain-specific information helps improve the performance on DG datasets. The ensemble learning methods combine different models to exploit model-specific information~\citep{DBLP:conf/icip/ManciniBC018,DBLP:journals/corr/abs-2011-12672,DBLP:journals/corr/abs-2110-10832,DBLP:journals/corr/abs-2203-04600}. In addition, SWAD~\citep{cha2021swad} inhibits models from being overfit to local sharp minima by averaging model weights below a validation loss threshold. Meta-learning-based approaches~\citep{DBLP:conf/aaai/LiYSH18,li2019feature,dou2019domain,zhang2021adaptive,DBLP:journals/corr/abs-2110-09410} leverage domain-specific feature to improve the accuracy on each domain.

\textbf{(3) Data Manipulation}: Diverse training data are helpful for improving generalization, and researchers have proposed different manipulation/augmentation techniques ~\citep{DBLP:conf/eccv/NazariK20,DBLP:conf/iclr/RiemerCALRTT19}, domain randomization~\citep{DBLP:conf/iccv/YueZZSKG19,DBLP:conf/iccv/ZakharovKI19}. Furthermore, a line of works~\citep{DBLP:conf/cvpr/QiaoZP20,DBLP:conf/nips/LiuLYW18,DBLP:conf/nips/ZhaoLYG0HCK19} have exploited generating data samples to enhance the model generalization ability.

\textbf{(4) Neural Architecture Search}: Recently, neural architecture search-based methods were employed for DG. In~\citet{DBLP:conf/icml/ZhangAX0C21}, the lottery ticket hypothesis was extended to DG settings: a biased full network contains an unbiased subnetwork that can achieve better OOD performance. The network pruning method was further adopted to identify such subnetworks. Differentiable architecture search was employed in~\citet{bai2021ood} to find a good CNN architecture for DG. These works are independent of backbone architecture design.

One unique advantage of the proposed method is that it is orthogonal to the above approaches, meaning that the performance of GMoE might be enhanced when combined with these approaches.

\subsection{Vision Transformers}
Originated from the machine translation tasks, Transformer~\citep{DBLP:conf/nips/VaswaniSPUJGKP17} has recently received great attention in computer vision~\citep{DBLP:conf/iclr/DosovitskiyB0WZ21,DBLP:conf/iccv/LiuL00W0LG21,wang2022image} for their unprecedented performance in image recognition, semantic segmentation, and visual question answering. In addition to strong performance, some recent works empirically demonstrate the robustness of the ViTs over CNNs in terms of adversarial noise~\citep{DBLP:journals/corr/abs-2110-07858} and distribution shifts~\citep{zhang2022delving}. Nevertheless, the theoretical understandings remain elusive and this paper theoretically validates these effects. The analysis also motivates us to employ mixture-of-experts to enhance the performance of ViT models.

\subsection{Sparse Mixture-of-Experts}
Mixture-of-Experts models, or MoEs, make use of the outputs from several sub-models (experts) through an input-dependent routing mechanism for stronger model performance~\citep{DBLP:journals/neco/JacobsJNH91,DBLP:journals/neco/JordanJ94}.
This training paradigm has led to the development of a plethora of methods for a wide ranges of applications~\citep{hu1997patient,tani1999learning}. However, integrating MoEs and big models will inevitably introduce even larger model sizes and longer inference time. Sparse MoEs~\citep{DBLP:conf/iclr/ShazeerMMDLHD17} were proposed with their routers to select only a few experts so that the inference time is on-par with the standalone counterpart. They were considered promising ways to scale up vision models~\citep{DBLP:conf/nips/RiquelmePMNJPKH21}. Some existing studies have applied MoE-type architecture to generalization tasks. \citep{guo2018multi} considers the muti-source domain adaptation in NLP, where one classifier is trained on each domain and a MoE is adopted to ensemble the classifiers. \citep{rahaman2021dynamic} investigates the systematic generalization problem and proposes to dynamically compose experts instead of selecting experts in MoEs. In this paper, we consider the domain generalization (DG) problem and prove that a backbone more aligned to invariant correlations is more robust to distribution shift. Based on the theory, we repurpose MoE for DG tasks and make several modifications for performance enhancement.

\subsection{Out-of-Distribution Generalization Theory}
The classic PAC learning~\citep{valiant1984theory} is only applicable to IID settings and recently there have been a line of works to investigate the theory of OOD generalization, including structured causal models (SCM)~\citep{DBLP:journals/corr/abs-1907-02893,DBLP:conf/nips/AhujaCZGBMR21,zhou2022sparse} and extrapolation theory~\citep{xu2020neural,ziyin2020neural,ye2021towards}. The SCM-based approaches often assume that the invariant correlation is linear~\citep{DBLP:journals/corr/abs-1907-02893,DBLP:conf/nips/AhujaCZGBMR21,zhou2022sparse}. The extrapolation theory focused on the situation where the supports of training distribution and test distribution are different~\citep{xu2020neural,ziyin2020neural,ye2021towards}, where the networks have very limited ability for nonlinear function approximation~\citep{xu2020neural}. Different from existing works, this paper studies the situation where the invariant correlation is an arbitrary nonlinear function, and one of our main focuses is to justify why these assumptions hold in DG of vision tasks.

\subsection{Algorithmic Alignment}
The algorithmic alignment framework~\citep{xu2019what} was originally proposed to understand the effectiveness of specialized network structures in reasoning tasks. Specifically, it characterizes the generalization bound of reasoning tasks by studying the alignment between the reasoning algorithm and the computation graph of the neural network. The underlying intuition is that if they align well, the neural network only needs to learn \emph{simple} functions, which leads to better sample efficiency. This framework was further extended in \citet{li2021does} to investigate the impact of backbone architectures on the robustness of noisy labels. This paper adopts algorithmic alignment to provide a novel perspective of DG and develops a new architecture based on it.

\subsection{Empirical Investigations of Distribution Shift} 
Prior to this work, some empirical studies~\citep{wiles2021fine,sivaprasad2021reappraising} systematically investigate the impact of the data argumentation, the optimizer, the backbone architecture, and DG algorithms on domain generalization datasets, drawing various conclusions. Nevertheless, they did not explain when a specific model will succeed in DG and how to design the backbone architecture to improve DG performance. As far as we know, our paper is an initial attempt to theoretically investigate these questions, which provides a novel perspective of DG. In this paper, we prove that a backbone is more robust to distribution shift if it aligns with the invariant correlation in the dataset and a better alignment leads to better performance in DG. A novel architecture is designed based on our theory. We believe these results will encourage further research in this direction.

\section{Theorems and Proofs}
\label{sec:proofs}
\subsection{Proof of Theorem \ref{thm:arch}}
\label{sec:proof_arch}
We first state the algorithmic alignment theorem in the IID setting and then extend it to the DG setting.
\begin{theorem}\label{thm:aa} (Alignment improves IID generalization;~\citep{xu2019what}) Fix $\epsilon$ and $\delta$. Given a target function $g$ and a neural network $\mathcal{N}$, suppose $\{\bm{x}_i\}_{i=1}^M$ are i.i.d. samples drawn from some distribution $D$, and let $y_i = g(\bm{x}_i)$. Assumptions:

\textbf{(a) Algorithm stability.} Let $A$ be a learning algorithm for
$\mathcal{N}_i$'s. Suppose $f = A(\{\bm{x}_i,y_i\})$ and $\hat{f} = A(\{\hat{\bm{x}}_i,y_i\})$. For any $\bm{x}$, $\|f(\bm{x}) - \hat{f}(\bm{x})\| \leq L_0 \cdot \max_i \|x_i - \hat{x}_i\|$, where $x_i$ is the $i$-th coordinate of $\bm{x}$.

\textbf{(b) Sequential learning.} We train $\mathcal{N}_i$ sequentially: $\mathcal{N}_1$ has input samples $\{\bm{x}_i^{(1)}, f_1(\bm{x}_i^{(1)}) \}_{i=1}^N$, with $\bm{x}_i^{(1)}$ obtained from the training dataset. For $j > 1$, the inputs $\hat{\bm{x}}_i^{(j)}$ for $\mathcal{N}_j$ are the outputs from the previous modules, but labels are generated by the correct functions $f_{j-1}, \cdots, f_1$ on $\hat{\bm{x}}_i^{(1)}$.

\textbf{(c) Lipschitzness.} The learned functions $\hat{f}_j$ satisfy $\|\hat{f}_j(\bm{x}) - \hat{f}_j(\hat{\bm{x} })\| \leq L_1\|\bm{x} - \hat{\bm{x}}\|$ for some $L_1$.

Under assumptions (a)(b)(c), $\text{Alignment}(\mathcal{N},g,\epsilon,\delta) \leq M$ implies there is a learning algorithm $A$ such that
\begin{align*}
    \mathbb{P}_{\bm{x} \sim D}[\|\mathcal{N}(\bm{x}) - g(\bm{x})\|\leq O(\epsilon)]  \geq 1 - O(\delta),
\end{align*}
where $\mathcal{N}$ is the network generated by $A$ on the training data $\{\bm{x}_i, y_i\}_{i=1}^M$.
\end{theorem}

\begin{remark} \textup{(Explanations of Assumptions in Theorem \ref{thm:aa}) The first and third assumptions are common assumptions in machine learning and are practical. The second assumption is impractical as we do not have auxiliary labels. However, the same pattern is observed for end-to-end learning in experiments~\citep{xu2019what,li2021does}.}
\end{remark}

We now prove Theorem \ref{thm:aa}.

\begin{theorem} (Impact of Backbone Architecture in DG) Let $\mathcal{N}' = \{\mathcal{N}_2,\cdots, \mathcal{N}_n \}$. Assuming we train the neural network with ERM, and Assumption \ref{assumption:ds}, \ref{assumption:invariant}, \ref{assumption:spurious}, we have the following statements:
\begin{enumerate}
    \item If $\text{Alignment}(\mathcal{N}',g_c,\epsilon,\delta) \leq |\mathcal{E}_{tr}|$, we have $\mathbb{P}_{D_{te}}[\|\mathcal{N}(\bm{x}) - y\| \leq O(\epsilon)] > 1 - O(\delta)$;
    \item If $\text{Alignment}(\mathcal{N}',g_s,\epsilon,\delta) \leq |\mathcal{E}_{tr}|$, we have $\mathbb{P}_{D_{te}}[\|\mathcal{N}(\bm{x}) - y\| > \omega(\epsilon)] > 1 - O(\delta)$.
\end{enumerate}
\end{theorem}
\begin{proof}
The proof is mainly based on Theorem \ref{thm:aa}. We tackle the distribution shift in $\bm{x}$ with Lemma \ref{lem:ds} and analyze the distribution shift in $y$ with Assumption \ref{assumption:invariant}, \ref{assumption:spurious}.

    \textbf{First condition:} From Theorem \ref{thm:aa},
    \begin{align*}
        \mathbb{P}_{\bm{x} \sim D_{tr}}[\|\mathcal{N}(\bm{x}) - g_c(\mathcal{N}_1(\bm{x}) )\| > O(\epsilon)]  \leq O(\delta).
    \end{align*}
    By Lemma \ref{lem:ds}, we have
    \begin{align}\label{eq:N_gc}
        \mathbb{P}_{\bm{x} \sim D_{te}}[\|\mathcal{N}(\bm{x}) - g_c(\mathcal{N}_1(\bm{x}))\| \leq O(\epsilon)]  \leq C \mathbb{P}_{\bm{x} \sim D_{tr}}[\|\mathcal{N}(\bm{x}) - g_c(\bm{x})\| < O(\epsilon)] \leq O(\delta).
    \end{align}
    Combing with the second condition of Assumption \ref{assumption:invariant}, we have
    \begin{align*}
        &\mathbb{P}_{D_{te }}[\|\mathcal{N}(\bm{x}) - g(\bm{x})\| \leq O(\epsilon)] \\
        \geq &\mathbb{P}_{D_{te}}[\|g_c(\mathcal{N}_1(\bm{x}) ) - g(\bm{x})\| \leq O(\epsilon)]\cdot \mathbb{P}_{D_{te}}[\|\mathcal{N}(\bm{x}) - g_c(\mathcal{N}_1(\bm{x}) )\| \leq O(\epsilon)] \\
        \geq & (1 - \delta) (1 - O(\delta)) = 1 - O(\delta),
    \end{align*}
    where the first inequality follows $O(\epsilon)+O(\epsilon) = O(\epsilon)$, and the second inequality follows the second condition of Assumption \ref{assumption:invariant} and \eqref{eq:N_gc}.
   
    \textbf{Second condition:} Following the proof of \eqref{eq:N_gc}, we obtain
    \begin{align}\label{eq:N_gs}
        \mathbb{P}_{D_{te}}[\|\mathcal{N}(\bm{x}) - g_s(\mathcal{N}_1(\bm{x}) ) \| \leq O(\epsilon)] \leq 1 - O(\delta).
    \end{align}
    Combing with the second condition of Assumption \ref{assumption:spurious}, we have
    \begin{align*}
        &\mathbb{P}_{D_{te }}[\|\mathcal{N}(\bm{x}) - g(\bm{x})\| > \omega(\epsilon)] \\
        \geq &\mathbb{P}_{D_{te}}[\|g_s(\mathcal{N}_1(\bm{x}) )  - g(\bm{x})\| > \omega(\epsilon)] \cdot \mathbb{P}_{D_{te }}[\|\mathcal{N}(\bm{x}) - g_s(\mathcal{N}_1(\bm{x}) )\| \leq O(\epsilon)] \\
        \geq & (1 - \delta) (1 - O(\delta)) = 1 - O(\delta),
    \end{align*}
    where the first inequality follows $\omega(\delta)-O(\delta) = O(\delta)$, and the second inequality follows the second condition of Assumption \ref{assumption:spurious} and \eqref{eq:N_gs}. This finishes the proof for Theorem \ref{thm:arch}.
\end{proof}

\subsection{Synthetic Experiments for Theorem \ref{thm:arch}}
\label{subsec:synthetic_experiments}
In this subsection, we set up synthetic experiments to illustrate Theorem \ref{thm:arch}.
\subsubsection{Noiseless Version}
We first set up experiments following the assumptions in Theorem \ref{thm:arch}.

\textbf{Training dataset generation:} Consider data pairs $(\bm{x},y) \in (\mathbb{R}^K)^P \times \{1,\cdots, K\}$ in the training and validation datasets generated as follows.
\begin{itemize}
    \item Generate label $y \in  \{1,\cdots, K\}$ uniformly.
    \item Generate $K$ mutually orthogonal feature vectors $\{\bm{c}_k\}$ such that $\bm{c}_k \in \mathbb{R}^K$ and $\|\bm{c}_k\|_2 = 1$.  
    \item Generate $x$ as a collection of $P$ patches: $\bm{x} = (\bm{x}^{(1)}, \cdots, \bm{x}^{(P)}) \in (\mathbb{R}^d)^P$
    \begin{itemize}
    \item \textbf{Pixel-level feature.} For the first patch, the $y$-th pixel is set as $1$ and other pixels drawn from $N(0,1)$.
    \item \textbf{Patch-level feature.} Uniformly select one and only one patch given by $\bm{c}_y$.
    \item \textbf{Random noise.} The rest patches are Gaussian noise drawn independently from $N(0,\bm{I}_K)$.
\end{itemize}
\end{itemize}
The generation of the validation dataset follows that of the training dataset.

\textbf{Test dataset generation:} Two OOD test datasets are given. The first dataset differs from the training dataset in the pixel-level feature. Specifically, we set the $(y+1) \text{mod } K$-th pixel of the first patch as $1$ and other pixels in the first patch follow $N(0,1)$, i.e., making the pixel-level features spurious. The second dataset is distinguished from the training dataset such that the patch-level feature is $\bm{c}_{(y+1) \text{mod } K}$, i.e., making patch-level features spurious.

\textbf{Backbone architectures:} We adopt MLPs and fully convolutional networks (FCNs) as MLPs align with pixel-level features while FCNs align with patch-level features. It is clear that both MLPs and FCNs can fit the label only from pixel-level features or patch-level features.

\textbf{Parameters in experiments:} We set $P = 10$, $K = 4$. We generate $100,000$ samples for the training dataset and $2,000$ test samples for the validation and test datasets. We use a three-layer MLP with hidden size $\{40,100,100,4\}$ and a two-layer FCN with $20$ filters.

\begin{figure*}
    \centering
    \subfigure[MLP's performance.]
    {
        \includegraphics[width=0.48\columnwidth]{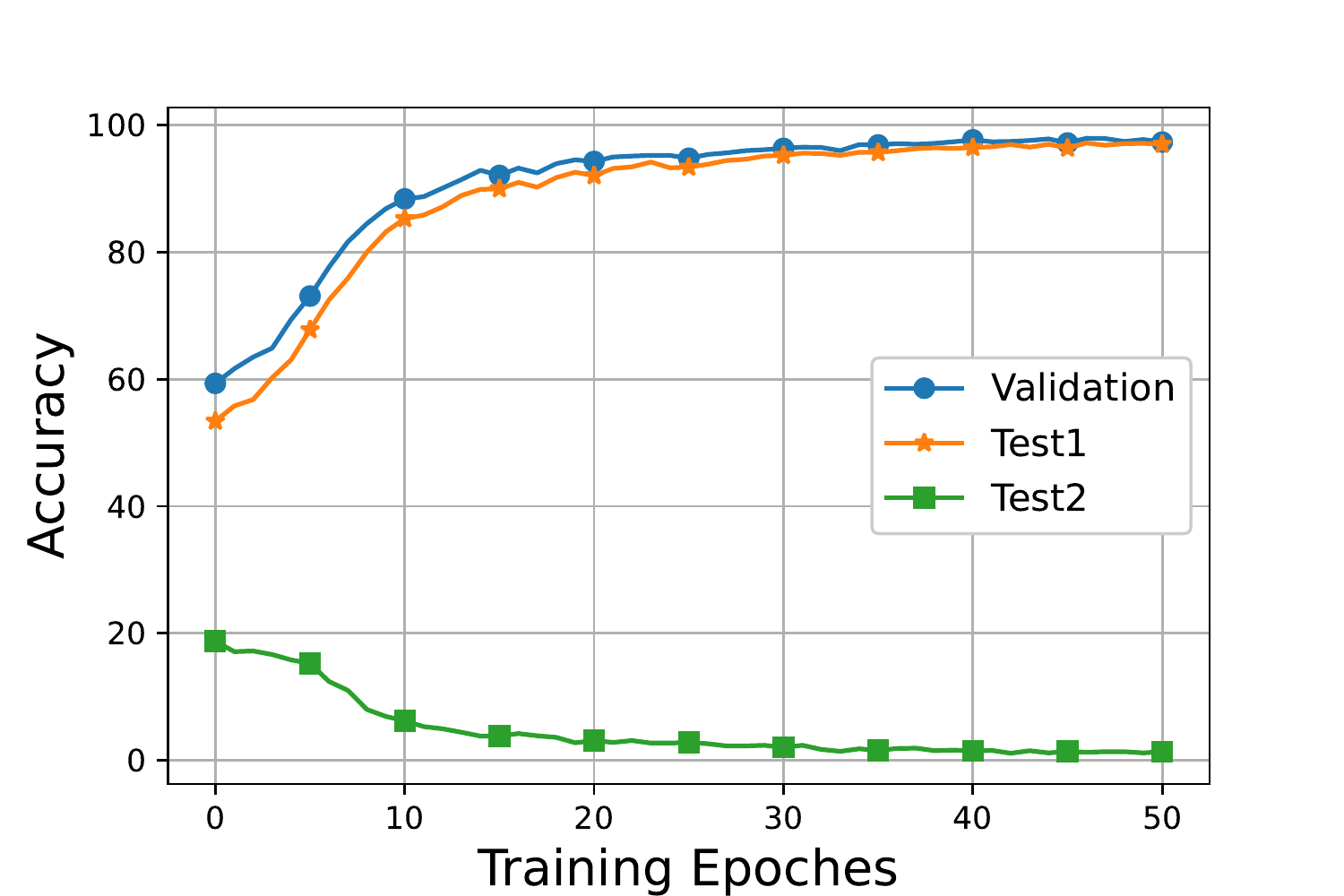}
    }\hfil
    \subfigure[FCN's performance.]
    {
        \includegraphics[width=0.48\columnwidth]{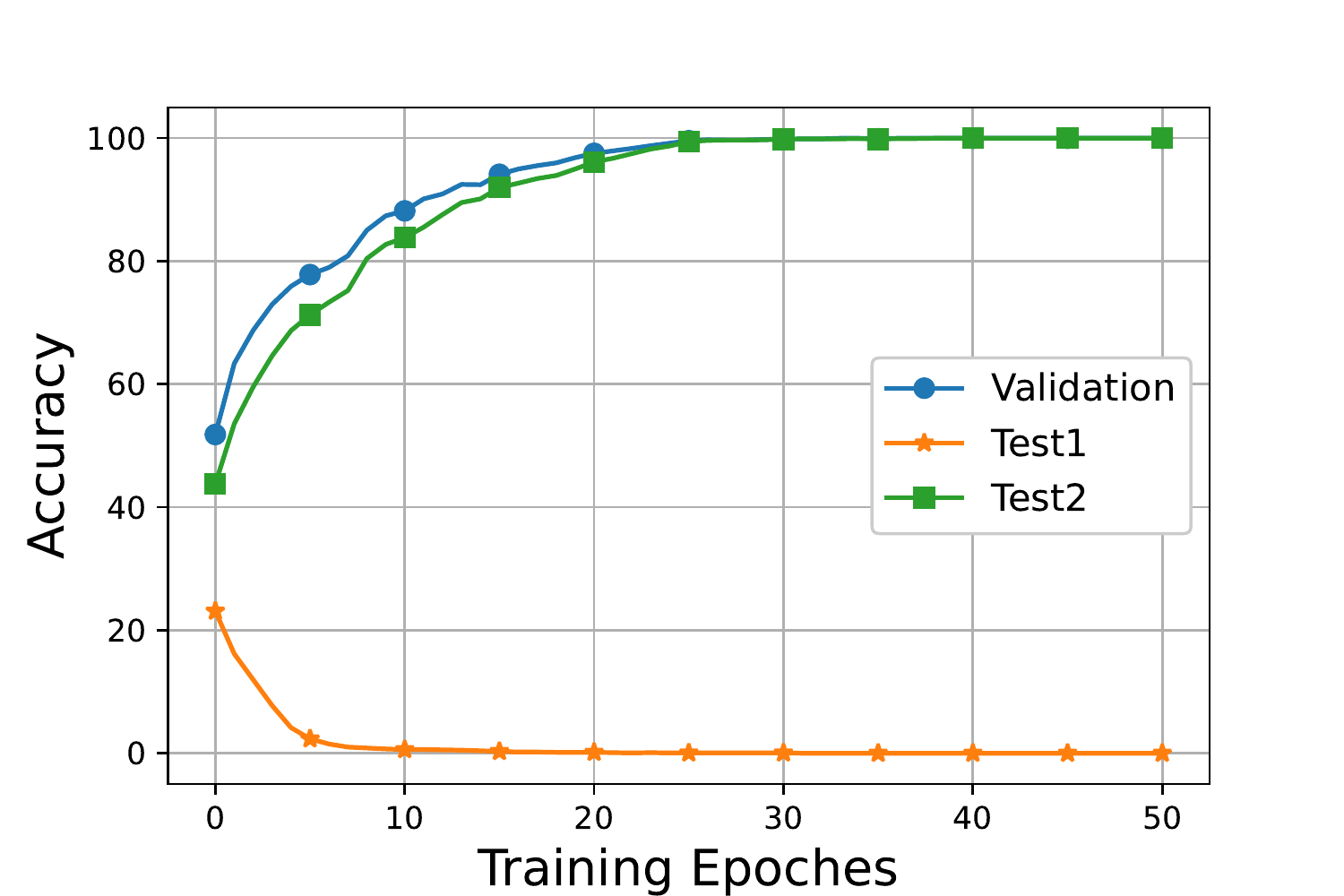}

    }\hfil
   
    \caption{Performance of the MLP and FCN on the synthetic dataset. MLP aligns with the invariant correlations in test dataset 1 while FCN aligns with the invariant correlations in test dataset 2.}
    \label{fig:MLP_FCN}
\end{figure*}

The experimental results are shown in Fig. \ref{fig:MLP_FCN}. The MLP completely fails on the second dataset while the FCN completely fails on the first dataset, which is consistent with Theorem \ref{thm:arch}.  

\subsubsection{Noisy Version}
In the previous experiment, both pixel-level features and patch-level features can perfectly fit the label. We give a noisy version as follows:
\begin{enumerate}
    \item \textbf{Pixel-level feature.} Given a probability value $p_1$. With probability $p_1$, the $y$-th pixel of the first patch is set as $1$ and other pixels are drawn from $N(0,1)$. With probability $1 - p_1$, the $y'$-th ($y'\neq y$) pixel of the first patch is set as $1$ and other pixels are drawn from $N(0,1)$.
    \item \textbf{Patch-level feature.} Given a probability value $p_2$. With probability $p_2$, uniformly select one and only one patch given by $\bm{c}_y$. With probability $1 - p_2$, uniformly select one and only one patch given by $\bm{c}_{y'}$ ($y'\neq y$).
\end{enumerate}

\begin{figure}
    \centering
    \subfigure[$p_1 = 0.9, p_2 = 0.9$.]
    {
        \includegraphics[width=0.31\columnwidth]{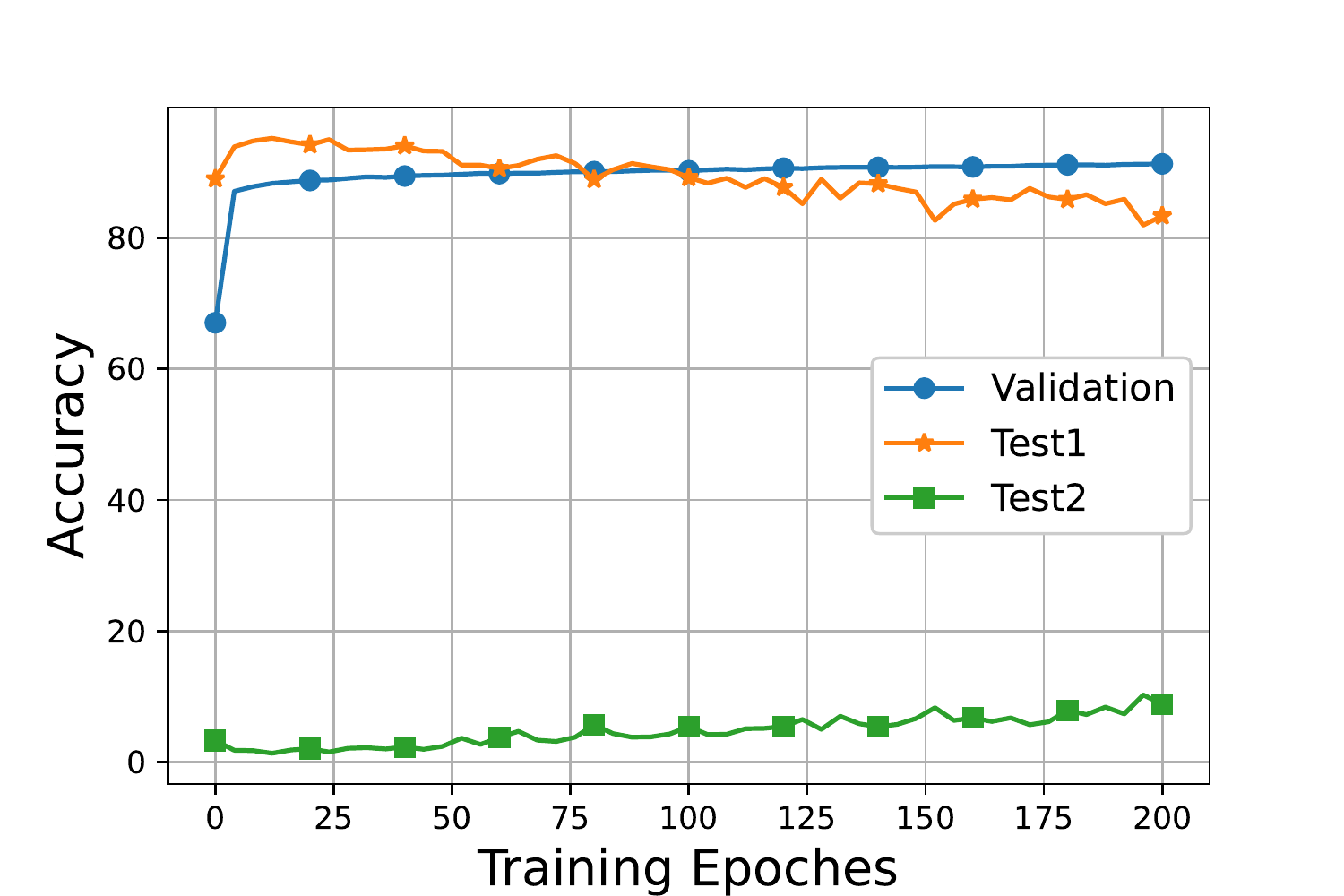}
    }\hfil
    \subfigure[$p_1 = 0.8, p_2 = 0.8$.]
    {
        \includegraphics[width=0.31\columnwidth]{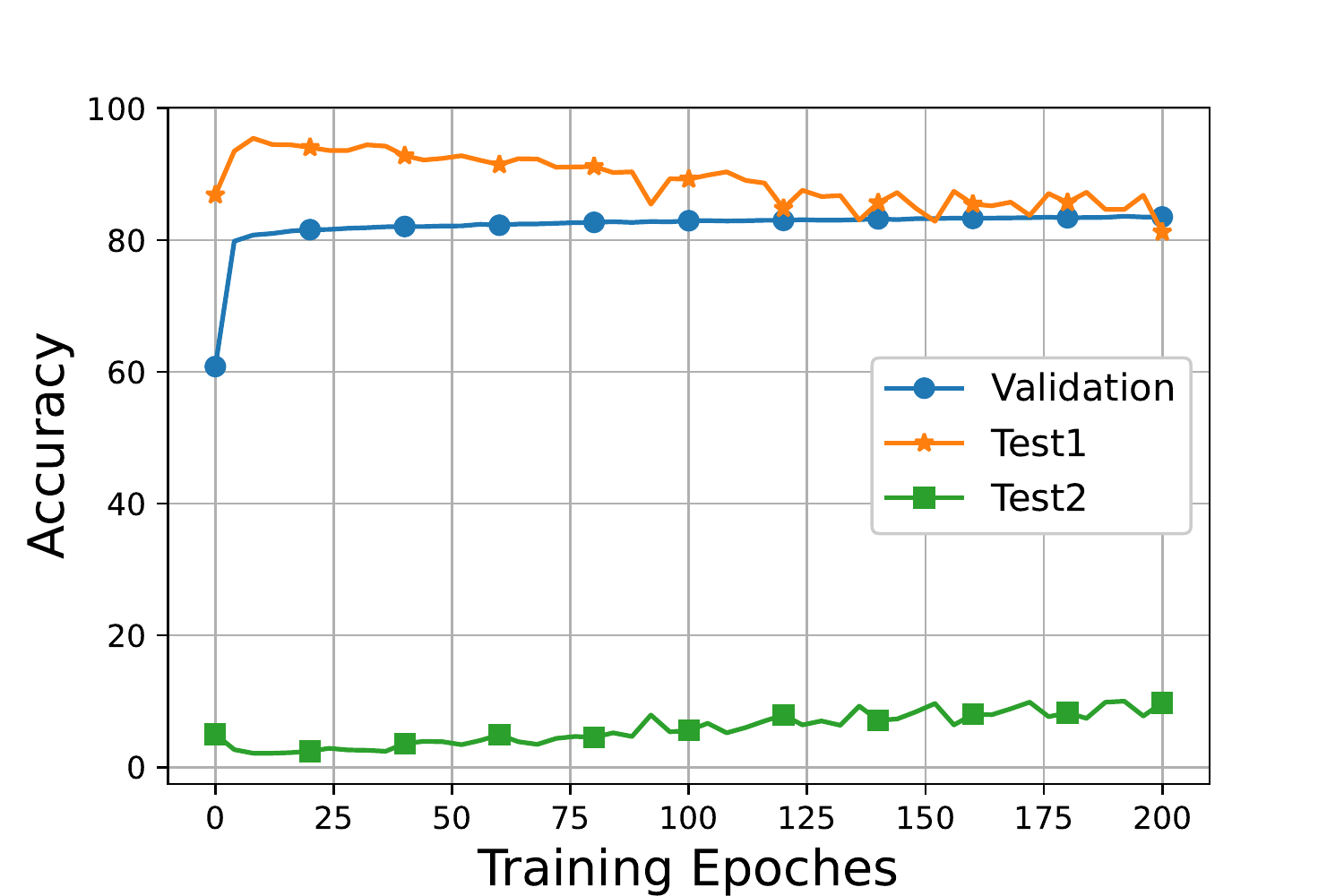}

    }\hfil
    \subfigure[$p_1 = 0.7, p_2 = 0.7$.]
    {
        \includegraphics[width=0.31\columnwidth]{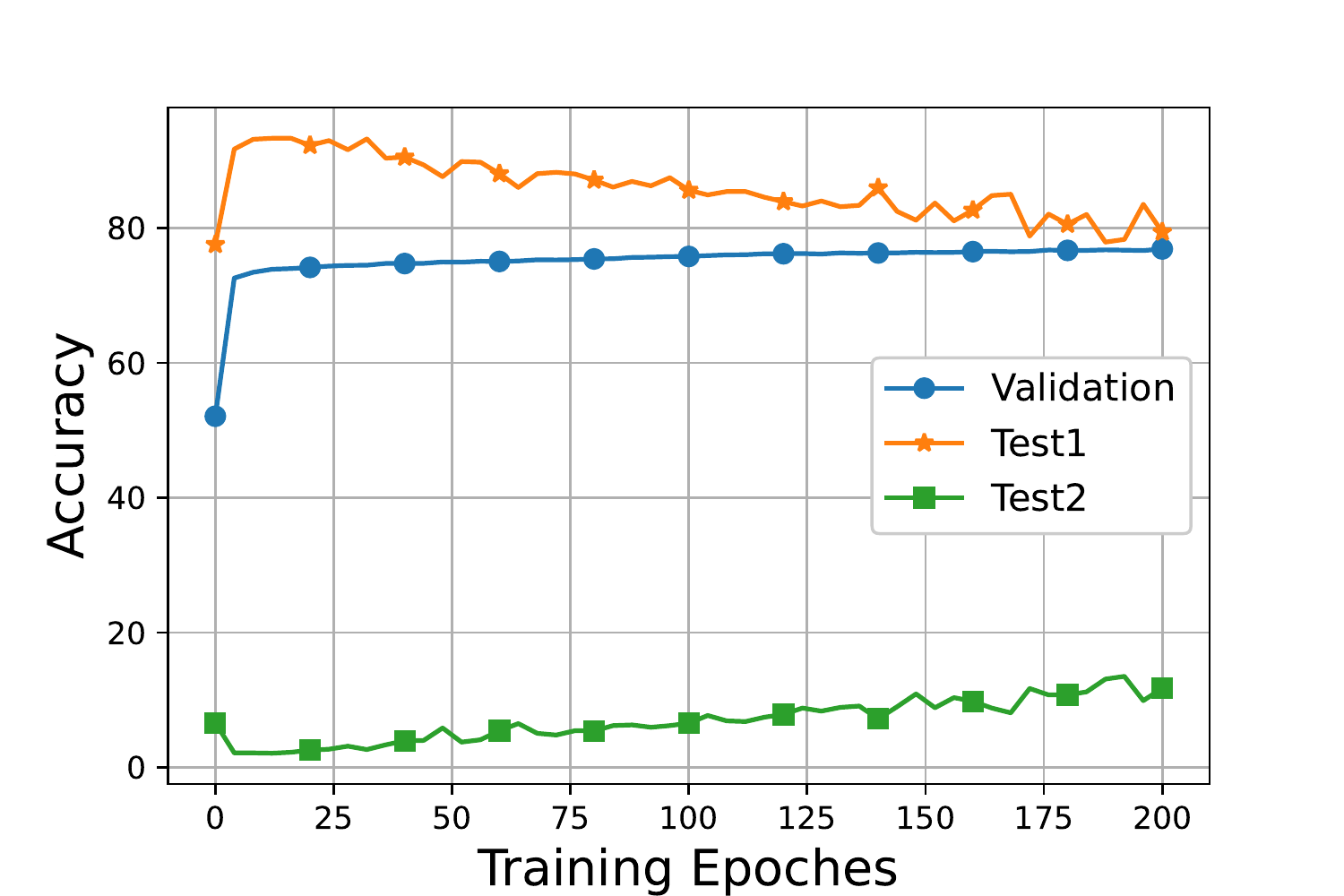}

    }
    \medskip
    \subfigure[$p_1 = 0.9, p_2 = 0.95$.]
    {
        \includegraphics[width=0.31\columnwidth]{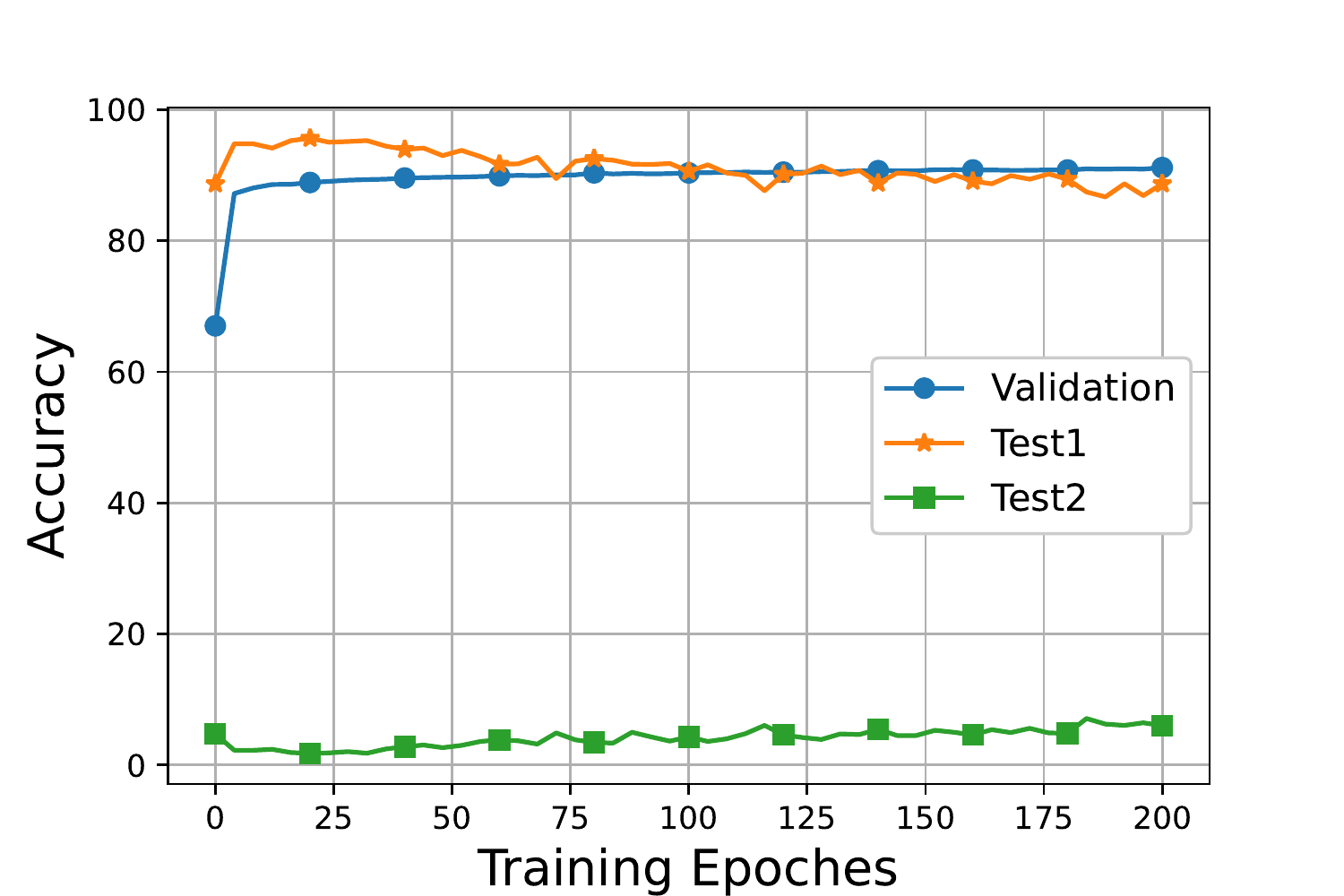}

    }\hfil
    \subfigure[$p_1 = 0.8, p_2 = 0.85$.]
    {
        \includegraphics[width=0.31\columnwidth]{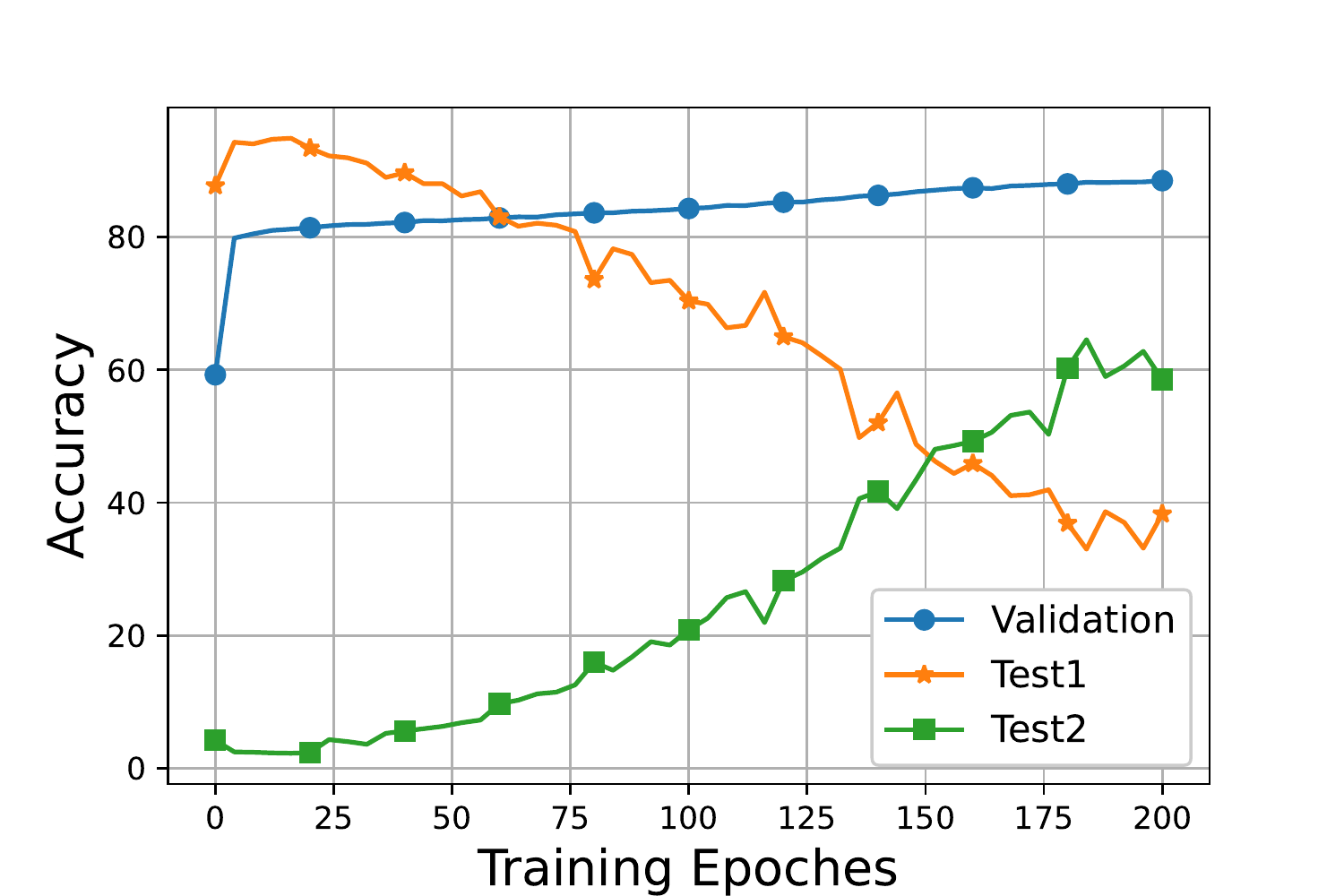}

    }\hfil
    \subfigure[$p_1 = 0.7, p_2 = 0.75$.]
    {
        \includegraphics[width=0.31\columnwidth]{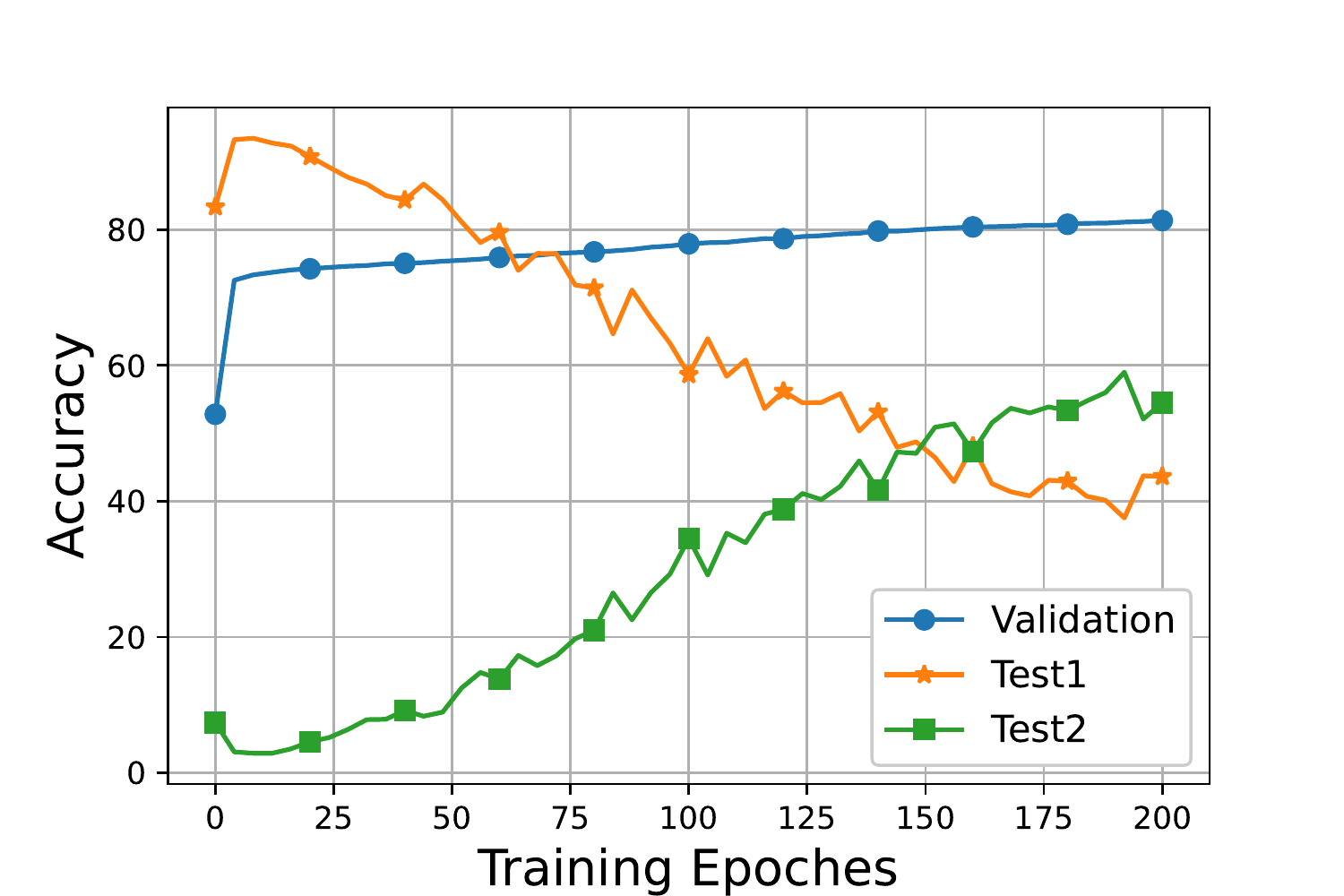}

    }
        \medskip
    \subfigure[$p_1 = 0.9, p_2 = 1$.]
    {
        \includegraphics[width=0.31\columnwidth]{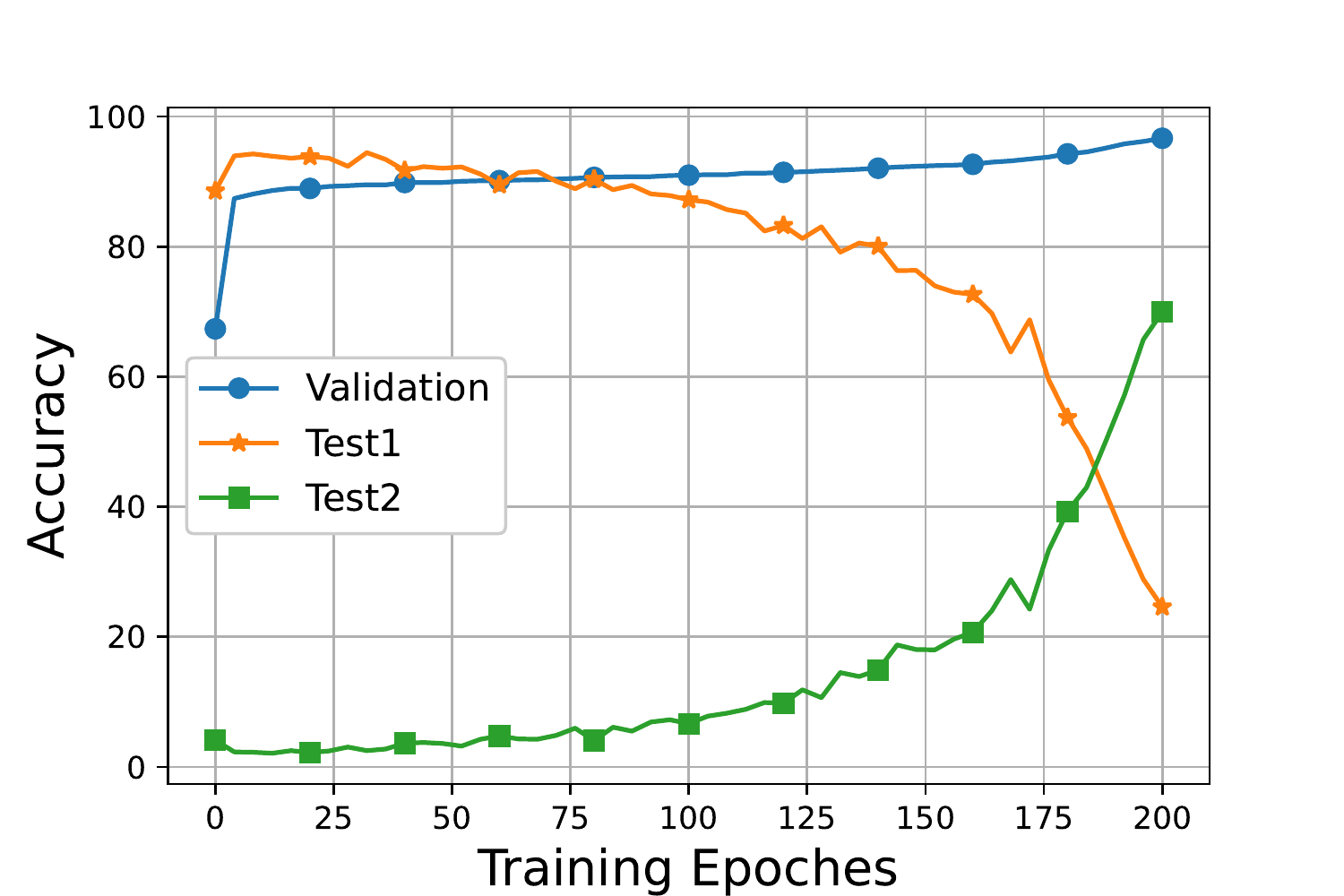}

    }\hfil
    \subfigure[$p_1 = 0.8, p_2 = 0.9$.]
    {
        \includegraphics[width=0.31\columnwidth]{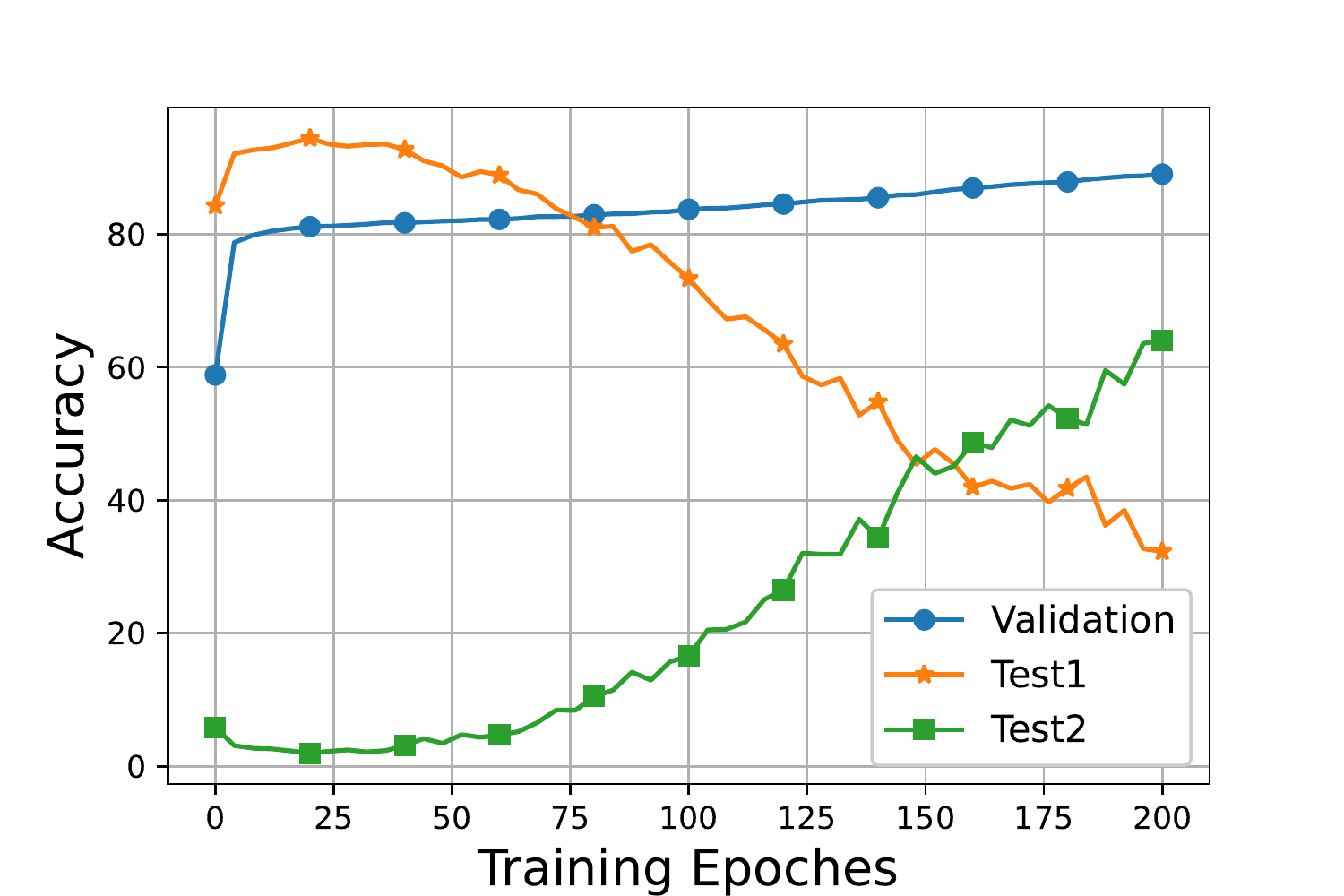}

    }\hfil
    \subfigure[$p_1 = 0.7, p_2 = 0.8$.]
    {
        \includegraphics[width=0.31\columnwidth]{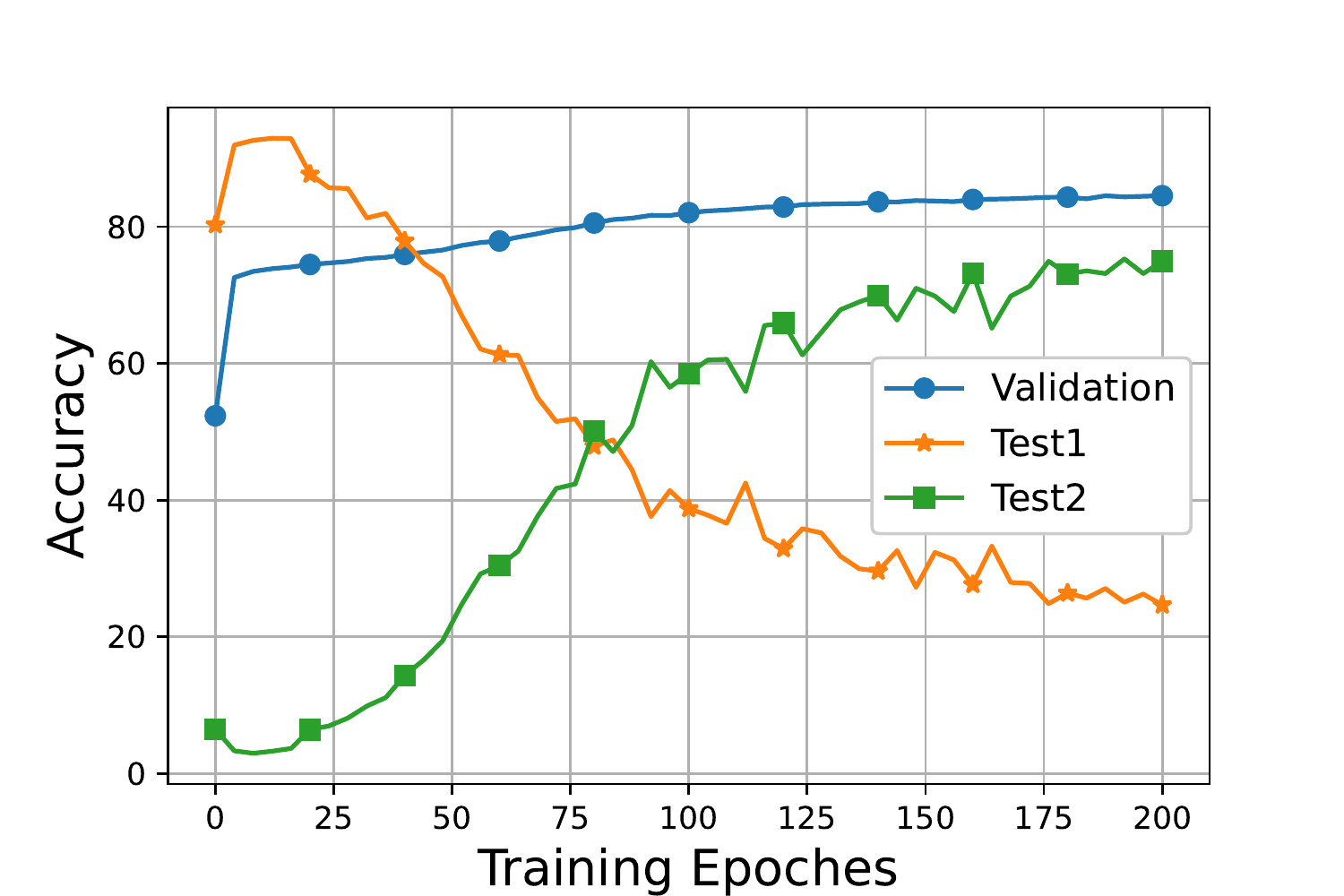}

    }
    \caption{MLP' s performance with different $p_1, p_2$.}
    \label{fig:noisy_syn_exp}
\end{figure}

The experiments with different $p_1$ and $p_2$ are shown in Fig. \ref{fig:noisy_syn_exp}. When the signal-to-noise ratios of pixel-level feature and patch-level feature are similar, the MLP learns the pixel-level feature, which corresponds to its alignment. It shows that the network prefers learning the correlation that aligns with its architecture, even when exploiting other correlations may lead to performance gains. When the difference between $p_1$ and $p_2$ is large, the network will still learn its aligned correlation at the early stage (with above $90\%$ accuracy on Test1), but may change when the number of epochs is large. This suggests that early stopping is beneficial to DG as suggested in~\citet{cha2021swad} under the condition that the backbone aligns to some invariant correlations.

\subsection{Proof of Theorem \ref{thm:moe}}
\label{sec:proof_moe}
\begin{theorem} An MoE module in \eqref{eq:moe} with $N$ experts and $k=1$ aligns the conditional sentences in Algorithm \ref{alg:condition} with
\begin{align*}
    \text{Alignment} = \left\{
\begin{aligned}
&(N+1) \cdot \max\left( \mathcal{M}^*_{\mathscr{P}}, \mathcal{M}(G, h_{1},\epsilon,\delta) \right), && \text{if } N < M,\\
&(N+1) \cdot \max \left( \max_{i\in \{1,\cdots, M\} } \mathcal{M}(f_{\text{FFN}_i}, h_{i+1},\epsilon,\delta), \mathcal{M}(G, h_{1},\epsilon,\delta) \right), &&\text{if } N \geq M,  \\
\end{aligned}
\right.
\end{align*}
where $\mathcal{M}(\cdot,\cdot,\cdot,\cdot)$ is defined in Definition \ref{def:aa}, and $\mathcal{M}^*_{\mathscr{P}}$ is the optimal objective value of the following optimization problem:
\begin{align*}
\mathscr{P}:&\underset{\mathcal{I}_1, \cdots \mathcal{I}_N}{\text{minimize}}
& &  \max_{i \in \{1,\cdots,N \} } \quad \mathcal{M}(f_{\text{FFN}_i}, ([1_{I_j}]_{j\in \mathcal{I}_{i} } \circ h_1)^T \cdot [h_j]_{j\in \mathcal{I}_{i} },\epsilon,\delta) \\
& \text{ subject to }
& & \cup_{i=1}^N \mathcal{I}_i = \{2,3, \cdots, M+1\},
\end{align*}
where $1_{I_i}$ is the indicator function on interval $I_i$.
\end{theorem}
\begin{proof}
For $N \geq M$, we assign one expert for each function. For $N < M$, similar functions should be assigned to one expert to minimize the sample complexity.

\textbf{Case $N \geq M$:} We define the mapping $H(x) = [1_{I_1}, \cdots, 1_{I_M}] \in \mathbb{R}^M$, where $1_{I_i}$ is the indicator function on interval $I_i$. To show the alignment between MoE and conditional sentences, we define the module functions $f_1, \cdots, f_{N}$ and neural network modules $\mathcal{N}_1, \cdots, \mathcal{N}_{N}$ as
\begin{align*}
    f_i = \left\{
\begin{aligned}
&H \circ h_1, &&\text{if } i=1,  \\
&h_{ i+1 }, && \text{if } 1 \leq i\leq M+1,\\
&0, && \text{o.w.},
\end{aligned}
\right. \quad \mathcal{N}_i = \left\{
\begin{aligned}
&G, &&\text{if }i=1,  \\
&f_{\text{FFN}_i}, && \text{if } 1 \leq i\leq M+1,\\
&f_{\text{FFN}_i}, && \text{o.w.},
\end{aligned}
\right.
\end{align*}
Replacing $\mathcal{N}_i$ with $f_i$, the MoE network simulates Algorithm \ref{alg:condition}. Thus, for $N > M$, MoE algorithmically aligns conditional sentences with
\begin{align}
    \text{Alignment} = (N+1) \cdot \max\left( \max_i \mathcal{M}(f_{\text{FFN}_i}, h_{i+1},\epsilon,\delta), \mathcal{M}(G, h_{1},\epsilon,\delta) \right)
\end{align}
where $\mathcal{M}(\cdot,\cdot,\cdot,\cdot)$ is defined in Definition \ref{def:aa}.

\textbf{Case $N < M$:} We define the mapping $H_0(x) = [1_{\cup_{j \in \mathcal{I}_1} I_j}, \cdots, 1_{\cup_{j \in \mathcal{I}_N} I_j}] \in \mathbb{R}^N$, where $ \mathcal{I}_i$ is the solution to the optimization problem in \eqref{eq:sc_opt}. To show the alignment between MoE and conditional sentences, we define the module functions $f_1, \cdots, f_{N}$ and neural network modules $\mathcal{N}_1, \cdots, \mathcal{N}_{N}$ as
\begin{align*}
    f_i = \left\{
\begin{aligned}
&H_0 \circ h_1, &&\text{if } i=1,  \\
&([1_{I_j}]_{j\in \mathcal{I}_{i-1} } \circ h_1)^T \cdot [h_j]_{j\in \mathcal{I}_{i-1} }, && \text{o.w.},
\end{aligned}
\right. \quad \mathcal{N}_i = \left\{
\begin{aligned}
&G, &&\text{if }i=1,  \\
&f_{\text{FFN}_i}, && \text{o.w.}
\end{aligned}
\right.
\end{align*}
Replacing $\mathcal{N}_i$ with $f_i$, the MoE network simulates Algorithm \ref{alg:condition}. Thus, for $N < M$, MoE algorithmically aligns with conditional sentences with
\begin{align*}
    \text{Alignment} = (N+1) \cdot \max\left( \mathcal{M}^*_{\mathscr{P}}, \mathcal{M}(G, h_{1},\epsilon,\delta) \right)
\end{align*}
where $\mathcal{M}^*_{\mathscr{P}}$ is the optimal objective value to \eqref{eq:sc_opt} and $\mathcal{M}(\cdot,\cdot,\cdot,\cdot)$ is defined in Definition \ref{def:aa}.
\end{proof}

\subsection{Technical Lemmas}
\begin{lemma}\label{lem:ds}
Under Assumption \ref{assumption:ds}, for a given functions $f(\cdot)$ and an interval $A$, $\mathbb{P}_{D_{tr}}[f(\mathcal{N}_1(\bm{x})) \in A] < \delta$ implies  $\mathbb{P}_{D_{te}}[f(\mathcal{N}_1(\bm{x})) \in A] \leq C\delta$.
\end{lemma}
\begin{proof}
We denote the inverse mapping as $f^{-1}(A) = \{\bm{s}|f(\bm{s}) \in A\}$. Note that
\begin{align*}
    &\mathbb{P}_{D_{te}}[f(\mathcal{N}_1(\bm{x})) \in A] = \int_{\bm{s} \in f^{-1}(A)} p_{\text{test},\mathcal{N}_1}(\bm{s})d\bm{s} \\
    \leq &C \int_{\bm{s} \in f^{-1}(A)} p_{\text{train},\mathcal{N}_1}(\bm{s})d\bm{s} = C \mathbb{P}_{D_{tr}}[f(\mathcal{N}_1(\bm{x})) \in A]
\end{align*}
This finishes the proof for Lemma \ref{lem:ds}.
\end{proof}
\begin{remark}\textup{Assumption \ref{assumption:ds} is similar to $\mu$-exploration in batch reinforcement learning (RL), which aims at controlling the distribution shift in RL~\citep{xie2021batch}. Lemma \ref{lem:ds} suggests that the generalization with the same training and test support resembles the IID generalization. Unfortunately, this is the largest distribution shift we can assume as the impossibility of nonlinear out-of-support generalization is proved in~\citet{xu2020neural}.}
\end{remark}
 \section{Implementations}

\subsection{Detailed Architecture}
\label{subsec:detailed_arch}
For image input $\mathbf{X}$, we first process the images as a sequence of equal-sized patches $\bm{x}_i \in \mathbf{X}$ using a 1-layer convolutional neural network with layer normalization. Then those patches are added with positional embeddings, and the patch embeddings (tokens) are ready to be processed by later $L$ blocks.

To make a fair comparison with other reported methods on DomainBed, we choose the ViT-S/16 which has similar parameters and run-time memory cost with ResNet-50 as the basic architecture for our main model GMoE-S/16. The chosen ViT-S/16 model has an input patch size of $16 \times 16$, 6 heads in multi-head attention layers, and 12 transformer blocks.

We consider two-layer configurations, \emph{every-two} and \emph{last-two}, that are widely adopted in Sparse MoE design~\citep{DBLP:conf/nips/RiquelmePMNJPKH21}. The detailed architectures are illustrated in Figure~\ref{fig:gmoe_arch}. Besides, we also consider evaluating GMoE's generalization performance with larger scale models (\eg ViT-Base).

V-MoE~\citep{DBLP:conf/nips/RiquelmePMNJPKH21} trains $32$ experts model on ImageNet-21K dataset. However, considering the number of training data, we need to adopt a smaller number of experts. More experts require larger datasets to ensure that every expert is adequately trained since each expert only sees a small portion of the dataset. Domain generalization datasets are usually 1-2 orders of magnitude smaller than ImageNet-21K. Based on this fact, each GMoE block contains $6$ experts. For each image patch, the cosine router selects the TOP-2 index out of $6$ and routes the patch to the corresponding $2$ experts.

We put the ablation study results on layer configuration and larger size backbone model on Sec.~\ref{subsec:ablation_model_design}.

\begin{figure}[h]
    \centering
    \subfigure[\textit{Every 2} GMoE.]
    {
        \includegraphics[width=0.48\columnwidth]{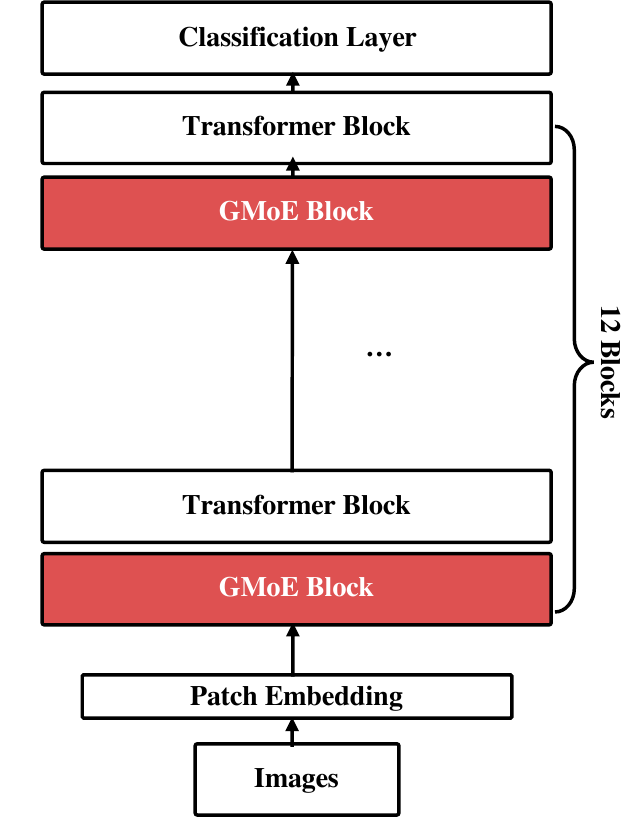}
        \label{fig:gmoe_vis1}
    }
    \subfigure[\textit{Last 2} GMoE.]
    {
        \includegraphics[width=0.48\columnwidth]{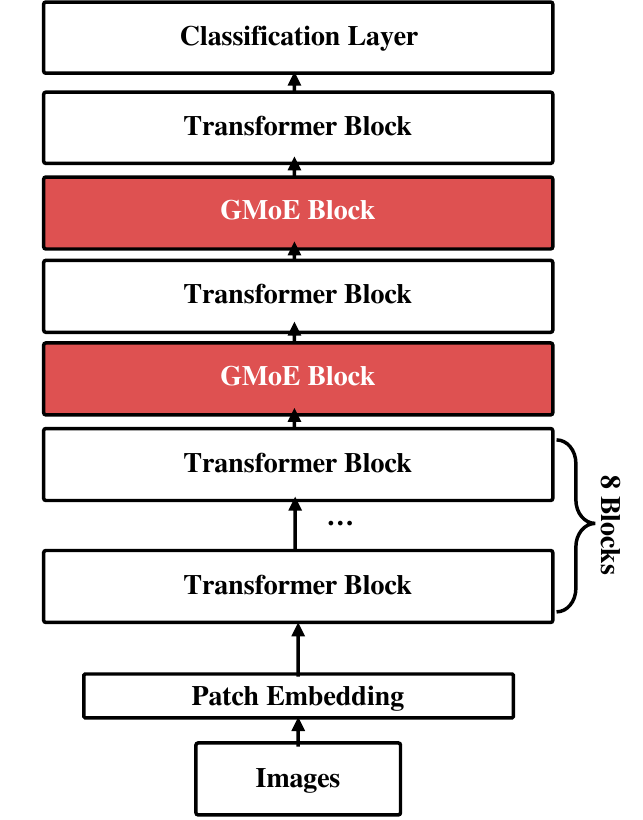}
        \label{fig:vit_vis1}
    }
    \caption{Diagrams of detailed GMoE architecture with two types of layer configuration.}
    \label{fig:gmoe_arch}
\end{figure}

\subsection{Model Initialization}
\label{subsec:model_init}
To make a fair comparison with other algorithms from DomainBed (mostly with ResNet-50 pretrained on ImageNet-1K), we also utilize the pretrained ViT models on ImageNet-1K from \href{https://github.com/facebookresearch/deit/blob/main/README_deit.md}{DeiT} to initialize our GMoE. Our practice avoids the generalization performance improvement of GMoE coming from a stronger pretrained model (\eg directly from V-MoE pretrained on ImageNet-21K). Meanwhile, there is no released pretrained MoE model on ImageNet-1K. In Table~\ref{tab:backbone_comparison}, we see ViT-S/16 and GMoE-S/16 achieve excellent generalization performance even with fewer parameters and smaller architecture.

In the GMoE block, we initialize the multiple experts with pretrained ViT block's FFN in the same position (\eg block-8 and 10) and multiply it by $E$ copies ($E$ is the number of experts). For cosine routers, we randomly initialize them in each GMoE block, to make them evenly choose experts at first.

\begin{table}[htp]
\centering
\caption{Comparison of the different backbone models used in the experiment. ImageNet-1K* denotes GMoE are initialized with ViT pretrained on ImageNet-1K and copy weights for different experts.}
\setlength{\tabcolsep}{6pt}
\renewcommand{\arraystretch}{1.6}
\label{tab:backbone_comparison}
\resizebox{0.85\textwidth}{!}{%
\begin{tabular}{c|cccc}
\toprule
\rowcolor{COLOR_MEAN}
\textbf{Backbone}           & ResNet-50 & ResNet-101 &  ViT-S/16    & GMoE-S/16 \\ \hline
\textbf{Pretrained Dataset} & ImageNet-1K & ImageNet-1K & ImageNet-1K & ImageNet-1K* \\ \hline
\textbf{Parameters}             & 25.6M   & 42.9M &  21.7M       & 33.8M \\ \hline
\textbf{Flops}             & 4.1G        & 7.9G & 4.6G        & 4.8G \\
\bottomrule
\end{tabular}%
}
\end{table}

\subsection{Perturbation}
Another specific design for the router is adding Gaussian noise to increase stability. Hard routing brings sparsity benefits while also causing discontinuities in routing. In a TOP-2 selection routing, a small perturbation would cause order changes if the largest output and second-largest output are close.
In domain generalization, data come from diverse domains and we need to learn a model to capture the invariant semantic meaning of an object.
Existing work~\citep{DBLP:journals/corr/abs-2208-02813} theoretically claims that adding noise to the gating function provides a smooth transition of the gating function during training. Empirically, we find adding noise with a standard deviation of $\frac{1}{N}$ would empirically improve performance and stabilize training dynamics.

\subsection{Auxiliary loss}
\label{subsec:aux_loss}
It has been observed that the gating network tends to converge to a self-reinforcing imbalanced state, where a few experts are more frequently activated than others. Following~\citet{DBLP:conf/nips/RiquelmePMNJPKH21,DBLP:conf/iclr/ShazeerMMDLHD17}, we define an importance loss $\mathcal{L}_{\text{imp}}$ to encourage balanced usage of experts. The importance of expert $e$ is defined as the normalization of the gating function's output in a batch of images:
\begin{align}
\label{eq:importance}
    \text{imp}_{e} (\textbf{X})= \sum\limits_{\textbf{x} \in \textbf{X}} g(\bm{x})_{e}
\end{align}
And the loss $\mathcal{L}_{imp}$ is defined by the squared coefficient of variation of $\text{imp}(\textbf{X})$:
\begin{align*}
   \mathcal{L}_{\text{imp}} = \left(\frac{\textsc{std}(\text{imp}(\textbf{X}))}{\textsc{mean}(\text{imp}(\textbf{X}))}\right)^{2}
\end{align*}
While importance loss seeks that all experts have on average similar output routing weights, it is still likely some experts would not be routed (see example in Table~\ref{tab:simple_case}).
\begin{table}[htp]
\centering
\caption{Simple example on $\text{TOP}_{1}$ score (before SoftMax operation) where average weights across patches ($x_{1,2,3}$) are balanced. However, Expert 2 has always been ignored.\\}
\label{tab:simple_case}

\setlength{\tabcolsep}{14pt}
\renewcommand{\arraystretch}{1.3}
\footnotesize
\begin{tabular}{c|cccc|c}
\toprule
\rowcolor{COLOR_MEAN}
\textbf{Patch} & \textbf{Expert 1} & \textbf{Expert 2} & \textbf{Expert 3} & \textbf{Expert 4} & \textbf{Selected Expert} \\ \hline
$x_1$      & 0.9 & 0.4 & 0.1 & 0.2 & Expert 1 \\ \hline
$x_2$      & 0.2 & 0.4 & 0.9 & 0.1 & Expert 3 \\ \hline
$x_3$      & 0.1 & 0.4 & 0.2 & 0.9 & Expert 4 \\ \bottomrule
\end{tabular}%
\end{table}

To address this issue, besides balancing across patches, we also need to encourage balanced assignment across experts. Ideally, we expect a balanced assignment via a load loss $\mathcal{L}_{\text{load}}$. However, the assignment is a discrete value; then we expect a proxy to make it differentiable in back-propagation. Following the idea in ~\citet{DBLP:conf/nips/RiquelmePMNJPKH21,DBLP:conf/iclr/ShazeerMMDLHD17,mustafa2022multimodal}, for each token $\bfx$ and each expert $e_i$, we could compute the probability of $e_i$ once being selected in $\text{TOP}_{k}$ after SoftMax operation, and then kept if re-sampling again \textit{only} the noise for expert $e_i$. In other words, we are computing the probability that having expert $e_i$ still being among the $\text{TOP}_k$ while re-sampling again \textit{only} the noise of the expert. This probability is formally defined as in Eq.~\ref{eq:load_prob}, where $\eta_{k}$ the $K$-th largest entry after SoftMax operation and $\Phi$ is the cumulative distribution function of a Gaussian distribution.
\begin{align}
\label{eq:load_prob}
    p_{e}(\bm{x}) = 1 - \Phi\left(\frac{\eta_k - (\bm{E}^{T} \bm{W} \bm{x})_{e}}{\sigma}\right)
\end{align}
Then the expert load could be defined as $\text{load}_{e}(\textbf{X}) = \sum_{x \in \textbf{X}} p_{e}(\bm{x})$ and the whole experts load is defined as $\text{load}(\textbf{X}) = \{\text{load}_{e} (\textbf{X})\}_{e=1}^{E}$ .The load loss $\mathcal{L}_{\text{load}}$ is defined by
\begin{align*}
   \mathcal{L}_{\text{load}}(\textbf{X}) = \left(\frac{\textsc{std} (\text{load}(\bm{x}))}{\textsc{mean} (\text{load}(\bm{x}))}\right)^{2}
\end{align*}
With the above adaptations for DG tasks, GMoE still works with a simple overall loss:
\begin{align*}
    \mathcal{L}(\textbf{X}) = \mathcal{L}_{\text{classification}} (\textbf{X}) + \frac{1}{2} \lambda (\mathcal{L}_{\text{imp}} (\textbf{X}) + \mathcal{L}_{\text{load}} (\textbf{X}))
\end{align*}
with a hyperparameter $\lambda > 0$ to balance the main task loss and the auxiliary loss, in which usually we view equal importance of the load and importance loss.

 \clearpage
\section{Experimental Results}
\label{sec:extended_experiments}
\subsection{DomainBed Details}
\label{subsec:domainbed}
\paragraph{Benchmark Datasets}
Following previous DG studies, we mainly evaluate our proposed method and report empirical results on DomainBed~\citep{DBLP:conf/iclr/GulrajaniL21} with $8$ benchmark datasets, PACS, VLCS, OfficeHome, TerraIncognita, DomainNet, SVIRO, Wilds-Camelyon and Wilds-FMOW. Besides DomainBed, we also conduct experiments on the self-created CUB-DG dataset for model analysis. In the following, we will provide the details of different datasets.

\paragraph{Dataset Details}
In Table~\ref{tab:dataset}, we provide statistics for the 8 datasets in DomainBed, as well as our self-created dataset CUB-DG (e.g., number of domains, categories and images).
\begin{table}[ht]
\setlength{\tabcolsep}{12pt}
\renewcommand{\arraystretch}{1.4}
\centering
\caption{Statistics of DomainBed datasets.}
\label{tab:dataset}
\resizebox{0.9\textwidth}{!}{%
\begin{tabular}{c|ccccc}
\toprule
\rowcolor{COLOR_MEAN}
\textbf{Dataset} & \textbf{PACS} & \textbf{VLCS} & \textbf{OfficeHome} & \textbf{TerraInc} & \textbf{DomainNet}  \\ \hline
\# Domains          & 4             & 4             & 4                   & 4                       & 6    
 \\ \hline
\# Classes          & 7             & 5             & 65                  & 10                      & 345  
 \\ \hline
\# Examples         & 9,991         & 10,729        & 15,588              & 24,788                  & 586,575 \\
\bottomrule
\end{tabular}%
}
\end{table}
\begin{table}[ht]
\setlength{\tabcolsep}{14pt}
\renewcommand{\arraystretch}{1.4}
\centering
\caption{Statistics of Wilds and CUB-DG. TerraInc stands for TerraIncognita, Camelyon stands for Wilds-Camelyon and FMOW stands for Wilds-FMOW.}
\label{tab:dataset2}
\resizebox{0.8\textwidth}{!}{%
\begin{tabular}{c|cccc}
\toprule
\rowcolor{COLOR_MEAN}
\textbf{Dataset} & \textbf{SVIRO} & \textbf{W-Camelyon} & \textbf{W-FMOW} & \textbf{CUB-DG} \\ \hline
\# Domains & 10             & 5                       & 6 & 4 \\ \hline
\# Classes & 7              & 2                       & 62 & 200 \\ \hline
\# Examples & 56,000         & 455,954                 & 523,846 & 47,152 \\ \bottomrule
\end{tabular}%
}
\end{table}

In detail, the $9$ multi-domain image classification datasets are comprised of:

\begin{enumerate}
    \item \textbf{PACS}~\citep{DBLP:conf/iccv/LiYSH17} comprises four domains $d \in $ \{art, cartoons, photos, sketches\}. This dataset contains 9, 991 examples of dimension (3, 224, 224) and 7 classes.
    \item \textbf{VLCS}~\citep{DBLP:conf/iccv/FangXR13} comprises photographic domains $d \in $ \{Caltech101, LabelMe, SUN09, VOC2007\}. This dataset contains 10, 729 examples of dimension (3, 224, 224) and 5 classes.
    \item \textbf{Office-Home}~\citep{DBLP:conf/cvpr/VenkateswaraECP17} includes domains $d \in $ \{art, clipart, product, real\}. This dataset contains 15, 588 examples of dimension (3, 224, 224) and 65 classes.
    \item \textbf{TerraIncognita}~\citep{DBLP:conf/eccv/BeeryHP18} contains photographs of wild animals taken by camera traps at locations $d \in \{\text{L100}, \text{L38}, \text{L43}, \text{L46}\}$. Our version of this dataset contains 24, 788 examples of dimension (3, 224, 224) and 10 classes.
    \item \textbf{DomainNet}~\citep{DBLP:conf/iccv/PengBXHSW19} has 6 domains $d \in$ \{clipart, infograph, painting, quickdraw, real, sketch\}. This dataset contains 586, 575 examples of size (3, 224, 224) and 345 classes.
    \item \textbf{SVIRO}~\citep{DBLP:conf/wacv/CruzWBSS20} is a Synthetic dataset for Vehicle Interior Rear seat Occupancy across 10 different vehicles, the 10 domains. \\ The domains are $d \in \{\text{aclass}, \text{escape}, \text{hilux}, \text{i3}, \text{lexus}, \text{tesla}, \text{tiguan}, \text{tucson}, \text{x5}, \text{zoe}\}$. This dataset for image classification contains 56,000 image examples of size (3, 224, 224) and 7 classes.
    \item \textbf{Wilds-Camelyon}~\citep{DBLP:conf/icml/KohSMXZBHYPGLDS21} dataset contains histopathological image slides collected and processed by different hospitals and curated by Wilds benchmark. It contains 455,954 examples of dimension (3, 224, 224) and 2 classes from 5 hospitals.
    \item \textbf{Wilds-FMOW} ~\citep{DBLP:conf/icml/KohSMXZBHYPGLDS21} dataset is a variant of the functional map of the world dataset~\citep{DBLP:conf/cvpr/ChristieFWM18} and contains satellite images of 62 buildings or land classes across 6 regions. The image example is of size (3, 224, 224).
    \item \textbf{CUB-DG} dataset is built from the original Caltech-UCSD Birds (CUB) dataset~\citep{wah2011caltech}. We stylize the original images into three domains, Candy, Mosaic, and Undie, while keeping the original images as the Natural domain. Overall, CUB-DG has 4 domains $d \in $ \{Natural, Candy, Mosaic, Udnie\}. It contains 47,152 bird examples and are categorized into 200 classes. Each image has detailed annotations: 1 category label, 15 positional attributes (\eg left wing, back), 312 binary attributes (\eg \texttt{has\_bill\_shape=dagger}), and 1 bounding box information. We treat the 15 positional attributes as visual attributes of a bird image, and thus we could evaluate the correlation between visual attributes and expert selections.
\end{enumerate}

\paragraph{Evaluation protocols}
We evaluate our proposed method with two protocols on DomainBed~\citep{DBLP:conf/iclr/GulrajaniL21} benchmark.

For \textbf{train-validation selection}, we split each training domain into training and validation subsets. Then, we pool the validation subsets of each training domain to create an overall validation set. Finally, we choose the model maximizing the accuracy on the overall validation set, and report the final accuracy on one leave-out test domain. 

For \textbf{leave-one-domain-out validation selection}, we train the model on all training domains and one domain out as the validation set. We choose the model maximizing the accuracy on the leave-out validation domain, and report accuracy on another leave-out test domain. We should emphasize that the leave-one-domain-out validation setting means we choose two domains as leave-out domains. This term is used in conformity in literature~\citep{DBLP:journals/corr/abs-2203-10789,cha2021swad}.

\paragraph{Standard error bars} We train the model three times with random seeds on weight initializations. The mean and standard error of these repetitions are reported.

\begin{table}[tp]
\centering
\caption{Hyperparameters to reproduce best performance of GMoE on each dataset.}
\setlength{\tabcolsep}{8pt}
\renewcommand{\arraystretch}{1.4}
\label{tab:hparams}
\resizebox{0.85\textwidth}{!}{%
\begin{tabular}{c|ccccc}
\toprule
\rowcolor{COLOR_MEAN}
\textbf{Hyperparameters} & \textbf{PACS} & \textbf{VLCS} & \textbf{OfficeHome} & \textbf{TerraInc} & \textbf{DomainNet} \\ \hline
Learning Rate            & $3 \times 10^{-5}$      & $3 \times 10^{-5}$      & $1 \times 10^{-5}$            & $5 \times 10^{-5}$          & $5 \times 10^{-5}$           \\
Weight Decay             & 0      & $1 \times 10^{-6}$      & $1 \times 10^{-6}$           & $1 \times 10^{-4}$          & 0                  \\
\bottomrule
\end{tabular}%
}
\end{table}

\subsection{Training Details \& Hyperparameters}
\label{subsec:train_details}
In this section, we report the training details of our experiments. We optimize models using Adam optimizer~\citep{DBLP:journals/corr/KingmaB14} with slightly different parameters on different datasets (see Table~\ref{tab:hparams}). The training and inference batch size is set to 32 for each domain.

\begin{table}[htp]
\centering
\caption{Train-validation selection performance comparison with 5K and 15K training iterations, where \textit{train-val} stands for train-validation selection and \textit{lodo} stands for leave-one-domain-out selection. The place marked --- is not applicable as stated in Sec.~\ref{subsec:lodo}.}
\setlength{\tabcolsep}{12pt}
\renewcommand{\arraystretch}{1.4}
\label{tab:iterations_comparison}
\resizebox{\textwidth}{!}{%
\begin{tabular}{c|ccc}
\toprule
\rowcolor{COLOR_MEAN}
\textbf{Iterations}               & \textbf{Algorithms \& Models} & \textbf{DomainNet (\textit{train-val})} & \textbf{DomainNet (\textit{lodo})} \\ \hline
\multirow{8}{*}{5K}  & ERM~\citep{vapnik1991principles}  & 41.2 $\pm$ 0.2  & 40.6 $\pm$ 0.2     \\
& IRM~\citep{DBLP:journals/corr/abs-1907-02893}  & 33.9 $\pm$ 2.8  & 33.5 $\pm$ 3.0     \\
& DANN~\citep{DBLP:journals/jmlr/GaninUAGLLML16}  & 41.8 $\pm$ 0.2 & 38.2 $\pm$ 0.2 \\
& CORAL~\citep{DBLP:conf/eccv/SunS16}  & 41.8 $\pm$ 0.2 & 41.1 $\pm$ 0.1 \\
& MMD~\citep{DBLP:conf/cvpr/LiPWK18}   & 39.4 $\pm$ 0.8 & 23.4 $\pm$ 9.4 \\
& MLDG~\citep{DBLP:conf/aaai/LiYSH18} & 41.2 $\pm$ 0.1 & 41.0 $\pm$ 0.2 \\
& Fish~\citep{DBLP:journals/corr/abs-2104-09937}  & 42.7 $\pm$ 0.2 & --- \\
& Fishr~\citep{rame2021fishr} & 41.7 $\pm$ 0.2  & ---  \\
& ViT-S/16~\citep{DBLP:conf/iclr/DosovitskiyB0WZ21} & 42.4 $\pm$ 0.0 & 42.0 $\pm$ 0.2 \\
& \textbf{GMoE-S/16}  & \textbf{44.6 $\pm$ 0.3}  & \textbf{44.7 $\pm$ 0.2}     \\ \midrule
\multirow{3}{*}{15K} & Swad~\citep{DBLP:conf/nips/ChaCLCPLP21}  & 46.5 $\pm$ 0.1  & ---     \\
& MIRO~\citep{DBLP:journals/corr/abs-2203-10789}   & 47.4 $\pm$ 0.0 & ---      \\
& ViT-S/16~\citep{DBLP:conf/iclr/DosovitskiyB0WZ21} & 47.1 $\pm$ 0.3 & 46.1 $\pm$ 0.4 \\
& \textbf{GMoE-S/16} & \textbf{48.7 $\pm$ 0.2}  & \textbf{48.4} $\pm$ 0.3 \\ \bottomrule
\end{tabular}%
}
\end{table}

Since~\citet{DBLP:conf/iclr/GulrajaniL21} only evaluates baselines for 5K iterations on DomainNet, while 5K iterations are less than 2 epochs on DomainNet. This results in insufficient training on this dataset and it can not fully reveal models' performance with different algorithms. Subsequent state-of-the-art works~\citep{cha2021swad}~\citep{DBLP:journals/corr/abs-2203-10789} re-evaluated ERM and proposed their methods on 15,000 iterations of DomainNet.

To compare with those SOTA counterparts, we report the results of 15K iterations on DomainNet in Table~\ref{tab:main_results}. However, we also train GMoE for 5K iterations to make a fair comparison with both types of algorithms (see Table~\ref{tab:iterations_comparison}). For the rest of the datasets on DomainBed, we train GMoE for 5K iterations following the original setting.

\subsection{Leave-one-domain-out Results}
\label{subsec:lodo}
We have demonstrated the experimental results in Table~\ref{tab:main_results} with train-validation selection. In this section, we will further provide the results of GMoE on leave-one-domain-out selection, which is a more challenging setting. In this setting, we should train GMoE on 5 datasets with in total of 61 individual experiments and summarize them as average performance. Due to the huge amount of experiments, recent methods rarely evaluate their methods on such setting, except for those provided by DomainBed. In the Table~\ref{tab:lodo_results}, we can find that GMoE outperforms previous methods in all 5 datasets.

\begin{table}[htp]
\setlength{\tabcolsep}{12pt}
\renewcommand{\arraystretch}{1.4}
\centering
\caption{Overall out-of-domain accuracies (\%) with leave-one-domain-out selection criterion.}
\label{tab:lodo_results}
\resizebox{\textwidth}{!}{%
\begin{tabular}{c|ccccc}
\toprule
\rowcolor{COLOR_MEAN}
\textbf{Algorithm} & \textbf{PACS} & \textbf{VLCS} & \textbf{OfficeHome} & \textbf{TerraInc} & \textbf{DomainNet} \\ \midrule
ERM (ResNet50) & 83.0 \tiny{$\pm$ 0.7 } & 77.2 \tiny{$\pm$ 0.4 } & 65.7 \tiny{$\pm$ 0.5 } & 41.4 \tiny{$\pm$ 1.4 } & 40.6 \tiny{$\pm$ 0.2} \\
IRM  & 81.5 \tiny{$\pm$ 0.8 } & 76.3 \tiny{$\pm$ 0.6 } & 64.3 \tiny{$\pm$ 1.5 } & 41.2 \tiny{$\pm$ 3.6 } & 33.5 \tiny{$\pm$ 3.0} \\
DANN  & 81.0 \tiny{$\pm$ 1.1 } & 76.9 \tiny{$\pm$ 0.4 } & 64.9 \tiny{$\pm$ 1.2 } & 44.4 \tiny{$\pm$ 1.1 } & 38.2 \tiny{$\pm$ 0.2} \\
CORAL & 82.6 \tiny{$\pm$ 0.5 } & 78.7 \tiny{$\pm$ 0.4 } & 68.5 \tiny{$\pm$ 0.2 } & 46.3 \tiny{$\pm$ 1.7 } & 41.1 \tiny{$\pm$ 0.1} \\
MMD & 83.2 \tiny{$\pm$ 0.2 } & 77.3 \tiny{$\pm$ 0.5 } & 60.2 \tiny{$\pm$ 5.2 } & 46.5 \tiny{$\pm$ 1.5 } & 23.4 \tiny{$\pm$ 9.5} \\
MLDG & 77.2 \tiny{$\pm$ 0.9 } & 82.9 \tiny{$\pm$ 1.7 } & 66.1 \tiny{$\pm$ 0.5 } & 46.2 \tiny{$\pm$ 0.9 } & 41.0 \tiny{$\pm$ 0.2} \\
ViT-S/16 & 86.6 \tiny{$\pm$ 0.1 } & 78.3 \tiny{$\pm$ 0.2 } & 71.9 \tiny{$\pm$ 0.2 } & 41.9 \tiny{$\pm$ 0.3 } & 46.1 \tiny{$\pm$ 0.2} \\
\rowcolor{ROW_COLOR}
\textbf{GMoE-S/16} & \textbf{87.1 \tiny{$\pm$ 0.3}} & \textbf{80.0 \tiny{$\pm$ 0.1}} & \textbf{73.9 \tiny{$\pm$ 0.2} } & \textbf{46.7 \tiny{$\pm$ 0.3} } & \textbf{48.4 \tiny{$\pm$ 0.2}} \\
\bottomrule
\end{tabular}%
}
\end{table}

\subsection{Ablation Study with DG Algorithms}
\label{subsec:ablation_dg_algos}
This section presents ablation study to compare GMoE with DG algorithms and ViT with DG algorithms, as shown in Table \ref{tab:vit_dg}. From the table, we see that GMoE significantly outperforms ViT and ResNet under a variety of DG algorithms, showing the generalizability of the proposed backbone architecture.

\begin{table}[htp]
\setlength{\tabcolsep}{8pt}
\renewcommand{\arraystretch}{1.4}
\centering
\caption{Overall accuracies (\%) with train-validation selection criterion.}
\label{tab:vit_dg}
\resizebox{\textwidth}{!}{%
\begin{tabular}{c|c|ccccc|c}
\toprule
\rowcolor{COLOR_MEAN}
\textbf{Algorithms} & \textbf{Backbones}  & \textbf{PACS} & \textbf{VLCS} & \textbf{OfficeHome} & \textbf{TerraInc} & \textbf{DomainNet} & \textbf{Avg} \\ \midrule
\multirow{2}{*}{Swad}  & ViT-S/16           & 88.3          & 80.4          & 72.7          & 46.9          & 48.8          & 67.4          \\
                       & \textbf{GMoE-S/16} & \textbf{88.6} & \textbf{80.8} & \textbf{74.5} & \textbf{49.1} & \textbf{49.6} & \textbf{68.5} \\ \hline
\multirow{3}{*}{Fish}  & ResNet50           & 85.5          & 77.8          & 68.6          & 45.1          & 43.1          & 64            \\
                       & ViT-S/16           & 87.3          & 78            & 72.2          & 42.7          & 47.3          & 65.5          \\
                       & \textbf{GMoE-S/16} & \textbf{87.9} & \textbf{79.7} & \textbf{73.0}   & \textbf{46.9} & \textbf{48.8} & \textbf{67.1} \\ \hline
\multirow{3}{*}{Fishr} & ResNet50           & 85.5          & 77.8          & 68.6          & 47.4          & 41.7          & 64.2          \\
                       & ViT-S/16           & 85.6          & \textbf{79.1} & 71.8          & 44.1          & 47.6          & 65.6          \\
                       & \textbf{GMoE-S/16} & \textbf{88.3} & 78.8  & \textbf{72.7} & \textbf{45.6} & \textbf{48.3} & \textbf{66.7} \\ \bottomrule
\end{tabular}
}
\end{table}

\subsection{Single-source Domain Generalization Results}
\label{subsec:purely_ood_results}
In this section, we will demonstrate the \emph{single source domain generalization} results on DomainNet's 6 domains. We select different training domains in turn and take the remaining 5 domains as test domains. In Table~\ref{tab:purely_ood_app}, we show the OOD results when models are trained on \emph{only one} domain and test on 5 test domains, as well as the IID results on the training domain's validation set. We specify IID and OOD results in different colors.


\begin{table}[ht]
\centering
\caption{Single-source DG results on DomainNet. We alternatively choose 1 training domain and the other 5 as test domains. We report \colorbox{COLOR_MEAN}{IID generalization} results on training domain's validation set with grey color, and \colorbox{ROW_COLOR}{OOD generalization} results on test domains' validation set with light cyan color. \\}
\label{tab:purely_ood_app}
\setlength{\tabcolsep}{14pt}
\renewcommand{\arraystretch}{1.3}
\footnotesize
\begin{tabular}{c|cccccc}
\toprule
\rowcolor{ROW_COLOR}
\cellcolor{COLOR_MEAN}\textbf{Clipart} & \cellcolor{COLOR_MEAN}\textbf{Clipart} & \textbf{Info} & \textbf{Paint} & \textbf{Quick} & \textbf{Real} & \textbf{Sketch} \\ \hline
ResNet50 & 68.5 & 12.3 & 25.8 & 9.5 & 39.6 & 37.1 \\
ResNet101 & 70.4 & 12.2 & 28.6 & 12.4 & 43.0 & 40.7 \\
ViT-S/16 & 74.7 & 16.2 & 35.3 & 12.7 & 50.9 & 40.4 \\
GMoE-S/16 & \textbf{76.3} & \textbf{16.8} & \textbf{36.2} &  \textbf{13.9} & \textbf{51.6} & \textbf{42.6}
\\ \midrule
\rowcolor{ROW_COLOR}
\cellcolor{COLOR_MEAN}\textbf{Info} & \textbf{Clipart} & \cellcolor{COLOR_MEAN}\textbf{Info} & \textbf{Paint} & \textbf{Quick} & \textbf{Real} & \textbf{Sketch} \\ \hline
ResNet50 & 27.9 & 34.0 & 23.3 & 2.0 & 34.3 & 24.5 \\
ResNet101 & 29.2 & 35.6 & 24.0 & 3.1 & 34.1 & 24.6 \\
ViT-S/16 & 36.3 & 40.2 & 33.6 & \textbf{5.1} & 47.4 & 30.7 \\
GMoE-S/16 & \textbf{36.7} & \textbf{41.0} & \textbf{34.2} & 4.5 & \textbf{48.2} & \textbf{30.8} \\ \midrule
\rowcolor{ROW_COLOR}
\cellcolor{COLOR_MEAN}\textbf{Paint} & \textbf{Clipart} & \textbf{Info} & \cellcolor{COLOR_MEAN}\textbf{Paint} & \textbf{Quick} & \textbf{Real} & \textbf{Sketch} \\ \hline
ResNet50 & 37.1 & 12.9 & 62.7 & 2.2 & 49.3 & 33.3 \\
ResNet101 & 40.5 & 13.1 & 63.4 & 3.1 & 51.2 & 35.4 \\
ViT-S/16 & 42.7 & 15.9 & 69.0 & 5.0 & \textbf{56.4} & 37.0 \\
GMoE-S/16 & \textbf{43.5} & \textbf{16.1} & \textbf{69.3} & \textbf{5.3} & \textbf{56.4} & \textbf{38.0} \\
\rowcolor{ROW_COLOR}
\cellcolor{COLOR_MEAN}\textbf{Quick} & \textbf{Clipart} & \textbf{Info} & \textbf{Paint} & \cellcolor{COLOR_MEAN}\textbf{Quick} & \textbf{Real} & \textbf{Sketch} \\ \hline
ResNet50 & 17.3 & 1.1 & 2.8 & 62.4 & 7.3 & 9.2 \\
ResNet101 & 16.0 & 1.3 & 2.7 & 64.1 & 6.3 & 8.5 \\
ViT-S/16 & 22.1 & \textbf{2.0} & 6.3 & 66.2 & 11.8 & 11.8 \\
GMoE-S/16 & \textbf{22.8} & \textbf{2.0} & \textbf{6.4} & \textbf{66.8} & \textbf{12.2} & \textbf{12.9} \\ \midrule
\rowcolor{ROW_COLOR}
\cellcolor{COLOR_MEAN}\textbf{Real} & \textbf{Clipart} & \textbf{Info} & \textbf{Paint} & \textbf{Quick} & \cellcolor{COLOR_MEAN}\textbf{Real} & \textbf{Sketch} \\ \hline
ResNet50 & 39.5 & 13.0 & 40.0 & 4.4 & 73.6 & 30.5 \\
ResNet101 & 43.3 & 15.6 & 43.3 & 6.8 & 75.3 & 34.0 \\
ViT-S/16 & 46.9 & 18.0 & 46.0 & 5.0 & 79.5 & 33.4 \\
GMoE-S/16 & \textbf{47.4} & \textbf{18.5} & \textbf{46.9} & \textbf{7.5} & \textbf{79.6} & \textbf{36.4} \\ \midrule
\rowcolor{ROW_COLOR}
\cellcolor{COLOR_MEAN}\textbf{Sketch} & \textbf{Clipart} & \textbf{Info} & \textbf{Paint} & \textbf{Quick} & \textbf{Real} & \cellcolor{COLOR_MEAN}\textbf{Sketch} \\ \hline
ResNet50 & 48.9 & 11.7 & 31.6 & 10.8 & 40.7 & 63.4 \\
ResNet101 & 48.4 & 11.8 & 31.9 & 11.3 & 38.1 & 64.3 \\
ViT-S/16 & 52.4 & \textbf{14.4} & \textbf{38.6} & 13.4 & 48.2 & 67.4 \\
GMoE-S/16 & \textbf{52.9} & 14.1 & 36.6 & \textbf{13.7} & \textbf{49.0} & \textbf{68.4} \\
\bottomrule
\end{tabular}%
\end{table}

\subsection{Ablation Study on Model Design}
\label{subsec:ablation_model_design}
In order to validate discussions in Sec.~\ref{sec:impl:diff_teq} and Sec~\ref{subsec:detailed_arch}, we conduct experiments with different layer configurations, as well as adopting larger backbone (\eg ViT-Base) model. From the results in Table~\ref{tab:gmoe_arch_comp}, we see that the \textit{last two} configuration largely exceeds \textit{every two} configuration, thus verifying our idea that visual attributes exist in relatively high-level signals.
This result also demonstrates that the cosine router and larger model lead to better generalization performance.

\begin{table}[ht]
\centering
\caption{Train-validation selection performance comparison for different GMoE models. All GMoEs (S/16 and B/16) have $E = 6$ experts and $L = 12$ blocks. We specify the number of attention heads ($H$), the patch embedding size ($D$), layer configuration, and router's type in the table header.}
\setlength{\tabcolsep}{6pt}
\renewcommand{\arraystretch}{1.4}
\label{tab:gmoe_arch_comp}
\resizebox{\textwidth}{!}{%
\begin{tabular}{c|cccc|ccccc}
\toprule
\rowcolor{COLOR_MEAN}
\textbf{Algorithm} &
  \textbf{Config.} &
  \textbf{Router} &
  $H$ &
  $D$ &
  \textbf{PACS} &
  \textbf{VLCS} &
  \textbf{OfficeHome} &
  \textbf{TerraInc} &
  \textbf{DomainNet} \\ \toprule
GMoE-S/16 & \textit{Every 2} & Linear & 6  & 384  & 81.8 \tiny{$\pm$ 0.2} & 75.0 \tiny{$\pm$ 0.1}  &   64.0 \tiny{$\pm$ 0.4}  & 32.5 \tiny{$\pm$ 0.7} &  46.3 \tiny{$\pm$ 0.3}    \\
GMoE-S/16 & \textit{Every 2} & Cosine & 6  & 384  & 81.4 \tiny{$\pm$ 0.1} & 74.8 \tiny{$\pm$ 0.2}  &   62.2 \tiny{$\pm$ 0.4}  & 40.9 \tiny{$\pm$ 0.3} &  46.4 \tiny{$\pm$ 0.2}    \\ \hline
GMoE-S/16 & \textit{Last 2}  & Linear & 6  & 384  & 87.8 \tiny{$\pm$ 0.2} & 80.0 \tiny{$\pm$ 0.0}  & 72.7 \tiny{$\pm$ 0.2}  & 46.7 \tiny{$\pm$ 0.2} & 48.3 \tiny{$\pm$ 0.1} \\
GMoE-S/16 & \textit{Last 2}  & Cosine & 6  & 384  & 88.1 \tiny{$\pm$ 0.1} & 80.2 \tiny{$\pm$ 0.2}  & 74.2 \tiny{$\pm$ 0.4}  & 48.5 \tiny{$\pm$ 0.4} & 48.7 \tiny{$\pm$ 0.2} \\ \hline
GMoE-B/16 & \textit{Last 2}  & Cosine & 12 & 768 & 89.4 \tiny{$\pm$ 0.1} & 81.2 \tiny{$\pm$ 0.1} & 77.2 \tiny{$\pm$ 0.4} & 49.3 \tiny{$\pm$ 0.3} & 51.3 \tiny{$\pm$ 0.1} \\ \bottomrule
\end{tabular}%
}
\end{table}

\begin{table}[ht]
\centering
\caption{Comparison of different training recipes on ResNet-50 V2 and ViT-S/16.}
\setlength{\tabcolsep}{6pt}
\renewcommand{\arraystretch}{1.4}
\label{tab:training_recipe}
\resizebox{\textwidth}{!}{%
\begin{tabular}{c|ccccccccccc|c}
\toprule
\textbf{Recipe} & \textbf{LR Opt.} & \textbf{TrivialAug.} & \textbf{Ep.} & \textbf{Rand Er.} & \textbf{Label Sm.} & \textbf{FixRes Mt.} & \textbf{WDT} & \textbf{IRT} & \textbf{IN1K} \\ \midrule
ResNet-50 V2 & \greencheck & \greencheck & 600 & \greencheck & \greencheck & \greencheck & \greencheck & \greencheck & 80.8 \\
ViT-S/16 & \greencheck & \redcheck & 300 & \greencheck & \redcheck & \redcheck & \redcheck & \greencheck & 79.9 \\ \bottomrule
\end{tabular}%
}
\end{table}

\begin{table}[ht]
\centering
\caption{Train-validation selection performance comparison for ViT, GMoE, and other DG algorithms with ResNet-50 V2 as backbone model. On DomainNet, results are reported with 15K iterations.}
\label{tab:weights_v2}
\setlength{\tabcolsep}{14pt}
\renewcommand{\arraystretch}{1.4}
\resizebox{\textwidth}{!}{%
\begin{tabular}{c|ccccc}
\toprule
\rowcolor{COLOR_MEAN}
\textbf{Algorithm} & \textbf{PACS} & \textbf{VLCS} & \textbf{OfficeHome} & \textbf{TerraInc} & \textbf{DomainNet} \\ \midrule
ERM (w/ ResNet-50 V2)                & 87.2          & 78.2          & 68.7  & 49.9   &   45.3        \\
Fishr (w/ ResNet-50 V2)             & 87.5          & 77.9          & 70.4     & \textbf{51.7} &   47.0       \\
ViT-S/16           & 86.2          & 79.7          & 72.2     & 42.0   & 47.1       \\
\rowcolor{ROW_COLOR}
\textbf{GMoE-S/16}         & \textbf{88.1}          & \textbf{80.1}          & \textbf{74.2} & 48.5  & \textbf{48.7} \\ \bottomrule
\end{tabular}%
}
\end{table}

\subsection{Ablation Study on ResNet with Strong Data Augmentation}
\label{subsec:ablation_init_weights}
In this part, we will go into greater depth about the pre-trained model's effect on DG performance. As stated in Sec.~\ref{subsec:model_init}, GMoE is initialized from the same architecture ViT models, where pre-trained weights are from~\citet{DBLP:conf/icml/TouvronCDMSJ21}. In addition, the pre-training on ViT and ResNet models may adopt different training recipes. And it is generally believed that models with stronger data augmentation during pre-training could have better performance on downstream tasks.

To verify this difference in detail, we consider comparing GMoE-S/16's pre-trained model, ViT-S/16, with a stronger pre-trained ResNet-50 V2 model\footnote{From \textsc{torchvision}'s \href{https://pytorch.org/blog/how-to-train-state-of-the-art-models-using-torchvision-latest-primitives/}{Pretrained Models}.} with significant training tricks data augmentations. The pre-training details on ViT-S/16 and ResNet V2 are listed in Table~\ref{tab:training_recipe}.

Due to the space limit, we use abbreviations for some terms. In detail, they are:

\begin{itemize}
    \item \textbf{LR Opt.} stands for learning rate optimization.
    \item \textbf{TrivialAug.} stands for trivial augmentation~\citep{muller2021trivialaugment}.
    \item \textbf{Ep.} stands for pre-training epochs.
    \item \textbf{Rand Er.} stands for random erasing~\citep{zhong2017random}.
    \item \textbf{Label Sm.} stands for label smoothing~\citep{szegedy2016rethinking}.
    \item \textbf{FixRes Mt.} stands for FixRes mitigations~\citep{touvron2019fixing}.
    \item \textbf{WDT} stands for weight decay tuning.
    \item \textbf{IRT} stands for inference resize tuning~\citep{touvron2019fixing}.
\end{itemize}

The results in Table~\ref{tab:weights_v2} demonstrate that GMoE and ViT could still remain higher performance than the strong pre-trained model ResNet V2 due to the benefits of the backbone architecture. 


\subsection{Computational cost comparison}
Having shown the effects of GMoE's performance on different domain generalization benchmarks, we now conduct efficiency analysis with respect to \textbf{iteration time} and \textbf{run-time memory} during training and inference. Since algorithms developed from DomainBed mainly adopt the same architectures (\eg ResNet50, Linear Classifier), traditional model complexity measures such as flops and model parameters can not truly reflect the difference in efficiency. By evaluating the above two metrics, different algorithms with more complex loss design or gradient constraints will cause larger overheads with respect to iteration time or run-time memory. From the results in Table~\ref{tab:comp_cost}, we observe that GMoE achieves relatively low run-time memory and training step time among other competitors.

\begin{table}[htp]
\setlength{\tabcolsep}{6pt}
\renewcommand{\arraystretch}{1.4}
\centering
\caption{Comparison of training/inference iteration time and run-time memory for a mini-batch. A mini-batch is formed with 160 images in 224 $\times$ 224 resolutions from DomainBed. For both metrics, lower is better.}
\label{tab:comp_cost}
\resizebox{\textwidth}{!}{%
\begin{tabular}{c|cccccc|cc}
\toprule
\rowcolor{COLOR_MEAN}
\textbf{Training} & \textbf{ERM} & \textbf{DANN} & \textbf{IRM} & \textbf{Fish} & \textbf{Fishr} & \textbf{SWAD} &  \textbf{VIT-S/16} & \textbf{GMoE-S/16} \\ \hline
Step Time (s) $\downarrow$ & 1.01 & 1.02 & 1.10 & 2.79 & 1.10 & 1.21 &  0.90 & 0.98 \\ \hline
Run-time Memory (GB) $\downarrow$ & 13.40 & 13.42 & 13.40 &  3.41 & 15.25 & 14.32 &  11.15 & 12.28 \\ \midrule
\rowcolor{COLOR_MEAN}
\textbf{Inference} & \textbf{ERM} & \textbf{DANN} & \textbf{IRM} & \textbf{Fish} & \textbf{Fishr} & \textbf{SWAD} &  \textbf{VIT-S/16} & \textbf{GMoE-S/16} \\ \hline
Step Time (s) $\downarrow$ & 0.32 & 0.33 & 0.32 & 0.33 & 0.34 & 0.33 &  0.28 & 0.30 \\ \hline
Run-time Memory (GB) $\downarrow$ & 1.82 & 1.83 & 1.82 & 1.82 & 1.83 & 1.84 &  0.76 & 1.05 \\ \bottomrule
\end{tabular}%
}\end{table}
 \clearpage
\section{Model Analysis}

\subsection{CUB-DG Results \& Visualizations}
\label{subsec:cub_dg}

\paragraph{Generalization across Image Stylization}
To evaluate the model performance in generalizing to image stylization, we conduct experiments following DomainBed setting on CUB-DG. We compare GMoE with other three invariant learning DG algorithms, (i.e., DANN~\citep{DBLP:journals/jmlr/GaninUAGLLML16}, Fish~\citep{DBLP:journals/corr/abs-2104-09937} and Fishr~\citep{rame2021fishr}). Among them, DANN and Fish are widely discussed and cited papers, and Fishr is considered the best algorithm of its kind so far.

In this experiment, all three invariant learning methods, including ERM, adopt ResNet-50 as the backbone model. GMoE does not have any specific loss design and constraint. The only difference is the model architecture. In Table~\ref{tab:cub_dg}, we see that GMoE, the new backbone model for DG, is significantly more effective than those methods specifically designed to learn invariance with certain losses and gradient constraints.

\begin{table}[htp]
\setlength{\tabcolsep}{16pt}
\renewcommand{\arraystretch}{1.4}
\centering
\caption{Out-of-domain accuracy (\%) in each domain on CUB-DG dataset. The best is in bold.}
\label{tab:cub_dg}
\resizebox{0.75\textwidth}{!}{%
\begin{tabular}{c|cccc}
\toprule
\rowcolor{COLOR_MEAN}
\textbf{Algorithm} & \textbf{Candy} & \textbf{Mosaic} & \textbf{Natural} & \textbf{Udnie} \\ \hline
ERM (ResNet50)                & 64.3    & 20.8     & 75.3      & 77.9    \\
DANN               & 47.0    & 14.9     & 76.5      & 70.6    \\
Fish               & 62.6    & 22.4     & 84.5      & 77.7    \\
Fishr               & 62.4    & 24.3     & 78.8      & 73.8    \\
\rowcolor{ROW_COLOR}
\textbf{GMoE-S/16  }        & \textbf{83.1}    & \textbf{42.8}     & \textbf{89.3}      & \textbf{82.7}    \\
\bottomrule
\end{tabular}%
}
\end{table}

\paragraph{Experts Selection}
Following the analysis in Sec.~\ref{subsec:model_analysis}, we provide more visualization results in this section.

In Figure~\ref{fig:cub_experts_corr}-\ref{fig:cub_experts_corr2}, we provide the expert selection results in natural domain images across different bird classes. We use different colored lines that correspond to different visual attributes, and we connect these lines across images in a row. It indicates that experts focus on the same visual attributes while seeing different bird images from different angles. This suggests that mixture-of-experts have the ability to capture and process visual attributes \emph{across class}.

In Figure~\ref{fig:cub_vis1}-\ref{fig:cub_vis4}, we provide the expert selection results for the same image on four domains (with four types of stylization). In order to more vividly display the expert selection across domains, we assign different colors to different selected experts. We see that, in most cases, the same areas and visual attributes are handled by the same experts across domains. For example, in Figure~\ref{fig:cub_vis1}, expert 1,3 and 4 focus on the bird while expert 0,2 focus on the background. And it is consistent across four domains with different stylization. The results reveal that experts are potentially invariant in dealing with visual attributes \emph{across domains}.


\paragraph{Multi-head Attention Visualization}
\label{subsec:mha_results}
Ideally, the multi-head attention mechanism can be viewed as jointly attending to multiple places by ensembling multiple attention heads. Each attention head would focus on its specific attention relationship between all patches and inherently collaborate with each other.

\begin{figure}[htp]
    \centering
    \includegraphics[width=0.7\textwidth]{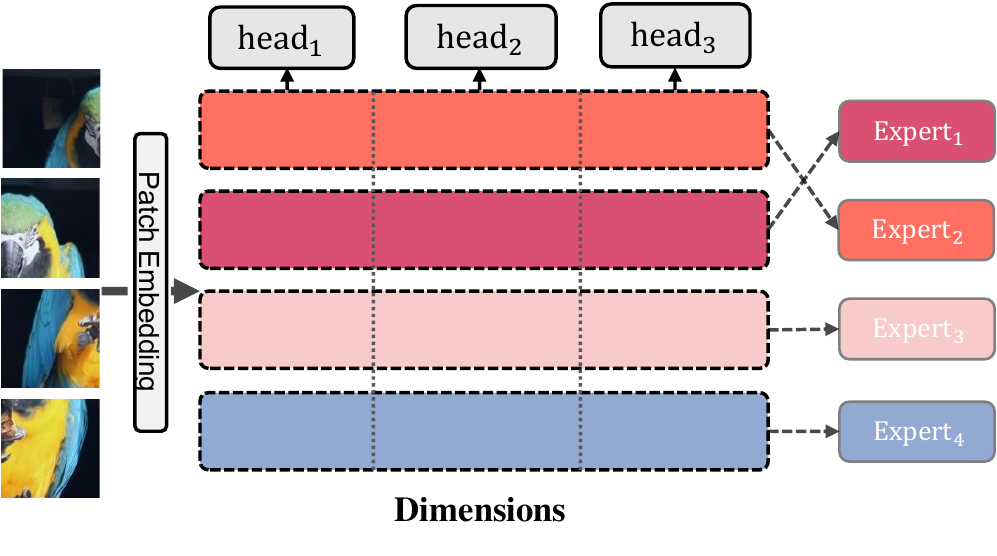}
    \caption{Diagram of a GMoE block, where different attention heads attend to different aspects of relations between patches. Different experts handle the learned attention in different patches. The router has been omitted for brevity.}
    \label{fig:grid_fashion}
\end{figure}

However, recent studies~\citep{DBLP:journals/corr/abs-2006-16362} demonstrate that multiple-head attention (MHA) layers could learn redundant key/query projections, i.e. some heads might attend to similar features in input space. This issue demonstrates that two heads $\text{head}_i$ and $\text{head}_j$ are computing the same key/query representations up to a unitary matrix $\mathbf{I} \in  \mathbb{R}^{d \times d}$ such that $W_{Q_i} = W_{Q_j} \mathbf{I}$ and $W_{K_i} = W_{K_j} \mathbf{I}$. In this case, even though the two heads are computing identical attention scores, i.e. $W_{Q_i} \mathbf{I} \mathbf{I}^{T} W_{K_i}^{T}$, the concatenation $[W_{Q_i}, W_{Q_j}] \in \mathbb{R}^{d \times 2d}$ can also be full rank, which indicates that some attention heads would focus on the same content but are agnostic to each other. We opine that the mixture-of-experts (MoE) layer could alleviate redundancy and increase disentanglement to some extent.

In Figure~\ref{fig:grid_fashion}, presented in a grid fashion, the MHA layer focuses on different signals (features) in different heads while MoE layer disentangles the information by handling each patch to different experts. This mechanism extends the network sparsity and hence is able to improve the model generalizability. In sum, the MoE layer is designed to leverage different experts to capture distinct visual attributes, as well as their correlations (via attention mechanism) to each other.


To see more about the learned representation in multi-head attention layer, we provide the visualization of last layer's attention activations of ViT and GMoE on Figure~\ref{fig:vis1}~\ref{fig:vis2}~\ref{fig:gmoe_vis3}~\ref{fig:gmoe_vis4}.

\clearpage




\begin{figure}
    \centering
    \includegraphics[width=0.95\textwidth]{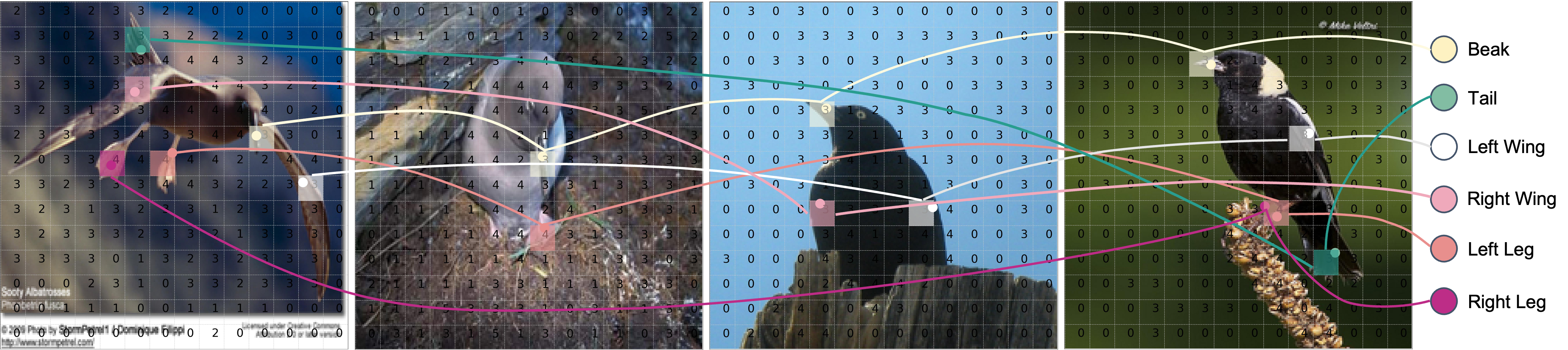}
    \caption{Expert selection visualization on block-10 of GMoE-S/16. Images are from different classes of natural domains on CUB-DG. Different colored lines connect the same type of visual attributes across images. Same visual attribute is processed by the same expert, \eg beak and tail by expert 3, left/right leg by expert 4.}
    \label{fig:cub_experts_corr}
\end{figure}

\begin{figure}
    \centering
    \includegraphics[width=0.95\textwidth]{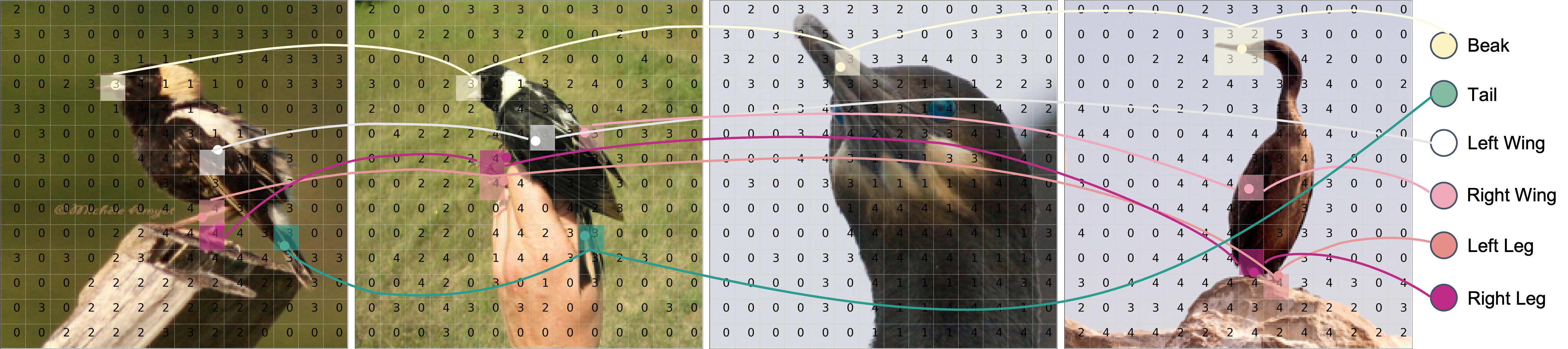}
    \caption{Expert selection visualization on block-10 of GMoE-S/16. Images are from different classes of natural domains on CUB-DG. Different colored lines connect the same type of visual attributes across images. Same visual attribute is processed by the same expert, \eg beak and tail by expert 3, left/right leg by expert 4.}
    \label{fig:cub_experts_corr2}
\end{figure}

\begin{figure}
    \centering
    \includegraphics[width=0.95\textwidth]{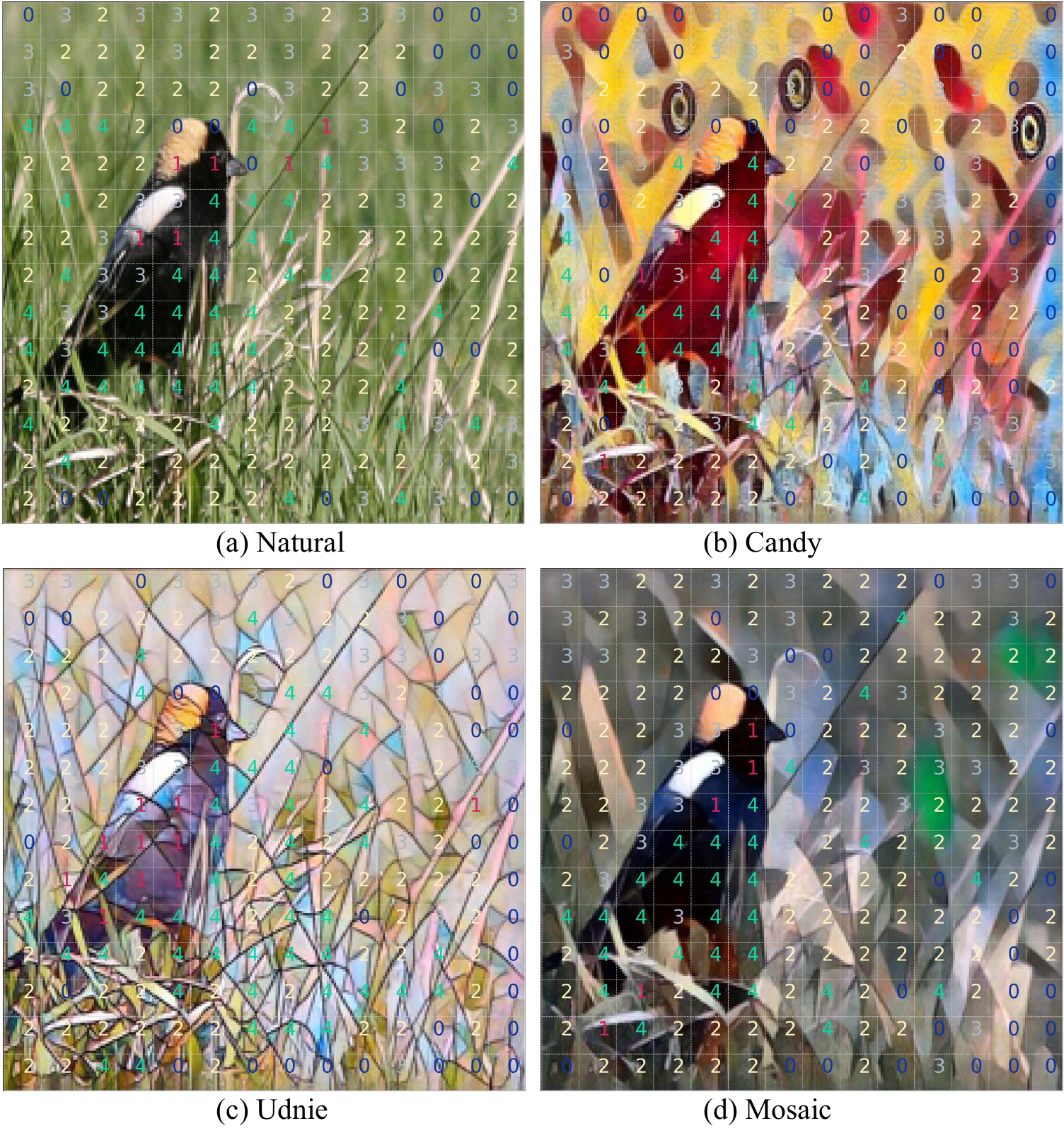}
    \caption{Expert selection visualization on block-10 of GMoE-S/16. Images are from different domains on CUB-DG.}
    \label{fig:cub_vis1}
\end{figure}

\begin{figure}
    \centering
    \includegraphics[width=0.95\textwidth]{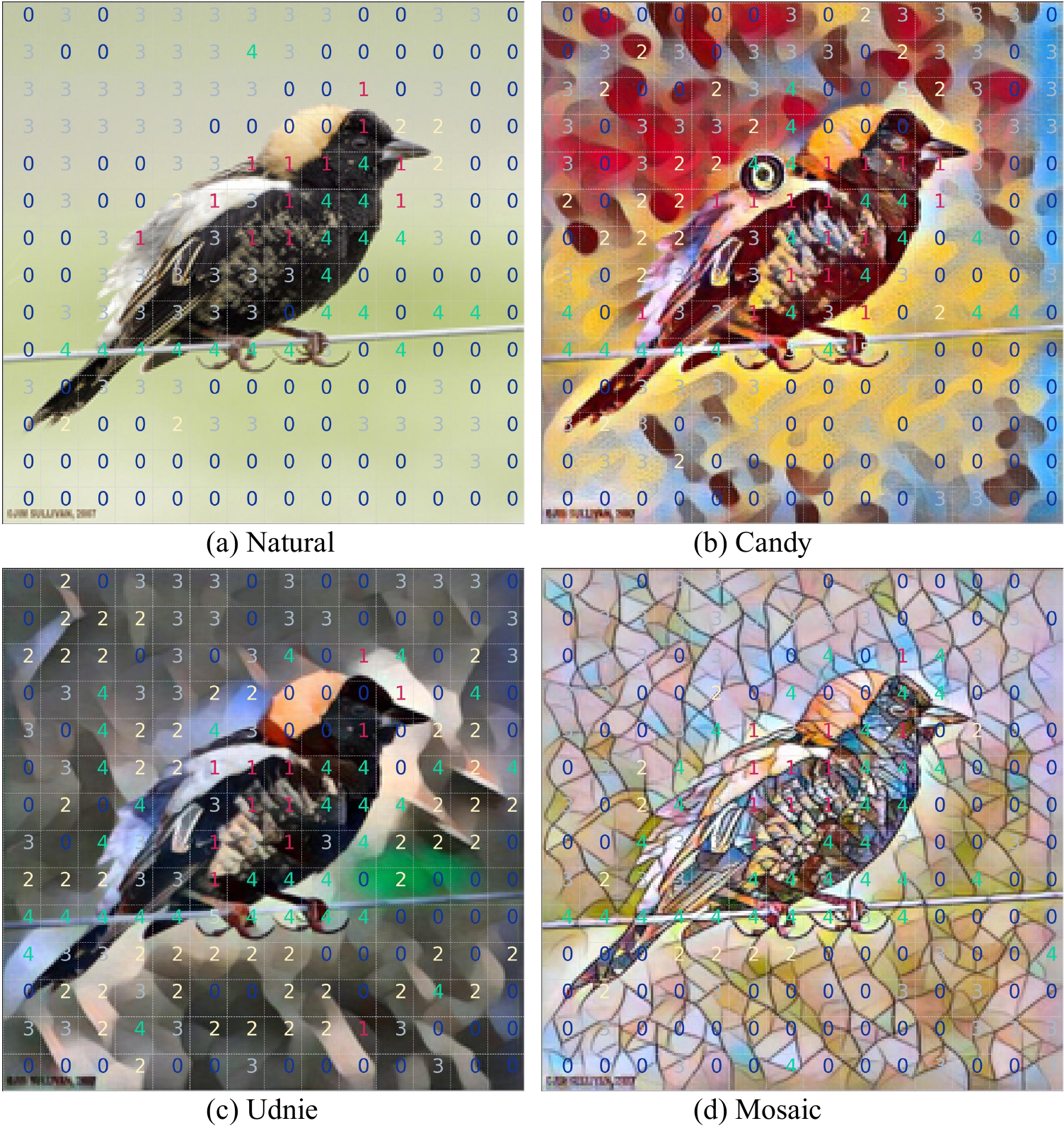}
    \caption{Expert selection visualization on block-10 of GMoE-S/16. Images are from different domains on CUB-DG.}
    \label{fig:cub_vis2}
\end{figure}

\begin{figure}
    \centering
    \includegraphics[width=0.95\textwidth]{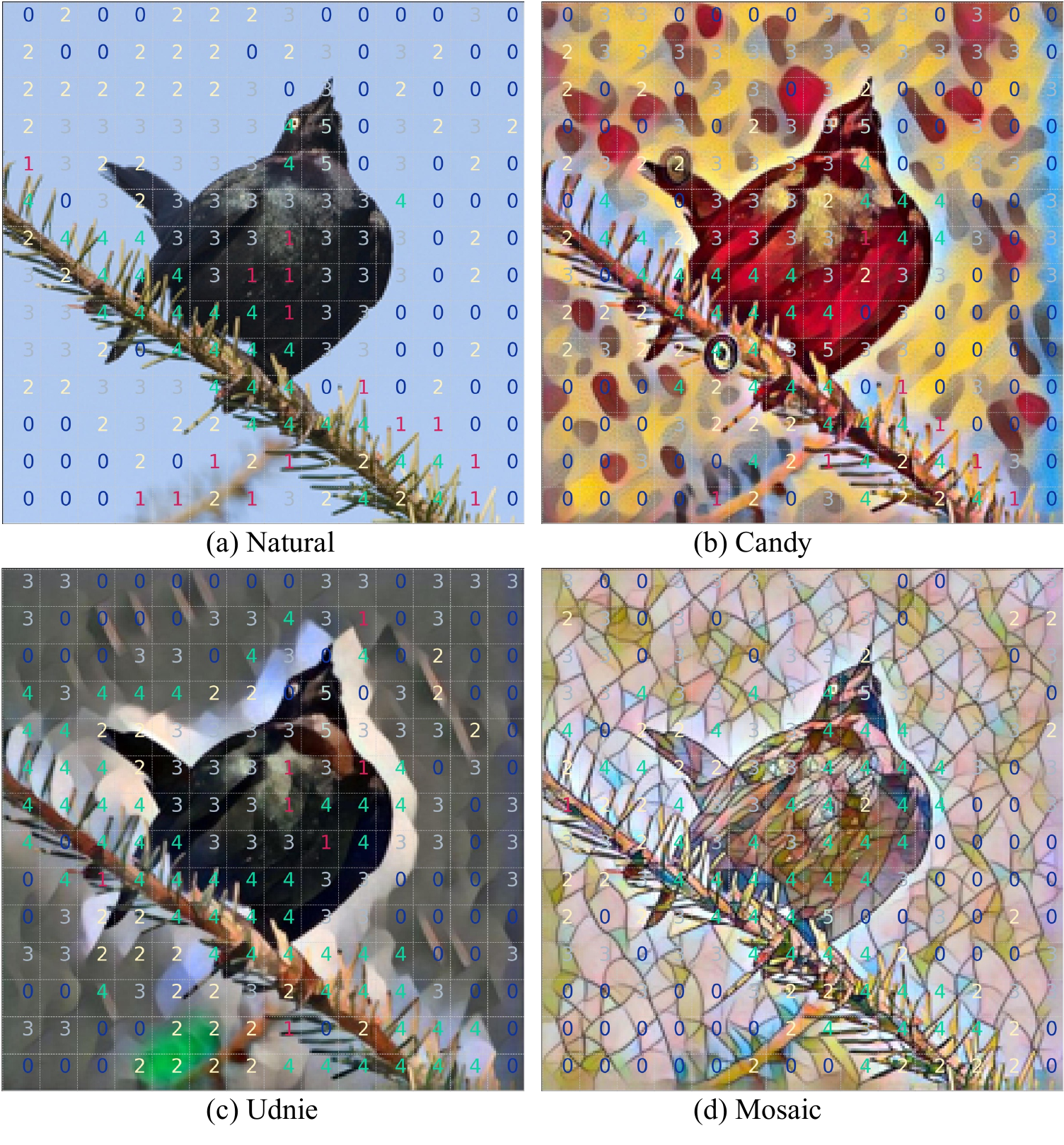}
    \caption{Expert selection visualization on block-10 of GMoE-S/16. Images are from different domains on CUB-DG.}
    \label{fig:cub_vis3}
\end{figure}

\begin{figure}
    \centering
    \includegraphics[width=0.95\textwidth]{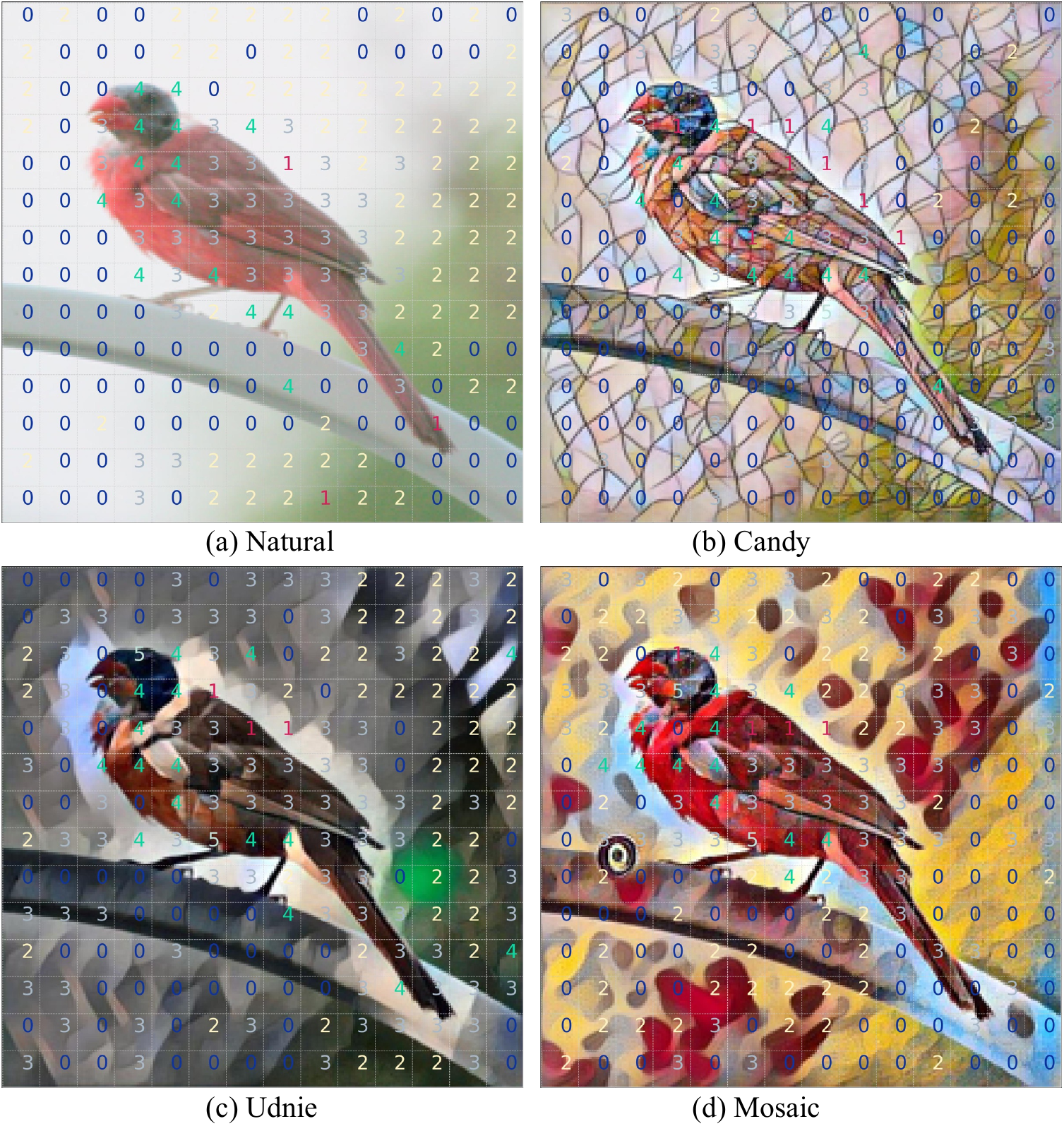}
    \caption{Expert selection visualization on block-10 of GMoE-S/16. Images are from different domains on CUB-DG.}
    \label{fig:cub_vis4}
\end{figure}

\begin{figure}
    \centering
    \includegraphics[width=0.95\textwidth]{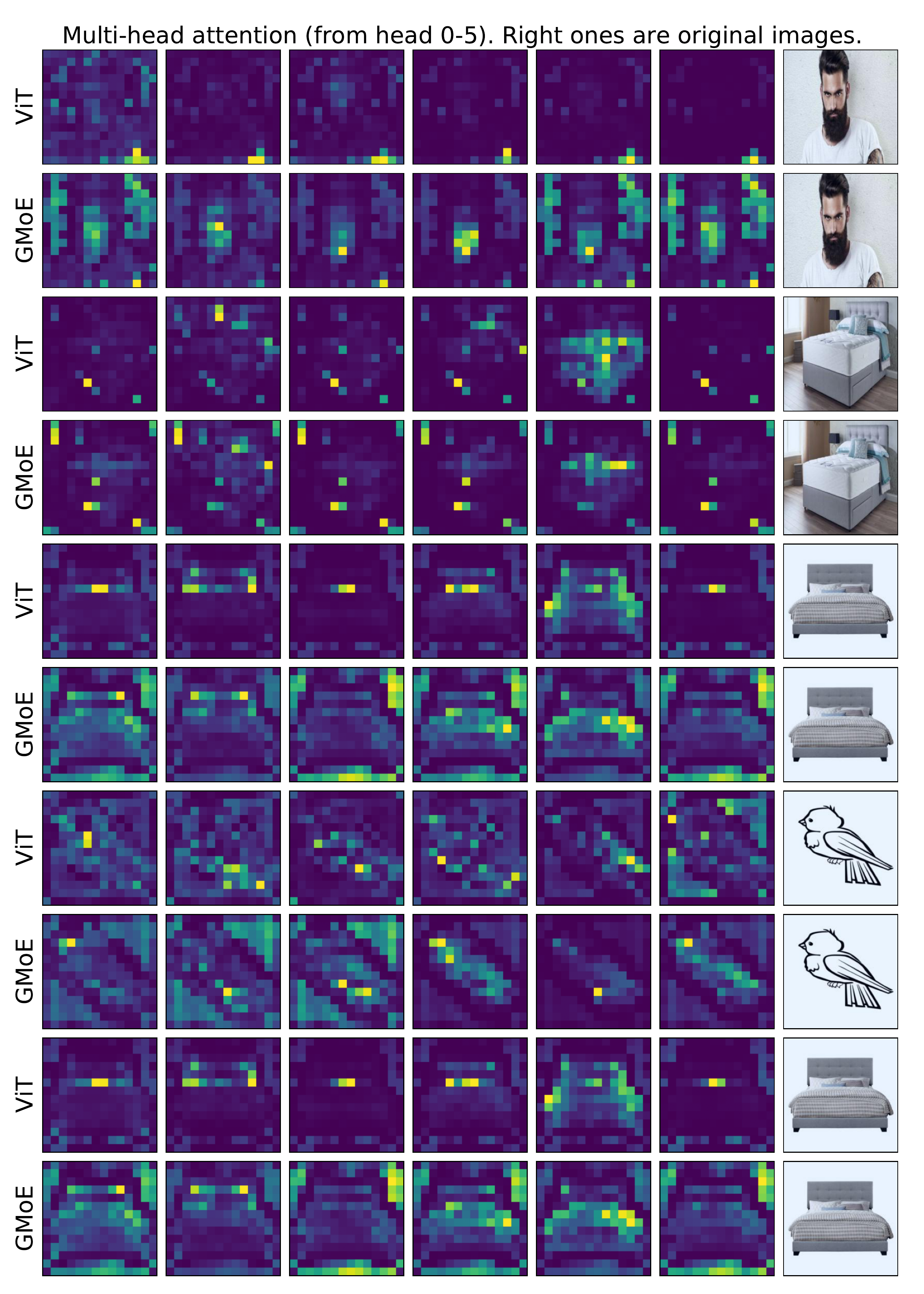}
    \caption{Multi-head attention visualization on the last block of ViT-S/16 and GMoE-S/16. Images are from \textit{Real} domain in DomainNet.}
    \label{fig:vis1}
\end{figure}

\begin{figure}
    \centering
    \includegraphics[width=0.95\textwidth]{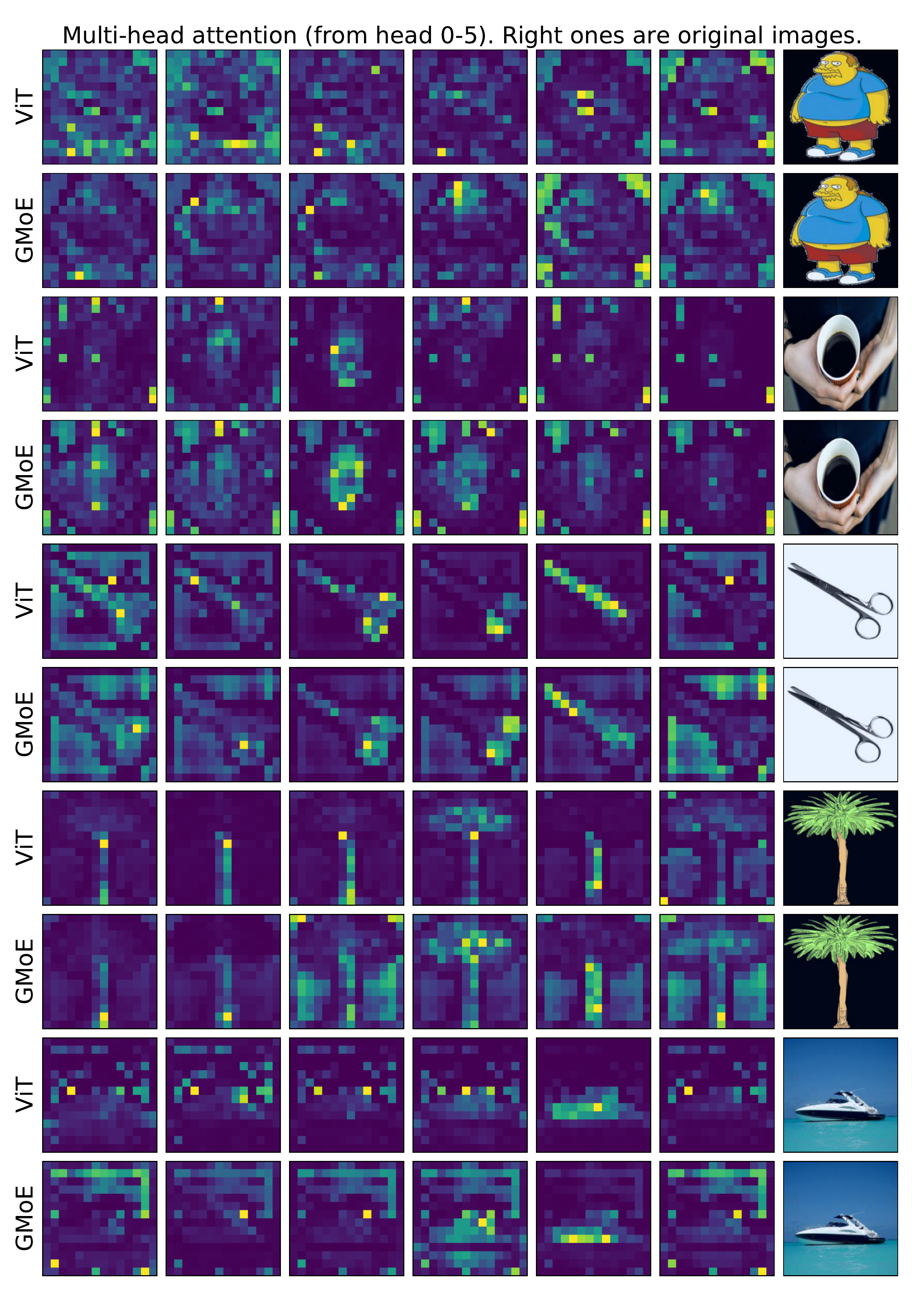}
    \caption{Multi-head attention visualization on the last block of ViT-S/16 and GMoE-S/16. Images are from \textit{Real} domain in DomainNet.}
    \label{fig:vis2}
\end{figure}

\begin{figure}
    \centering
    \includegraphics[width=0.95\textwidth]{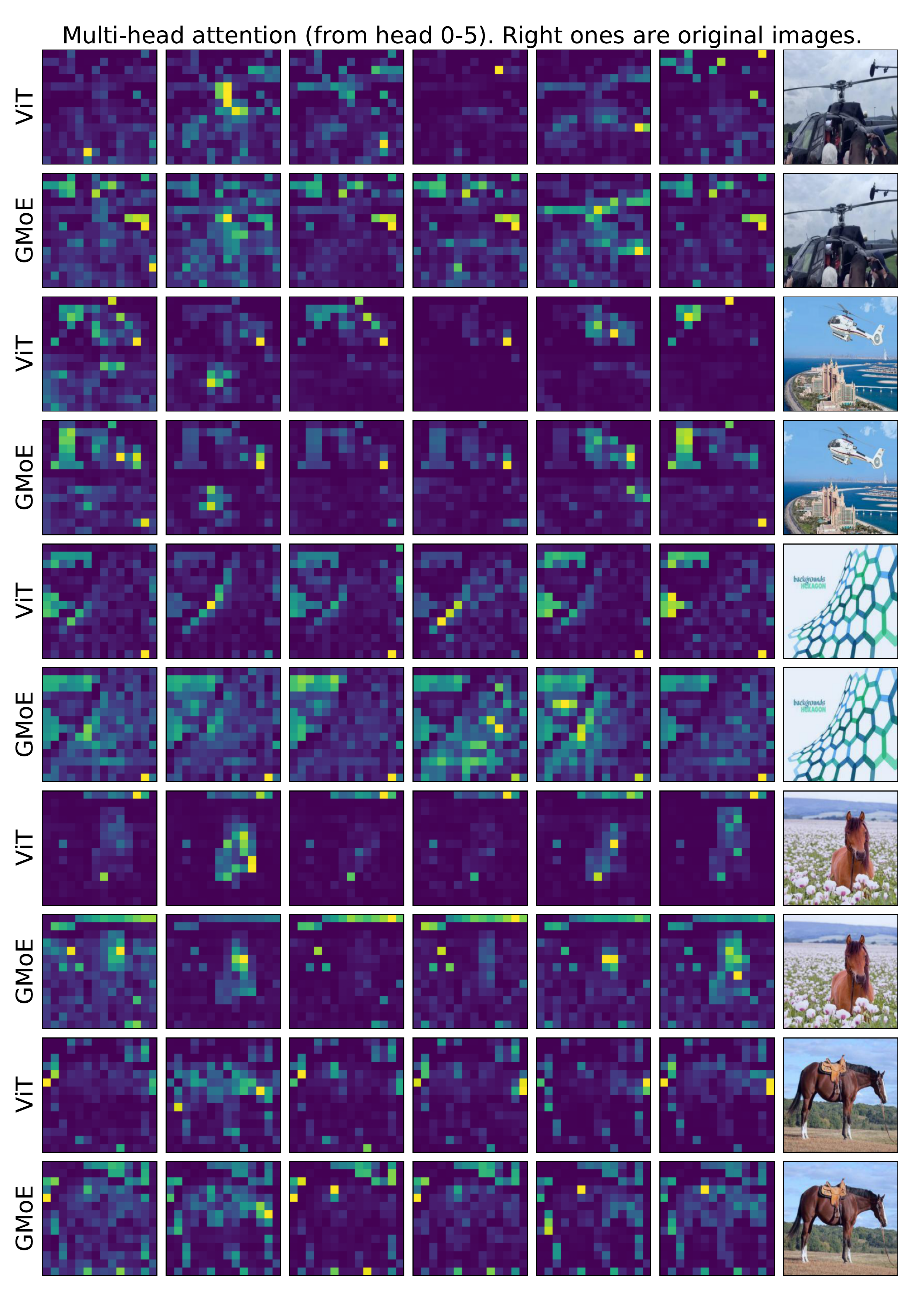}
    \caption{Multi-head attention visualization on the last block of ViT-S/16 and GMoE-S/16. Images are from \textit{Real} domain in DomainNet.}
    \label{fig:gmoe_vis3}
\end{figure}

\begin{figure}
    \centering
    \includegraphics[width=0.95\textwidth]{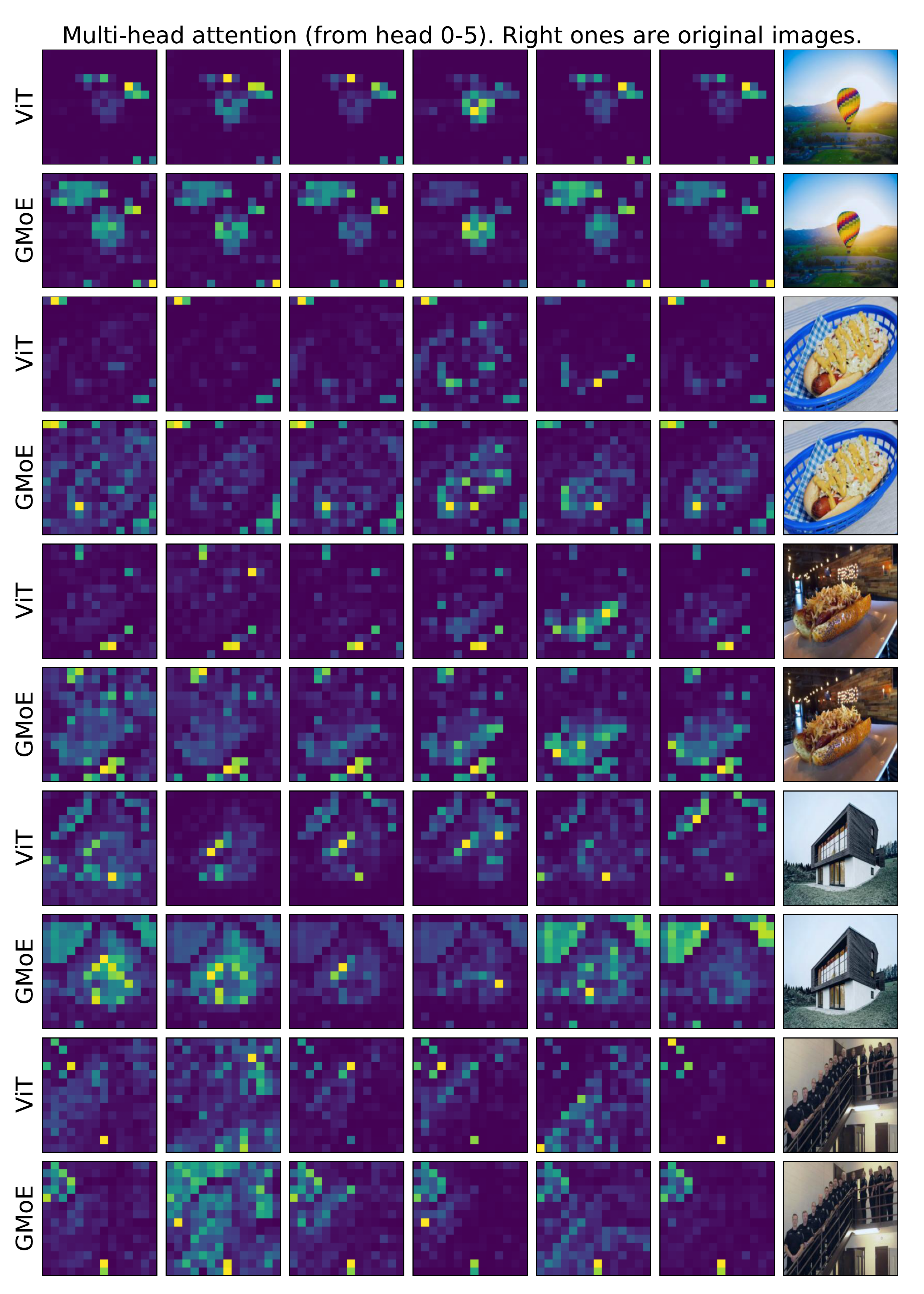}
    \caption{Multi-head attention visualization on the last block of ViT-S/16 and GMoE-S/16. Images are from \textit{Real} domain in DomainNet.}
    \label{fig:gmoe_vis4}
\end{figure}

\end{document}